\documentclass{article}

\usepackage[UKenglish]{babel}
\usepackage{csquotes}
\usepackage{xparse}

\usepackage{microtype}
\usepackage[shortlabels]{enumitem}
\usepackage{verbatim}
\usepackage{graphicx}
\usepackage[table]{xcolor}
\usepackage[outline]{contour}
\usepackage{caption}
\usepackage{subcaption}
\usepackage{longtable}
\usepackage{booktabs} 
\usepackage{multirow}

\newcolumntype{L}[1]{>{\raggedright\let\newline\\\arraybackslash\hspace{0pt}}m{#1}}
\newcolumntype{C}[1]{>{\centering\let\newline\\\arraybackslash\hspace{0pt}}m{#1}}
\newcolumntype{R}[1]{>{\raggedleft\let\newline\\\arraybackslash\hspace{0pt}}m{#1}}
\newcolumntype{K}[1]{>{\centering\let\newline\\\arraybackslash\hspace{0pt}}p{#1}}

\usepackage{hyperref}

\usepackage[accepted]{icml2025}

\defcitealias{AlbarghouthiDAntoniDrewsEtAl2017}{\({}^{\ast}\)}
\defcitealias{MarzariRoncolatoFarinelli2023}{\({}^{\dagger}\)}
\defcitealias{MarzariCorsiMarchesiniEtAl2024}{\({}^{\ddagger}\)}

\usepackage{mathtools}
\usepackage{amsmath}
\usepackage{amsthm}
\usepackage{amssymb}
\usepackage{xfrac}
\DeclareMathAlphabet{\mathdcal}{U}{dutchcal}{m}{n}
\usepackage{bm}  
\usepackage{scalerel}[2016/12/29]  
\usepackage{trimclip}
\usepackage{nicefrac}

\renewcommand{\algorithmiccomment}[1]{\textit{ // #1}}

\usepackage{tikz}
\usetikzlibrary{%
    calc,
    arrows,            
    automata,          
    backgrounds,       
    chains,            
    matrix,            
    shapes,            
    spy,               
    patterns,
}
\usepackage{tikz-network}

\usepackage{layouts}
\usepackage{pgfplots, pgfplotstable}
\pgfplotsset{compat=1.18}
\usepgfplotslibrary{groupplots}
\usepgfplotslibrary{statistics}
\usepgfplotslibrary{colormaps}
\usepgfplotslibrary{fillbetween}
\usepgfplotslibrary{colorbrewer}
\usetikzlibrary{pgfplots.colorbrewer}
\pgfplotsset{%
  cycle list/Set1,
  cycle multiindex* list={
    mark list*\nextlist
    Set1\nextlist
  },
}
\pgfplotsset{%
  cycle list/PuRd-9,
  cycle multiindex* list={
    mark list*\nextlist
    PuRd-9\nextlist
  },
}
\pgfplotsset{%
  cycle list/YlGnBu-9,
  cycle multiindex* list={
    mark list*\nextlist
    YlGnBu-9\nextlist
  },
}
\pgfplotsset{%
  cycle list/Reds-9,
  cycle multiindex* list={
    mark list*\nextlist
    YlGnBu-9\nextlist
  },
}
\pgfplotsset{%
  cycle list/YlOrRd-9,
  cycle multiindex* list={
    mark list*\nextlist
    YlOrRd-9\nextlist
  },
}
\pgfplotsset{/pgfplots/colormap={hot2}{[1cm]rgb255(0cm)=(0,0,0)rgb255(3cm)=(255,0,0)rgb255(6cm)=(255,255,0)rgb255(8cm)=(255,255,255)}
}
\pgfplotsset{%
  /pgfplots/ybar legend new/.style={
    /pgfplots/legend image code/.code={
      \draw[##1,/tikz/.cd,bar width=3pt,yshift=-0.3em,bar shift=0pt]
        plot coordinates {(0cm,0.6em) (2*\pgfplotbarwidth,0.4em)};
    },
  },
}
\pgfplotsset{%
  /pgfplots/const legend/.style={
    /pgfplots/legend image code/.code={
      \draw[mark repeat=2,mark phase=2,line width=1pt,##1]
        plot coordinates {
          (0cm,-0.1cm)
          (0cm,0.0cm)
          (0.1cm,0.1cm)
          (0.2cm,0.1cm)
        };
    },
  },
}
\usepgfplotslibrary{external}
\tikzexternaldisable%
\renewcommand{\tikzexternalenable}{\tikzexternaldisable}

\usepackage{pdflscape} 
\usepackage{rotating} 

\definecolor{Colors-A}{RGB}{230,159,0}  
\definecolor{Colors-C}{RGB}{86,180,233}  
\definecolor{Colors-B}{RGB}{0,158,115}  
\definecolor{Colors-D}{RGB}{213,94,0}  
\definecolor{Colors-E}{RGB}{204,121,167}  
\definecolor{Colors-F}{RGB}{0,114,178}  
\definecolor{Colors-G}{RGB}{240,228,66}  
\definecolor{Input-Color}{RGB}{229,158,197}
\definecolor{Param-Color}{RGB}{239,231,119}
\definecolor{Output-Color}{RGB}{141,176,252}
\colorlet{Class1}{Colors-A}
\colorlet{Class2}{Colors-B}
\definecolor{loss-color}{HTML}{045A8D}
\definecolor{alternative-loss-color}{HTML}{A50F15}
\definecolor{success-color}{HTML}{31A354}
\definecolor{failure-color}{HTML}{A50F15}
\definecolor{unknown-color}{HTML}{08519C}
\definecolor{markup-color}{HTML}{195CA8}

\definecolor{VersusComOurs}{RGB}{255,88,51}
\definecolor{VersusComTheirs}{RGB}{60,89,252}

\DeclareRobustCommand{\ms}{
  \mathrel{\mathpalette\short@leftarrow\relax}%
}
\newcommand{\short@leftarrow}[2]{%
  \mkern2mu
  \clipbox{0 0 {.5\width} 0}{$\m@th#1\vphantom{+}{\boldsymbol{\leftarrow}}$}%
} 
\makeatother

\newcommand\transp[1]{{
    #1^{\mkern-1.5mu\mathsf{T}}
}}

\newcommand{\Reals}{\mathbb{R}}

\newcommand{\RealsNonNeg}{\mathbb{R}_{\geq{} 0}}

\newcommand{\Nats}{\mathbb{N}}
\newcommand{\NatsZero}{\mathbb{N}_{0}}

\DeclareMathOperator*{\argmax}{arg\,max}

\DeclareMathOperator*{\vol}{vol}
\DeclareMathOperator*{\diag}{diag}
\renewcommand{\vec}[1]{\mathbf{#1}}
\newcommand{\varvec}[1]{\boldsymbol{#1}}
\newcommand{\mat}[1]{\mathbf{#1}}

\newcommand{\powerset}[1]{2^{#1}}

\newcommand{\ka}{\mathdcal{k}}

\NewDocumentCommand{\squarebracketed}{m}{%
  {%
    \mathchoice%
    {\left[#1\right]}%
    {[#1]}%
    {\left[#1\right]}%
    {\left[#1\right]}%
  }
}

\NewDocumentCommand{\inp}{o}{%
  \vec{x}^{\IfNoValueF{#1}{(#1)}}
}%
\NewDocumentCommand{\NN}{o}{%
  \mathsf{net}_{\IfNoValueF{#1}{#1}}%
}%
\NewDocumentCommand{\fSAT}{s o}{%
  \IfBooleanTF{#1}{g}{f}_{\mathrm{Sat}}^{\IfNoValueF{#2}{(#2)}}%
}%
\NewDocumentCommand{\fSATprime}{s}{%
  \IfBooleanTF{#1}{g}{f}'_{\mathrm{Sat}}%
}%
\NewDocumentCommand{\prob}{o o d()}{%
  \mathbb{P}_{\IfNoValueF{#1}{#1}}^{\IfNoValueF{#2}{(#2)}}\IfNoValueF{#3}{\squarebracketed{#3}}%
}%
\NewDocumentCommand{\probdensity}{o o d()}{%
  f_{\IfNoValueF{#1}{#1}}^{\IfNoValueF{#2}{(#2)}}\IfNoValueF{#3}{\squarebracketed{#3}}%
}%

\NewDocumentCommand{\pterm}{s o}{%
  \prob[\inp[#2]](\fSAT*[#2](\inp[#2], \NN(\inp[#2])) \IfBooleanTF{#1}{>}{\geq} 0)%
}
\NewDocumentCommand{\parg}{s o}{%
  \fSAT*[#2](\inp[#2], \NN(\inp[#2])) \IfBooleanTF{#1}{>}{\geq} 0%
}
\NewDocumentCommand{\pfun}{o}{%
  \fSAT*[#1](\inp[#1], \NN(\inp[#1]))%
}

\NewDocumentCommand{\Xsat}{s s o d<>}{%
  {\IfBooleanTF{#1}{\IfBooleanTF{#2}{\tilde{\mathcal{X}}}{\hat{\mathcal{X}}}}{\mathcal{X}}}_{sat}^{\IfNoValueF{#3}{(#3)}\IfNoValueF{#4}{#4}}%
}
\NewDocumentCommand{\Xsatp}{s}{%
  {\IfBooleanTF{#1}{\hat{\mathcal{X}}}{\mathcal{X}}}_{sat}'%
}
\NewDocumentCommand{\Xviol}{s s o d<>}{%
  {\IfBooleanTF{#1}{\IfBooleanTF{#2}{\tilde{\mathcal{X}}}{\hat{\mathcal{X}}}}{\mathcal{X}}}_{viol}^{\IfNoValueF{#3}{(#3)}\IfNoValueF{#4}{#4}}%
}
\NewDocumentCommand{\Xviolp}{s}{%
  {\IfBooleanTF{#1}{\hat{\mathcal{X}}}{\mathcal{X}}}_{viol}'%
}

\NewDocumentCommand{\branch}{o}{%
  \mathcal{B}_{\IfNoValueF{#1}{#1}}%
}
\NewDocumentCommand{\HR}{o m}{%
  [\underline{#2}\IfNoValueF{#1}{#1}, \overline{#2}\IfNoValueF{#1}{#1}]%
}

\NewDocumentCommand{\pospart}{m}{%
  \squarebracketed{#1}^+
}
\NewDocumentCommand{\negpart}{m}{%
  \squarebracketed{#1}^-
}
\NewDocumentCommand{\ReLU}{m}{\pospart{#1}}

\NewDocumentCommand{\ulvec}{m}{\underline{\vec{#1}}}
\NewDocumentCommand{\olvec}{m}{\overline{\vec{#1}}}
\NewDocumentCommand{\ulvarvec}{m}{\underline{\varvec{#1}}}
\NewDocumentCommand{\olvarvec}{m}{\overline{\varvec{#1}}}
\NewDocumentCommand{\ulmat}{m}{\underline{\mat{#1}}}
\NewDocumentCommand{\olmat}{m}{\overline{\mat{#1}}}

\NewDocumentCommand{\ulx}{}{\underline{\inp}}
\NewDocumentCommand{\olx}{}{\overline{\inp}}
\NewDocumentCommand{\uly}{}{\ulvec{y}}
\NewDocumentCommand{\oly}{}{\olvec{y}}
\NewDocumentCommand{\ulz}{}{\ulvec{z}}
\NewDocumentCommand{\olz}{}{\olvec{z}}
\NewDocumentCommand{\ulv}{}{\ulvec{v}}
\NewDocumentCommand{\olv}{}{\olvec{v}}
\NewDocumentCommand{\ulw}{}{\ulvec{w}}
\NewDocumentCommand{\olw}{}{\olvec{w}}

\NewDocumentCommand{\uld}{}{\ulvec{d}}
\NewDocumentCommand{\old}{}{\olvec{d}}
\NewDocumentCommand{\ulalpha}{}{\ulvarvec{\alpha}}
\NewDocumentCommand{\olalpha}{}{\olvarvec{\alpha}}

\NewDocumentCommand{\olbeta}{}{\olvarvec{\beta}}
\NewDocumentCommand{\ulA}{}{\ulmat{A}}
\NewDocumentCommand{\olA}{}{\olmat{A}}

\newcommand{\legendcolorbox}[1]{{%
  \begin{tikzpicture}%
    \draw[fill=#1,draw=black,very thin] (0cm,-0.1cm) rectangle (0.2cm,0.1cm);%
  \end{tikzpicture}%
}}

\newcommand{\legendmixedline}[2]{{%
  \begin{tikzpicture}%
    \draw[line width=6pt, #1] (0.0cm,0.1cm) -- (0.4cm,0.0cm);
    \begin{scope}[overlay]
      \clip (-0.2cm, -0.2cm) -- (-0.2cm,0.1cm) -- (0.0cm,0.1cm) -- (0.4cm,0.0cm) -- (0.6, 0.0cm) -- (0.6cm,-0.2cm) -- cycle;
      \draw[line width=7pt, #2] (0.0cm,0.1cm) -- (0.4cm,0.0cm); 
    \end{scope}
    \draw[black,thick] (0cm,0.1cm) -- (0.4cm,0.0cm);
  \end{tikzpicture}%
}}

\newcommand{\exsuccessnc}{\scaleobj{1.2}{\boldsymbol{\checkmark}}}
\newcommand{\exfailurenc}{\scaleobj{1.15}{\times}}

\newcommand{\exsuccess}{\textcolor{success-color}{\exsuccessnc}}
\newcommand{\exfailure}{\textcolor{failure-color}{\exfailurenc}}
\newcommand{\exunknown}{\scaleobj{0.9}{?}}

\NewDocumentCommand{\AlgorithmName}{m}{\textrm{\textsc{#1}}}
\NewDocumentCommand{\ToolName}{s}{\IfBooleanTF{#1}{\AlgorithmName{ProbabilisticVerification} (\AlgorithmName{PV})}{\AlgorithmName{PV}}}

\NewDocumentCommand{\ProbBounds}{}{\AlgorithmName{ProbBounds}}
\NewDocumentCommand{\IntervalArithmetic}{s}{\AlgorithmName{IntervalArithmetic}\IfBooleanT{#1}{~\cite{MooreKearfottCloud2009}}}
\NewDocumentCommand{\CROWN}{s}{\AlgorithmName{CROWN}\IfBooleanT{#1}{~\cite{ZhangWengChenEtAl2018}}}
\NewDocumentCommand{\BaBSB}{s}{\AlgorithmName{BaBSB}\IfBooleanT{#1}{~\cite{BunelLuTurkaslanEtAl2020}}}
\NewDocumentCommand{\FairSquare}{}{\AlgorithmName{FairSquare}}
\NewDocumentCommand{\ProVeSLR}{}{\textrm{\textsc{ProVe\_SLR}}}
\NewDocumentCommand{\eProVe}{}{\(\varepsilon\)\textrm{\textsc{-ProVe}}}
\NewDocumentCommand{\SpaceScanner}{}{\AlgorithmName{SpaceScanner}}
\NewDocumentCommand{\PreimgApprox}{}{\AlgorithmName{PreimgApprox}}
\NewDocumentCommand{\Select}{}{\AlgorithmName{Select}}
\NewDocumentCommand{\Prune}{}{\AlgorithmName{Prune}}

\NewDocumentCommand{\ComputeBounds}{}{\AlgorithmName{ComputeBounds}}
\NewDocumentCommand{\Split}{}{\AlgorithmName{Split}}
\NewDocumentCommand{\Prob}{}{\AlgorithmName{SelectProb}}

\NewDocumentCommand{\LongestEdge}{}{\AlgorithmName{LongestEdge}}
\NewDocumentCommand{\CIE}{s}{\IfBooleanTF{#1}{\AlgorithmName{CROWN Interval Extension} (\AlgorithmName{CIE})}{\AlgorithmName{CIE}}}
\NewDocumentCommand{\BaBSBLongestEdge}{m}{\AlgorithmName{BaBSB-LongestEdge}-#1}

\NewDocumentCommand{\textcite}{m}{\citet{#1}}%
\NewDocumentCommand{\Textcite}{m}{\Citet{#1}}%

\usepackage[capitalize,noabbrev]{cleveref}
\theoremstyle{plain}
\newtheorem{theorem}{Theorem}[section]
\newtheorem{proposition}[theorem]{Proposition}
\newtheorem{lemma}[theorem]{Lemma}
\newtheorem{corollary}[theorem]{Corollary}
\theoremstyle{definition}
\newtheorem{defn}[theorem]{Definition}
\newtheorem{assumption}[theorem]{Assumption}
\newtheorem{example}[theorem]{Example}
\newtheorem*{example*}{Example}
\theoremstyle{remark}

\icmltitlerunning{%
  Solving Probabilistic Verification Problems of Neural Networks using~Branch~and~Bound%
}

\begin{document}

\twocolumn[
\icmltitle{%
  Solving Probabilistic Verification Problems of Neural Networks using~Branch~and~Bound%
}




\begin{icmlauthorlist}
\icmlauthor{David Boetius}{kn}
\icmlauthor{Stefan Leue}{kn}
\icmlauthor{Tobias Sutter}{kn}
\end{icmlauthorlist}

\icmlaffiliation{kn}{Department of Computer and Information Science, University of Konstanz, Konstanz, Baden-Würtemberg, Germany}

\icmlcorrespondingauthor{David Boetius}{david.boetius@uni-konstanz.de}

\icmlkeywords{Machine Learning, ICML, Safe Machine Learning, Trustworthy Machine Learning, Fairness, Deep Learning, Fair Machine Learning, Probabilistic Verification, Neural Network Verification, Formal Verification, Probabilistic Specification, Branch and Bound, FairSquare}

\vskip 0.3in
]



\printAffiliationsAndNotice{}  

\begin{abstract}
  Probabilistic verification problems of neural networks are concerned with formally analysing the output distribution of a neural network under a probability distribution of the inputs.
  Examples of probabilistic verification problems include verifying the demographic parity fairness notion or quantifying the safety of a neural network.
  We present a new algorithm for solving probabilistic verification problems of neural networks based on an algorithm for computing and iteratively refining lower and upper bounds on probabilities over the outputs of a neural network.
  By applying state-of-the-art bound propagation and branch and bound techniques from non-probabilistic neural network verification, our algorithm significantly outpaces existing probabilistic verification algorithms, reducing solving times for various benchmarks from the literature from tens of minutes to tens of seconds.
  Furthermore, our algorithm compares favourably even to dedicated algorithms for restricted probabilistic verification problems. 
  We complement our empirical evaluation with a theoretical analysis, proving that our algorithm is sound and, under mildly restrictive conditions, also complete when using a suitable set of heuristics.
\end{abstract}

\section{Introduction}\label{sec:intro}
As deep learning spreads through society, it becomes increasingly important to ensure the reliability of artificial neural networks, including aspects of fairness and safety.
However, manually introspecting neural networks is infeasible due to their non-transparent nature.
Furthermore, empirical assessments of neural networks are challenged by neural networks being fragile with respect to various types of input perturbations~\cite{SzegedyZarembaSutskeverEtAl2014,HosseiniXiaoPoovendran2017,BibiAlfadlyGhanem2018,EbrahimiRaoLowdEtAl2018,HendrycksZhaoBasartEtAl2021}.
In contrast, neural network verification analyses neural networks with mathematical rigour, facilitating the faithful auditing of neural networks.

In this paper, we consider probabilistic verification problems of neural networks, which are concerned with proving statements about the output distribution of a neural network given a distribution of the inputs.
We refer to solving probabilistic verification problems as \emph{probabilistic verification}.
An example of probabilistic verification is proving that a neural network~\(\NN\) making a binary decision affecting a person (for example, hire/do not hire, credit approved/denied) satisfies the demographic parity fairness notion~\cite{BarocasHardtNarayanan2023} under a probability distribution~\(\prob[\inp]\) of the network inputs~\(\inp\) representing the person
\begin{equation}
  \frac{%
    \prob[\inp](\NN(\inp) = \texttt{yes} \mid \inp\ \text{is disadvantaged})
  }{%
    \prob[\inp](\NN(\inp) = \texttt{yes} \mid \inp\ \text{is advantaged})
  } \geq \gamma,\label{eqn:demographic-parity-intro}
\end{equation}
where \enquote{\(\inp\ \text{is disadvantaged}\)} could, for example, refer to a person not being male and~\(\gamma \in [0, 1]\) with~\(\gamma = 0.8\) being a common choice~\cite{FeldmanFriedlerMoellerEtAl2015}.
A closely related problem to probabilistic verification is computing bounds on probabilities over a neural network.
An example of this is quantifying the safety of a neural network by bounding
\begin{equation}
  \prob[\vec{x}](\NN(\vec{x})\ \text{is unsafe}).\label{eqn:safety-intro}
\end{equation}
In this paper, we introduce a novel algorithm for computing bounds on probabilities such as \cref{eqn:safety-intro} using a branch and bound framework~\cite{LandDoig2010,MorrisonJacobsonSauppeEtAl2016,BunelLuTurkaslanEtAl2020}.
These bounds allow us to verify probabilistic statements like \cref{eqn:demographic-parity-intro} using bound propagation~\cite{MooreKearfottCloud2009,AlbarghouthiDAntoniDrewsEtAl2017}.
Our algorithm \ToolName*{} is a fast and generally applicable probabilistic verification algorithm for neural networks based on massively parallel branch and bound neural network verification~\cite{XuZhangWangEtAl2021} using linear relaxations of neural networks~\cite{WengZhangChenEtAl2018,ZhangWengChenEtAl2018,SinghGehrPueschelEtAl2019}.
Our theoretical analysis shows that \ToolName{} is sound and, under mildly restrictive conditions, complete when using suitable branching and splitting heuristics.

Our experimental evaluation reveals that \ToolName{} significantly outpaces the probabilistic verification algorithms \FairSquare{}~\cite{AlbarghouthiDAntoniDrewsEtAl2017} and \SpaceScanner{}~\cite{ConverseFilieriGopinathEtAl2020}. 
In particular, \ToolName{} solves benchmark instances that \FairSquare{} can not solve within 15 minutes in less than one minute and solves the ACAS~Xu~\cite{KatzBarrettDillEtAl2017} probabilistic robustness case study of~\textcite{ConverseFilieriGopinathEtAl2020} in a mean runtime of 22 seconds, compared to 33 minutes for \SpaceScanner{}.
 
Applying \ToolName{} to \#DNN verification~\cite{MarzariCorsiCicaleseEtAl2023}, a subset of probabilistic verification, reveals that \ToolName{} also compares favourably to specialized algorithms, such as \ProVeSLR{}~\cite{MarzariRoncolatoFarinelli2023}. 
It even compares favourably to \eProVe{}~\cite{MarzariCorsiMarchesiniEtAl2024} that relaxes \#DNN verification to computing a confidence interval on the solution and \PreimgApprox{}~\cite{ZhangWangKwiatkowska2024} that only computes sampling approximations.
In contrast to this, \ToolName{} computes lower and upper bounds on probabilities like \cref{eqn:safety-intro} that are guaranteed to hold with absolute certainty. 
Such bounds are preferable to confidence intervals in high-risk machine-learning applications, such as medical applications or autonomous driving and flight.

To test the limits of \ToolName{}, we introduce a significantly more challenging probabilistic verification benchmark:
The MiniACSIncome benchmark is based on the ACSIncome dataset~\cite{DingHardtMillerEtAl2021} and is concerned with verifying the demographic parity of neural networks for datasets of increasing input dimensionality.
MiniACSIncome provides more complex input distributions of higher dimensionality than earlier probabilistic verification benchmarks.
Our main contributions are
\begin{itemize}
  \item the \ToolName{} algorithm for the probabilistic verification of neural networks,
  \item a theoretical analysis proving the soundness and completeness of \ToolName{},
  \item a thorough experimental comparison of \ToolName{} with existing probabilistic verifiers for neural networks and tools dedicated to restricted subsets of probabilistic verification, and
  \item MiniACSIncome: a new, challenging probabilistic verification benchmark.
\end{itemize}
%
Our code is available at \url{https://github.com/sen-uni-kn/probspecs}.
MiniACSIncome is available as a Python package at \url{https://pypi.org/project/miniacsincome/}.

\section{Related Work}\label{sec:discussion}\label{sec:related-work}
\begin{table*}
  \centering
  \caption{%
    Comparison of Related Approaches. 
    Sound approaches provide definite guarantees.
    Probably sound approaches only provide PAC-type guarantees.  
    Approaches marked \enquote{unsound\textsuperscript{\(\ast\)}} are sound in specific settings but neither sound nor probably sound in general.
  }\label{tab:related-work}
  \vspace*{0.1in}
  \begin{tabular}{L{.39\textwidth}@{}C{.12\textwidth}@{\hspace{.025\textwidth}}C{.055\textwidth}@{\hspace{.02\textwidth}}C{.075\textwidth}@{\hspace{.02\textwidth}}C{.10\textwidth}C{.13\textwidth}}
    \toprule
    & \multicolumn{4}{c}{\textbf{Verification Problem}} \\ \cmidrule(lr){2-5}
    \textbf{Approach} & \textbf{Non-\newline Probabilistic} & \textbf{\#DNN} & \textbf{Group Fairness} & \textbf{General\newline Probabilistic} & \textbf{Verifier\newline Guarantee} \\ \midrule
    \Textcite{ZhouBrixHanasusantoEtAl2024,TranYangLopezEtAl2020}
    & \(\exsuccessnc\) & \(\exfailurenc\) & \(\exfailurenc\) & \(\exfailurenc\) & sound
    \\
    \Textcite{WengChenNguyenEtAl2019}
    & \(\exsuccessnc\) & \(\exfailurenc\) & \(\exfailurenc\) & \(\exfailurenc\) & probably sound
    \\ \midrule
    \Textcite{BorcaTasciucGuoBakEtAl2023}
    & \(\exfailurenc\) & \(\exsuccessnc\) & \(\exfailurenc\) & \(\exfailurenc\) & unsound
    \\
    \Textcite{ZhangWangKwiatkowska2024}
    & \(\exfailurenc\) & \(\exsuccessnc\) & \(\exfailurenc\) & \(\exfailurenc\) & unsound\textsuperscript{\(\ast\)}
    \\
    \Textcite{ConverseFilieriGopinathEtAl2020}
    & \(\exfailurenc\) & \(\exsuccessnc\) & \(\exsuccessnc\) & \(\exsuccessnc\) & unsound\textsuperscript{\(\ast\)}
    \\
    \Textcite{BalutaShenShindeEtAl2019,MarzariCorsiCicaleseEtAl2023,MarzariCorsiMarchesiniEtAl2024}
    & \(\exfailurenc\) & \(\exsuccessnc\) & \(\exfailurenc\) & \(\exfailurenc\) & probably sound
    \\
    \Textcite{BastaniZhangSolarLezama2019}
    & \(\exfailurenc\) & \(\exfailurenc\) & \(\exsuccessnc\) & \(\exfailurenc\) & probably sound
    \\ 
    \Textcite{MarzariRoncolatoFarinelli2023}
    & \(\exfailurenc\) & \(\exsuccessnc\) & \(\exfailurenc\) & \(\exfailurenc\) & sound
    \\
    \Textcite{AlbarghouthiDAntoniDrewsEtAl2017}
    & \(\exfailurenc\) & \(\exfailurenc\) & \(\exsuccessnc\) & \(\exfailurenc\) & sound
    \\
    \Textcite{MorettinPasseriniSebastiani2024}
    & \(\exfailurenc\) & \(\exsuccessnc\) & \(\exsuccessnc\) & \(\exsuccessnc\) & sound
    \\
    Ours (\ToolName{})
    & \(\exfailurenc\) & \(\exsuccessnc\) & \(\exsuccessnc\) & \(\exsuccessnc\) & sound
    \\ \bottomrule
  \end{tabular}
\end{table*}
%
Approaches for non-probabilistic neural network verification include Satisfiability Modulo Theories (SMT) solving~\cite{KatzBarrettDillEtAl2017}, Mixed Integer Linear Programming (MILP)~\cite{TjengXiaoTedrake2019,ChengNuehrenbergRuess2017}, and reachability analysis~\cite{BakTranHobbsEtAl2020,TranBakXiangEtAl2020,TranYangLopezEtAl2020}.
Many of these approaches can be understood as branch and bound algorithms~\cite{BunelLuTurkaslanEtAl2020}.
Branch and bound~\cite{LandDoig2010,MorrisonJacobsonSauppeEtAl2016} also powers the \AlgorithmName{\(\alpha,\!\beta\)-CROWN}~\cite{ZhouBrixHanasusantoEtAl2024} verifier that leads the table in recent international neural network verifier competitions~\cite{BrixBakLiuEtAl2023,BrixBakJohnsonEtAl2024}.
A critical component of a branch and bound verification algorithm is computing bounds on the output of a neural network.
Approaches for bounding neural network outputs include interval arithmetic~\cite{PulinaTacchella2010}, dual approaches~\cite{WongKolter2018}, and linear bound propagation techniques~\cite{WengZhangChenEtAl2018,SinghGehrPueschelEtAl2019,XuZhangWangEtAl2021}, such as \CROWN*{}.

Probabilistic verification algorithms can be divided into \emph{sound} algorithms that provide valid proofs, \emph{probably sound} algorithms that provide valid proofs only with a certain predefined probability, and \emph{unsound} algorithms that do not quantify their probability of providing invalid proofs.
Probably sound algorithms provide similar guarantees as \emph{probably approximately correct (PAC)} learning~\cite{Bishop2007}.
Fairness verification is a subset of probabilistic verification that studies problems such as \cref{eqn:demographic-parity-intro}. 
Another subset of probabilistic verification is \#DNN verification~\cite{MarzariCorsiCicaleseEtAl2023}, corresponding to probabilistic verification under uniformly distributed inputs.
\Cref{tab:related-work} provides an overview of approaches for neural network verification.

By sacrificing soundness, probably sound verification algorithms~\cite{BastaniZhangSolarLezama2019,BalutaShenShindeEtAl2019,WengChenNguyenEtAl2019,MarzariCorsiCicaleseEtAl2023,MarzariCorsiMarchesiniEtAl2024} obtain efficiency.
One example is \eProVe{}~\cite{MarzariCorsiMarchesiniEtAl2024}, a probably sound \#DNN verification algorithm.
\Textcite{BastaniZhangSolarLezama2019},~\textcite{ConverseFilieriGopinathEtAl2020}, and~\textcite{MarzariRoncolatoFarinelli2023} compare sound and probably sound approaches for probabilistic verification.
We study sound algorithms since certainly valid results are preferable in high-risk applications.

Concerning sound approaches, \FairSquare{}~\cite{AlbarghouthiDAntoniDrewsEtAl2017} is a sound fairness verification algorithm that partitions the input space into disjoint hyperrectangles and iteratively refines the input space partitioning using SMT solving.
\ProVeSLR{}~\cite{MarzariRoncolatoFarinelli2023} is a sound \#DNN verifier based on a massively parallel branch and bound algorithm.
\Textcite{ConverseFilieriGopinathEtAl2020}~(\SpaceScanner{}) and~\textcite{BorcaTasciucGuoBakEtAl2023} divide the input space into disjoint polytopes using concolic execution and reachable set verification, respectively.
Both approaches are unsound for fairness verification due to approximating continuous probability distributions using histograms.
\SpaceScanner{} is sound for \#DNN verification.
\PreimgApprox{}~\cite{ZhangWangKwiatkowska2024} divides the input space into disjoint polytopes using ReLU branching~\cite{BunelLuTurkaslanEtAl2020}, but is unsound due to using sampling to approximate probabilities.
However, it offers a post-verification soundness check for low-dimensional verification problems.
\Textcite{TranChoiOkamotoEtAl2023} provide a sound verifier based on reachability analysis that is only applicable for truncated Gaussian input distributions.
\Textcite{MorettinPasseriniSebastiani2024} use weighted model integration~\cite{BellePasseriniBroeck2015} to obtain a general sound probabilistic verification algorithm but neither provide code nor report runtimes. 

From the above approaches, \FairSquare{} and \ProVeSLR{} are most closely related to our algorithm \ToolName{}.
However, \ToolName{} is more general than \FairSquare{} and \ProVeSLR{}, which are restricted to fairness verification and \#DNN verification, respectively.
Like \FairSquare{}, \ToolName{} iteratively refines bounds on probabilities to verify probabilistic statements like \cref{eqn:demographic-parity-intro}.
However, while \FairSquare{} uses expensive SMT solving for refining the input space, we use computationally inexpensive input splitting and bound propagation techniques from non-probabilistic neural network verification.
\ProVeSLR{} builds on similar techniques but computes probabilities exactly instead of iteratively refining bounds.
These differences allow \ToolName{} to significantly outpace both \FairSquare{} and \ProVeSLR{}, as demonstrated in \cref{sec:experiments}.

A related problem to probabilistic verification of neural networks is verifying Bayesian Neural Networks (BNNs)~\cite{CardelliKwiatkowskaLaurentiEtAl2019a,CardelliKwiatkowskaLaurentiEtAl2019b,WickerLaurentiPataneEtAl2020,WickerPataneLaurentiEtAl2024,BerradaDathathriDvijothamEtAl2021,AdamsParaneLahijanianEtAl2023,BattenHosseiniLomuscio2024}. 
In BNNs, the network parameters follow a probability distribution~\cite{Neal1996}.
Since neural networks typically have vastly more parameters than inputs, BNN verification is concerned with much higher dimensional probability distributions than we consider in this paper.
Our restriction to deterministic neural networks allows us to provide a sound and complete yet practically scalable probabilistic verification algorithm.

Dependency fairness~\cite{GalhotraBrunMeliou2017,UrbanChristakisWuestholzEtAl2020} is a non-probabilistic individual fairness notion~\cite{DworkHardtPitassiEtAl2012}.
Approaches for verifying dependency fairness include~\cite{RuossBalunovicFischerEtAl2020,UrbanChristakisWuestholzEtAl2020,BiswasRajan2023,MohammadiSivaramanFarnadi2023,KimWangWang2024}.
We are concerned with probabilistic fairness notions.


\section{Preliminaries and Problem Statement}\label{sec:prelim}
Throughout this paper, we are concerned with computing (provable) lower and upper bounds.
\begin{defn}[Bounds]
For~\(f: \Reals^n \to \Reals^m\), we call~\(\ell, u \in \Reals^m\) a \emph{lower}, respectively, \emph{upper bound} on~\(f\) for~\(\mathcal{X}' \subseteq \Reals^n\) if~\(\ell \leq f(\vec{x}) \leq u\), for all~\(\vec{x} \in \mathcal{X}'\).
\end{defn}

\textbf{Neural Networks.} In particular, we are concerned with computing bounds on neural networks~\(\NN: \mathcal{X} \to \Reals^m\), where~\(\mathcal{X} \subseteq \Reals^n\) is the input space of the neural network.
A neural network is a composition of linear functions and a predefined set of non-linear functions, such as ReLU and max pooling.
We only study Lipschitz continuous neural networks, which includes many common architectures~\cite{SzegedyZarembaSutskeverEtAl2014,RuanHuangKwiatkowska2018}.
Besides this, we refrain from defining neural networks further, as our approach is not specific to any particular architecture.
%

\textbf{Notation and Terminology.}
We use~\([\underline{\vec{x}}, \overline{\vec{x}}] = \{\vec{x} \in \Reals^n \mid \underline{\vec{x}} \leq \vec{x} \leq \overline{\vec{x}}\}\) to denote hyperrectangles and write~\([n] = \{1, \ldots, n\}\) for~\(n \in \Nats\).
The term \emph{bounds} generally refers to a pair of a lower and an upper bound in this paper.
We assume that all random objects are defined on the same abstract probability space~$(\Omega, \mathcal{F}, \prob)$ and that all continuous random variables admit a probability density function.

\subsection{Probabilistic Verification of Neural Networks}\label{sec:prob-specs}
Let~\(\mathcal{X} \subseteq \Reals^n\) be a  (potentially unbounded) hyperrectangle and let~\(v \in \Nats\).
We are concerned with proving or disproving whether a neural network~\(\NN: \mathcal{X} \to \Reals^m\) is feasible for the \emph{probabilistic verification problem}
\begin{equation}
  \left\{\!\!
  \begin{array}{l}
    \fSAT(p_1, \ldots, p_v) \geq 0,\\
    p_i = \prob[\inp[i]](
      \fSAT*[i](\inp[i], \NN(\inp[i])) \geq 0
    ) \;\; \forall i \in [\mathrlap{v],}
  \end{array}
  \right.\label{eqn:pvp}
\end{equation}
where~\(\inp[i]\),~\(i \in [v]\), is an~$\mathcal{X}$-valued random variable with distribution~\(\prob[{\inp[i]}]\) and~\(\fSAT: \Reals^v \to \Reals\),~\(\fSAT*[i]: \Reals^{n} \times \Reals^{m} \to \Reals\),~\(i \in [v]\) are \emph{satisfaction functions} that are compositions of linear functions, multiplication, division, and monotone functions.
\Cref{sec:specs} expresses \cref{eqn:demographic-parity-intro} in the form of \cref{eqn:pvp} as a concrete example.

Probabilistic verification with a single uniformly distributed random variable corresponds to \#DNN verification~\cite{MarzariCorsiCicaleseEtAl2023}.
As~\textcite{MarzariCorsiCicaleseEtAl2023} prove, \#DNN verification is \#P~complete, implying that probabilistic verification is \#P~hard.
However, this does not determine which probabilistic verification problems are practically solvable.

\subsection{Non-Probabilistic Neural Network Verification}\label{sec:nn-verif}

Non-probabilistic neural network verification determines whether~\(\NN: \mathcal{X} \to \Reals^m\) is feasible for
\begin{equation}
    \parg{} \quad \forall \vec{x} \in \mathcal{X}',\label{eqn:nn-verif}
\end{equation}
where~\(\mathcal{X}' \subseteq \mathcal{X}\) is a bounded hyperrectangle, and~\(\fSAT*:\Reals^{n} \times \Reals^{m} \to \Reals\) is a \emph{satisfaction function} that indicates whether the output of~\(\NN\) is desirable (\(\fSAT*(\cdot, \NN(\cdot)) \geq 0\)) or undesirable (\(\fSAT*(\cdot, \NN(\cdot)) < 0\)).
In neural network verification,~\(\fSAT*\) can generally be considered a part of~\(\NN\)~\cite{BunelLuTurkaslanEtAl2020,XuShiZhangEtAl2020b}.
\emph{Neural network verifiers} are algorithms for proving or disproving \cref{eqn:nn-verif}.
Two desirable properties of neural network verifiers are \emph{soundness} and \emph{completeness}.
%
\begin{defn}[Soundness and Completeness]\label{defn:sound}\label{defn:complete}\label{defn:sound-complete}
  A verification algorithm is \emph{sound} if it only produces genuine counterexamples and valid proofs for \cref{eqn:nn-verif}.
  It is \emph{complete} if it produces a counterexample or proof for \cref{eqn:nn-verif} for any neural network in a finite amount of time.
\end{defn}
Analogous notions of soundness and completeness also apply to probabilistic verification.
\Cref{sec:related-work} discusses \emph{probably sound} verification.

\textbf{Interval Arithmetic.}\label{sec:interval-arithmetic}
Interval arithmetic~\cite{MooreKearfottCloud2009} is a bound propagation technique that derives bounds on the output of a neural network from bounds on the network input. 
Assume~\(\underline{\inp} \leq \inp \leq \overline{\inp}\) are bounds on the network input~\(\inp\).
We apply interval arithmetic to compute~\(\ell \leq \fSAT*(\inp, \NN(\inp)) \leq u\),~\(\forall \inp \in \HR{\inp}\).
If the lower bound is large enough or the upper bound small enough, we can prove or disprove \cref{eqn:nn-verif} using~\(\ell \geq 0 \Longrightarrow \fSAT*(\inp, \NN(\inp)) \geq 0\), respectively,~\(u < 0 \Longrightarrow \fSAT*(\inp, \NN(\inp)) < 0\).
However, if the bounds are \emph{inconclusive}, that is~\(\ell < 0 \leq u\), we can neither prove nor disprove \cref{eqn:nn-verif}. 
Therefore, interval arithmetic is incomplete according to \cref{defn:sound-complete}.
%

Let~\(f: \Reals^n \to \Reals\) be a function that we want to bound for inputs~\(\inp \in \HR{\inp}\).
Interval arithmetic and other bound propagation techniques rely on~\(f = f^{(K)} \circ \cdots \circ f^{(1)}\) being a composition of more fundamental functions~\(f^{(k)}: \Reals^{n_k} \to \Reals^{n_{k+1}}\) for which we can already compute~\(\ell^{(k)} \leq f^{(k)}(\vec{z}) \leq u^{(k)}\) given~\(\underline{\vec{z}} \leq \vec{z} \leq \overline{\vec{z}}\).
Examples of such functions include monotone non-decreasing functions, for which~\(f^{(k)}(\underline{\vec{z}}) \leq f^{(k)}(\vec{z}) \leq f^{(k)}(\overline{\vec{z}})\).
This \emph{bounding rule} already allows us to bound addition,~\(\min\),~\(\max\), and ReLU, among others.
\Cref{sec:interval-arithmetic-extra} contains further bounding rules.
Given a bounding rule~\(F^{(k)}: \Reals^{n_k} \times \Reals^{n_k} \to \Reals^{n_{k+1}} \times \Reals^{n_{k+1}}\) for each~\(f^{(k)}\), interval arithmetic computes bounds on~\(f\) by computing~\((F^{(K)} \circ \cdots \circ F^{(1)})(\underline{\inp}, \overline{\inp})\).

\textbf{Branch and Bound.}\label{sec:branch-and-bound}
Bound propagation approaches, such as interval arithmetic and \CROWN*{}, are incomplete according to \cref{defn:sound-complete}.
To obtain a complete verifier, bound propagation can be combined with \emph{branching} to gain completeness.
This algorithmic framework is called \emph{branch and bound}~\cite{LandDoig2010,MorrisonJacobsonSauppeEtAl2016}.
In branch and bound, the search space is split (\emph{branching}) when the computed bounds are inconclusive (\(\ell < 0 \leq u\)).
The idea is that splitting improves the precision of the bounds for each part of the split (each \emph{branch}).
\Textcite{BunelLuTurkaslanEtAl2020} provide an introduction to branch and bound for non-probabilistic neural network verification.
%
The next section introduces our branch and bound algorithm for probabilistic neural network verification.
\section{Algorithm}\label{sec:algo}

\begin{figure*}
  \centering
  \tikzexternalenable%
  \tikzset{%
    side by side/.style 2 args={%
      line width=8pt,
      #1,
      postaction={%
          clip,postaction={draw,#2}
      }
    },
    class A/.style={%
      font=\Large\sffamily,
    },
    class B/.style={%
      font=\Large\sffamily,
      text=white,
    },
    class unknown/.style={%
      font=\Large\bfseries,
    },
  }
  \contourlength{0.05em}  
  \begin{subfigure}[t]{.275\textwidth}
    \tikzsetnextfilename{prob-bounds-illustration-four-splits}
    \begin{tikzpicture}
      \begin{axis}[
        width=\textwidth,
        axis equal image,
        scale only axis,
        xlabel=$\vec{x}_1$, ylabel=$\vec{x}_2$,
        xmin=0, xmax=10,
        ymin=0, ymax=10,
        xtick=\empty, ytick=\empty,
      ]
        \coordinate (D0)  at (-1,7.0);
        \coordinate (D1)  at (-0.1,7.45);
        \coordinate (D2)  at (1.6,8.2);
        \coordinate (D3)  at (4.0,8.5);
        \coordinate (D4)  at (4.6,8.1);
        \coordinate (D5)  at (4.75,7.3);
        \coordinate (D6)  at (4.65,6.3);
        \coordinate (D7)  at (4.4,4.7);
        \coordinate (D8)  at (4.5,3.8);
        \coordinate (D9)  at (4.7,2.6);
        \coordinate (D10) at (5.1,2.0);
        \coordinate (D11) at (5.5,1.3);
        \coordinate (D12) at (5.8,-0.2);
        \coordinate (D13) at (6.0,-1);
  
        \coordinate (S1_0) at (5,0);
        \coordinate (S1_1) at (5,10);
        \coordinate (S2_0) at (0,5);
        \coordinate (S2_1) at (5,5);
        \coordinate (S3_0) at (2.5,0);
        \coordinate (S3_1) at (2.5,5);
        \coordinate (S4_0) at (5,5);
        \coordinate (S4_1) at (10,5);
  
  
        \draw[fill=Class1] (0,0) -- (0,5) -- (2.5,5) -- (2.5,0) -- cycle;
        \draw[fill=Class2] (5,5) -- (5,10) -- (10,10) -- (10,5) -- cycle;
        \draw[fill=black!10!white] (0,5) -- (0,10) -- (5,10) -- (5,5) -- cycle;
        \draw[fill=black!10!white] (2.5,0) -- (2.5,5) -- (5,5) -- (5,0) -- cycle;
        \draw[fill=black!10!white] (5,0) -- (5,5) -- (10,5) -- (10,0) -- cycle;
  
        \draw[gray,thin,step=1.25] (0,0) grid (10,10);
  
        \draw[side by side={Class2}{Class1}] (D0) -- (D1) -- (D2) -- (D3) -- (D4) -- (D5) -- (D6) -- (D7) -- (D8) -- (D9) -- (D10) -- (D11) -- (D12) -- (D13);
        \draw[black,thick] (D0) -- (D1) -- (D2) -- (D3) -- (D4) -- (D5) -- (D6) -- (D7) -- (D8) -- (D9) -- (D10) -- (D11) -- (D12) -- (D13);
        
        \draw[black,very thick] 
          (S1_0) -- (S1_1)
          (S2_0) -- (S2_1)
          (S3_0) -- (S3_1)
          (S4_0) -- (S4_1)
        ;
  
        \node[class A] at (1.25,2.5) {V};
        \node[class B] at (7.5,7.5) {S};
        \node [pin={[pin edge={black,semithick,Stealth-},font=\small,align=center,xshift=-1.1cm,yshift=-1.95cm]{Satisfaction\\ Boundary}}] at ($(D4) + (.09cm,.09cm)$) {};
      \end{axis}
    \end{tikzpicture}
    \caption{%
      After four Splits.
    }%
    \label{fig:probability-bounds-four-splits}
  \end{subfigure}
  \qquad
  \begin{subfigure}[t]{.275\textwidth}
    \tikzsetnextfilename{prob-bounds-illustration-nineteen-splits}
    \begin{tikzpicture}
      \begin{axis}[
        width=\textwidth,
        axis equal image,
        scale only axis,
        xlabel=$\vec{x}_1$, ylabel=$\vec{x}_2$,
        xmin=0, xmax=10,
        ymin=0, ymax=10,
        xtick=\empty, ytick=\empty,
      ]
        \coordinate (D0)  at (-1,7.0);
        \coordinate (D1)  at (-0.1,7.45);
        \coordinate (D2)  at (1.6,8.2);
        \coordinate (D3)  at (4.0,8.5);
        \coordinate (D4)  at (4.6,8.1);
        \coordinate (D5)  at (4.75,7.3);
        \coordinate (D6)  at (4.65,6.3);
        \coordinate (D7)  at (4.4,4.7);
        \coordinate (D8)  at (4.5,3.8);
        \coordinate (D9)  at (4.7,2.6);
        \coordinate (D10) at (5.1,2.0);
        \coordinate (D11) at (5.5,1.3);
        \coordinate (D12) at (5.8,-0.2);
        \coordinate (D13) at (6.0,-1);
  
        \coordinate (S1_0) at (5,0);
        \coordinate (S1_1) at (5,10);
        \coordinate (S2_0) at (0,5);
        \coordinate (S2_1) at (5,5);
        \coordinate (S3_0) at (2.5,0);
        \coordinate (S3_1) at (2.5,5);
        \coordinate (S4_0) at (5,5);
        \coordinate (S4_1) at (10,5);
        \coordinate (S6_0) at (7.5,0);
        \coordinate (S6_1) at (7.5,5);
        \coordinate (S7_0) at (5,2.5);
        \coordinate (S7_1) at (10,2.5);
        \coordinate (S8_0) at (5,2.5);
        \coordinate (S8_1) at (10,2.5);
        \coordinate (S9_0) at (6.25,0);
        \coordinate (S9_1) at (6.25,2.5);
        \coordinate (S11_0) at (2.5,2.5);
        \coordinate (S11_1) at (5,2.5);
        \coordinate (S12_0) at (3.75,2.5);
        \coordinate (S12_1) at (3.75,5);
        \coordinate (S13_0) at (2.5,5);
        \coordinate (S13_1) at (2.5,10);
        \coordinate (S14_0) at (0,7.5);
        \coordinate (S14_1) at (2.5,7.5);
        \coordinate (S15_0) at (2.5,7.5);
        \coordinate (S15_1) at (5,7.5);
        \coordinate (S16_0) at (3.75,0);
        \coordinate (S16_1) at (3.75,2.5);
        \coordinate (S17_0) at (3.75,5);
        \coordinate (S17_1) at (3.75,7.5);
        \coordinate (S18_0) at (2.5,8.75);
        \coordinate (S18_1) at (5,8.75);
        \coordinate (S19_0) at (0,8.75);
        \coordinate (S19_1) at (2.5,8.75);
  
  
        \draw[fill=Class1] (0,0) -- (0,5) -- (2.5,5) -- (2.5,0) -- cycle;
        \draw[fill=Class1] (2.5,2.5) -- (2.5,5) -- (3.75,5) -- (3.75,2.5) -- cycle;
        \draw[fill=Class1] (2.5,0) -- (2.5,2.5) -- (3.75,2.5) -- (3.75,0) -- cycle;
        \draw[fill=black!10!white] (3.75,2.5) -- (3.75,5) -- (5,5) -- (5,2.5) -- cycle;
        \draw[fill=black!10!white] (3.75,0) -- (3.75,2.5) -- (5,2.5) -- (5,0) -- cycle;
        
        \draw[fill=Class1] (0,5) -- (0,7.5) -- (2.5,7.5) -- (2.5,5) -- cycle;
        \draw[fill=Class1] (2.5,5) -- (2.5,7.5) -- (3.75,7.5) -- (3.75,5) -- cycle;
        \draw[fill=Class2] (0,8.75) -- (0,10) -- (2.5,10) -- (2.5,8.75) -- cycle;
        \draw[fill=Class2] (2.5,8.75) -- (2.5,10) -- (5,10) -- (5,8.75) -- cycle;
        \draw[fill=black!10!white] (0,7.5) -- (0,8.75) -- (2.5,8.75) -- (2.5,7.5) -- cycle;
        \draw[fill=black!10!white] (2.5,7.5) -- (2.5,8.75) -- (5,8.75) -- (5,7.5) -- cycle;
        \draw[fill=black!10!white] (3.75,5) -- (3.75,7.5) -- (5,7.5) -- (5,5) -- cycle;
        
        \draw[fill=Class2] (5,5) -- (5,10) -- (10,10) -- (10,5) -- cycle;
        
        \draw[fill=Class2] (5,2.5) -- (5,5) -- (7.5,5) -- (7.5,2.5) -- cycle;
        \draw[fill=Class2] (7.5,2.5) -- (7.5,5) -- (10,5) -- (10,2.5) -- cycle;
        \draw[fill=Class2] (7.5,0) -- (7.5,2.5) -- (10,2.5) -- (10,0) -- cycle;
        \draw[fill=Class2] (6.25,0) -- (6.25,2.5) -- (7.5,2.5) -- (7.5,0) -- cycle;
        \draw[fill=black!10!white] (5,0) -- (5,2.5) -- (6.25,2.5) -- (6.25,0) -- cycle;
  
        \draw[gray,thin,step=1.25] (0,0) grid (10,10);
  
        \draw[side by side={Class2}{Class1}] (D0) -- (D1) -- (D2) -- (D3) -- (D4) -- (D5) -- (D6) -- (D7) -- (D8) -- (D9) -- (D10) -- (D11) -- (D12) -- (D13);
        \draw[black,thick] (D0) -- (D1) -- (D2) -- (D3) -- (D4) -- (D5) -- (D6) -- (D7) -- (D8) -- (D9) -- (D10) -- (D11) -- (D12) -- (D13);
        
        \draw[black,very thick] 
          (S1_0) -- (S1_1)
          (S2_0) -- (S2_1)
          (S3_0) -- (S3_1)
          (S4_0) -- (S4_1)
          (S6_0) -- (S6_1)
          (S7_0) -- (S7_1)
          (S8_0) -- (S8_1)
          (S9_0) -- (S9_1)
          (S11_0) -- (S11_1)
          (S12_0) -- (S12_1)
          (S13_0) -- (S13_1)
          (S14_0) -- (S14_1)
          (S15_0) -- (S15_1)
          (S16_0) -- (S16_1)
          (S17_0) -- (S17_1)
          (S18_0) -- (S18_1)
          (S19_0) -- (S19_1)
        ;
  
        \node[class A] at (1.25,2.5) {V};
        \node[class A] at (3.125,3.75) {V};
        \node[class A] at (3.125,1.25) {V};
        \node[class A] at (1.25,6.25) {V};
        \node[class A] at (3.125,6.25) {V};
        
        \node[class B] at (7.5,7.5) {S};
        \node[class B] at (6.25,3.75) {S};
        \node[class B] at (6.25,3.75) {S};
        \node[class B] at (8.75,3.75) {S};
        \node[class B] at (8.75,1.25) {S};
        \node[class B] at (6.875,1.25) {S};
        \node[class B] at (1.25,9.375) {S};
        \node[class B] at (3.75,9.375) {S};
        
      \end{axis}
    \end{tikzpicture}  
    \caption{After 17 Splits.}\label{fig:probability-bounds-nineteen-splits}
  \end{subfigure}
  \qquad
  \begin{subfigure}[t]{.275\textwidth}
    \tikzsetnextfilename{prob-bounds-illustration-density}
    \begin{tikzpicture}
      \begin{axis}[
        width=\textwidth,
        axis equal image,
        scale only axis,
        xlabel=$\vec{x}_1$, ylabel=$\vec{x}_2$,
        xmin=0, xmax=10,
        ymin=0, ymax=10,
        xtick=\empty, ytick=\empty,
        view={0}{90},
        colormap/PuBu,
      ]
        \coordinate (D0)  at (-1,7.0);
        \coordinate (D1)  at (-0.1,7.45);
        \coordinate (D2)  at (1.6,8.2);
        \coordinate (D3)  at (4.0,8.5);
        \coordinate (D4)  at (4.6,8.1);
        \coordinate (D5)  at (4.75,7.3);
        \coordinate (D6)  at (4.65,6.3);
        \coordinate (D7)  at (4.4,4.7);
        \coordinate (D8)  at (4.5,3.8);
        \coordinate (D9)  at (4.7,2.6);
        \coordinate (D10) at (5.1,2.0);
        \coordinate (D11) at (5.5,1.3);
        \coordinate (D12) at (5.8,-0.2);
        \coordinate (D13) at (6.0,-1);
  
        \coordinate (S1_0) at (5,0);
        \coordinate (S1_1) at (5,10);
        \coordinate (S2_0) at (0,5);
        \coordinate (S2_1) at (5,5);
        \coordinate (S3_0) at (2.5,0);
        \coordinate (S3_1) at (2.5,5);
        \coordinate (S4_0) at (5,5);
        \coordinate (S4_1) at (10,5);
        \coordinate (S6_0) at (7.5,0);
        \coordinate (S6_1) at (7.5,5);
        \coordinate (S7_0) at (5,2.5);
        \coordinate (S7_1) at (10,2.5);
        \coordinate (S8_0) at (5,2.5);
        \coordinate (S8_1) at (10,2.5);
        \coordinate (S9_0) at (6.25,0);
        \coordinate (S9_1) at (6.25,2.5);
        \coordinate (S11_0) at (2.5,2.5);
        \coordinate (S11_1) at (5,2.5);
        \coordinate (S12_0) at (3.75,2.5);
        \coordinate (S12_1) at (3.75,5);
        \coordinate (S13_0) at (2.5,5);
        \coordinate (S13_1) at (2.5,10);
        \coordinate (S14_0) at (0,7.5);
        \coordinate (S14_1) at (2.5,7.5);
        \coordinate (S15_0) at (2.5,7.5);
        \coordinate (S15_1) at (5,7.5);
        \coordinate (S16_0) at (3.75,0);
        \coordinate (S16_1) at (3.75,2.5);
        \coordinate (S17_0) at (3.75,5);
        \coordinate (S17_1) at (3.75,7.5);
        \coordinate (S18_0) at (2.5,8.75);
        \coordinate (S18_1) at (5,8.75);
        \coordinate (S19_0) at (0,8.75);
        \coordinate (S19_1) at (2.5,8.75);
  
        
        \addplot3 [surf,samples=30,domain=0:10,shader=interp] {exp(-0.5*(((x-8)/5)^2 + ((y-8)/4)^2))/(sqrt(2*pi*5*4)};
  
        \draw[gray,thin,step=1.25] (0,0) grid (10,10);
  
        \draw[side by side={Class2}{Class1}] (D0) -- (D1) -- (D2) -- (D3) -- (D4) -- (D5) -- (D6) -- (D7) -- (D8) -- (D9) -- (D10) -- (D11) -- (D12) -- (D13);
        \draw[black,thick] (D0) -- (D1) -- (D2) -- (D3) -- (D4) -- (D5) -- (D6) -- (D7) -- (D8) -- (D9) -- (D10) -- (D11) -- (D12) -- (D13);
        
        \draw[black,very thick] 
          (S1_0) -- (S1_1)
          (S2_0) -- (S2_1)
          (S3_0) -- (S3_1)
          (S4_0) -- (S4_1)
          (S6_0) -- (S6_1)
          (S7_0) -- (S7_1)
          (S8_0) -- (S8_1)
          (S9_0) -- (S9_1)
          (S11_0) -- (S11_1)
          (S12_0) -- (S12_1)
          (S13_0) -- (S13_1)
          (S14_0) -- (S14_1)
          (S15_0) -- (S15_1)
          (S16_0) -- (S16_1)
          (S17_0) -- (S17_1)
          (S18_0) -- (S18_1)
          (S19_0) -- (S19_1)
        ;
  
        \node[class A] at (1.25,2.5)   {V};
        \node[class A] at (3.125,3.75) {V};
        \node[class A] at (3.125,1.25) {V};
        \node[class A] at (1.25,6.25)  {V};
        \node[class A] at (3.125,6.25) {V};
        
        \node[class B] at (7.5,7.5)   {\contour{black}{S}};
        \node[class B] at (6.25,3.75)  {\contour{black}{S}};
        \node[class B] at (6.25,3.75)  {\contour{black}{S}};
        \node[class B] at (8.75,3.75)  {\contour{black}{S}};
        \node[class B] at (8.75,1.25)  {\contour{black}{S}};
        \node[class B] at (6.875,1.25) {\contour{black}{S}};
        \node[class B] at (1.25,9.375) {\contour{black}{S}};
        \node[class B] at (3.75,9.375) {\contour{black}{S}};
      \end{axis}
    \end{tikzpicture}  
    \caption{Probability Density Function.}%
    \label{fig:probability-bounds-probability-mass}
  \end{subfigure}

  \tikzexternaldisable%
  \caption[Computing Bounds on Probabilities]{%
    Computing bounds on probabilities.
    This figure illustrates the steps for computing bounds on~\(p = \pterm\).
    Our algorithm successively splits the input space to find regions that do not intersect the satisfaction boundary~\(\parg\) (\mbox{orange/green} line~\legendmixedline{Class2}{Class1}).
    Green, orange, and grey rectangles (\legendcolorbox{Class2}/\legendcolorbox{Class1}/\legendcolorbox{black!10!white}) denote regions for which we could prove~\(\parg\) (satisfaction)~\legendcolorbox{Class2},~\(\fSAT*(\inp, \NN(\inp)) < 0\) (violation)~\legendcolorbox{Class1}, or neither~\legendcolorbox{black!10!white}, respectively.
    By integrating the probability density~\(\probdensity[\inp]\) in \subref{fig:probability-bounds-probability-mass} (darker means higher density) over the green rectangles~\legendcolorbox{Class2}, we obtain a lower bound on~\(p\).
    Similarly, we can integrate over the orange rectangles~\legendcolorbox{Class1} to construct an upper bound on~\(p\).
    Refining the input splitting from \subref{fig:probability-bounds-four-splits} to \subref{fig:probability-bounds-nineteen-splits} tightens the bounds on~\(p\).
  }%
  \label{fig:probability-bounds-illustration}
\end{figure*}

This section introduces~\ToolName*{}, our algorithm for probabilistic verification of neural networks as defined in \cref{eqn:pvp}.
The overall approach of \ToolName{} is to iteratively refine bounds on each~\(p_i\) from \cref{eqn:pvp} until a bound propagation approach allows us to prove or disprove~\(\fSAT(p_1, \ldots, p_v) \geq 0\).
For computing the bounds on~\(p_i\), \ToolName{} partitions the input space into hyperrectangles since this allows for computing probabilities efficiently~\cite{AlbarghouthiDAntoniDrewsEtAl2017}.
To compute tighter bounds on~\(p_i\), \ToolName{} uses a branch and bound algorithm that uses computationally inexpensive input splitting and bound propagation techniques from non-probabilistic neural network verification for refining the input splitting.

\Cref{algo:overall} describes the \ToolName{} algorithm.
The centrepiece of \ToolName{} is the procedure \ProbBounds{} for computing bounds~\(\ell_i^{(t)} \leq p_i \leq u_i^{(t)}\) on a probability~\(p_i\) from \cref{eqn:pvp}.
Given~\(\ell_i^{(t)} \leq p_i \leq u_i^{(t)}\), we apply a bound propagation technique to prove or disprove~\(\fSAT(p_1, \ldots, p_v)\), as described in \cref{sec:interval-arithmetic}.
If this analysis is inconclusive, \ProbBounds{} refines~\(\ell_i^{(t)}, u_i^{(t)}\) to obtain~\(\ell_i^{(t+1)}, u_i^{(t+1)}\) with~\(\ell_i^{(t)} \leq \ell_i^{(t+1)} \leq p_i \leq u_i^{(t+1)} \leq u_i^{(t)}\).
We again apply bound propagation to~\(\fSAT(p_1, \ldots, p_v)\), this time using~\(\ell_i^{(t+1)}, u_i^{(t+1)}\).
If the result remains inconclusive, we iterate refining the bounds on each~\(p_i\) until we obtain a conclusive result.
\ToolName{} applies \ProbBounds{} for each~\(p_i\) in parallel, making use of several CPU cores or several GPUs.
Our main contribution is the \ProbBounds{} algorithm for computing a converging sequence of lower and upper bounds on~\(p_i\).


\begin{algorithm}
  \caption{\ToolName{}}\label{algo:overall}
  \newcommand{\pbvar}{\textsc{P\!\!\;B}}
  \begin{algorithmic}[1]
    \REQUIRE Probabilistic Verification Problem as in \cref{eqn:pvp}, Batch Size \(N\)
    \FOR[Launch \(v\) parallel instances]{$i \in [v]$} 
      \STATE \(\pbvar_i \gets \text{Launch}\ \ProbBounds(p_i, N) \)
    \ENDFOR
    \FOR{$t \in \mathbb{N}$}
      \STATE \textbf{for} \(i \in [v]\) \textbf{do} Gather \(b_i^{(t)} = (\ell_i^{(t)}, u_i^{(t)})\) from \(\pbvar_i\)
      \STATE \((\ell^{(t)},\!u^{(t)})\!\gets\!\ComputeBounds(\fSAT,b_1^{\mathrlap{(t)}}, \ldots, b_v^{(t)})\)
      \STATE \textbf{if} \(\ell^{(t)} \geq 0\) \textbf{then} \textbf{return} \texttt{Satisfied}
      \STATE \textbf{if} \(u^{(t)} < 0\) \textbf{then} \textbf{return} \texttt{Violated}
    \ENDFOR
  \end{algorithmic}
\end{algorithm}

\subsection{Bounding Probabilities}\label{sec:prob-bounds}
Our \ProbBounds{} algorithm for deriving and refining bounds on a probability is described in detail in \cref{algo:bound-prob} and illustrated in \cref{fig:probability-bounds-illustration}.
\ProbBounds{} is a massively parallel input-splitting branch and bound algorithm~\cite{BunelLuTurkaslanEtAl2020,WangPeiWhitehouseEtAl2018b,XuShiZhangEtAl2020b} that leverages a bound propagation algorithm for non-probabilistic neural network verification (\ComputeBounds{}).
Since we only consider a single probability in this section, we denote this probability as~\(p = \pterm\).

\ProbBounds{} receives~\(p\) and a \emph{batch size}~\(N \in \Nats\) as input.
The algorithm iteratively computes~\(\ell^{(t)}, u^{(t)} \in [0, 1]\), such that~\(\ell^{(t)} \leq \ell^{(t')} \leq p \leq u^{(t')} \leq u^{(t)}\),~\(\forall t, t' \in \Nats, t' \geq t\).
The following sections describe each step of \ProbBounds{} in detail. 


\begin{algorithm}[t]
  \caption{\ProbBounds{}}\label{algo:bound-prob}
  \begin{algorithmic}[1]
    \REQUIRE Probability \(\pterm\), Batch Size~\(N\)
    \STATE \(\mathrm{branches} \gets \{\mathcal{X}\}\) 
    \STATE \(\ell^{(0)} \gets 0, u^{(0)} \gets 1\)
    \FOR{$t \in \mathbb{N}$}
      \STATE \(\mathrm{batch} \gets \Select(\mathrm{branches}, N)\) 
      \STATE \((\underline{\vec{y}}, \overline{\vec{y}})\!\gets\!\ComputeBounds(\fSAT*(\cdot, \NN(\cdot)), \mathrm{batch})\)
      \STATE \((\mathrm{keep}, \Xsat[t], \Xviol[t]) \gets \Prune(\mathrm{batch}, \underline{\vec{y}}, \overline{\vec{y}})\)
      \STATE \(\ell^{(t)} \gets \ell^{(t-1)} + \prob[\inp](\Xsat[t])\)
      \STATE \(u^{(t)} \gets u^{(t-1)} - \prob[\inp](\Xviol[t])\)
      \STATE \textbf{yield} \((\ell^{(t)}, u^{(t)})\) \COMMENT{Report new bounds to \ToolName{}}
      \STATE \(\mathrm{new} \gets \Split(\mathrm{keep})\) 
      \STATE \(\mathrm{branches} \gets (\mathrm{branches} \setminus \mathrm{batch}) \cup \mathrm{new}\)
    \ENDFOR
  \end{algorithmic}
\end{algorithm}

\textbf{Initialisation.}
Initially, we consider a single branch encompassing \(\NN\)'s entire input space~\(\mathcal{X}\).
As in \cref{sec:prelim}, we assume~\(\mathcal{X}\) to be a (potentially unbounded) hyperrectangle.
We use the trivial bounds~\(\ell^{(0)} = 0 \leq p \leq 1 = u^{(0)}\) as initial bounds on~\(p\).

\textbf{Selecting Branches.}
First, we select a batch of~\(N \in \Nats\) branches.
In the spirit of~\textcite{XuZhangWangEtAl2021}, we leverage the data parallelism of modern CPUs and GPUs to process several branches at once.
In iteration~\(t=1\), the batch only contains the branch~\(\mathcal{X}\).
Which branches we select determines how fast we obtain tight bounds on~\(p\).
We propose the \Prob{} heuristic for selecting branches. Inspired by \FairSquare{}~\cite{AlbarghouthiDAntoniDrewsEtAl2017}, \Prob{} selects the~\(N\) branches~\(\branch[i]\) with the largest~\(\prob[\inp](\branch[i])\).
This heuristic is motivated by the observation that pruning these branches would lead to the largest improvement of~\(\ell^{(t)}, u^{(t)}\).

\textbf{Pruning.}
The next step is to prune those branches~\(\branch[j] \in \mathrm{batch}\), for which we can determine that~\(y = \parg\) is either certainly satisfied or certainly violated.
For this, we first compute~\(\underline{\vec{y}} \leq \fSAT*(\cdot, \NN(\cdot)) \leq \overline{\vec{y}}\) for the entire~\(\mathrm{batch}\) using a bound propagation algorithm for neural networks, such as \CROWN*{}.
%
If~\(\underline{\vec{y}}_j \geq 0\) (\(\vec{y}_j \geq 0\) is certainly satisfied) or~\(\overline{\vec{y}}_j < 0\) (\(\vec{y}_j \geq 0\) is certainly violated), we can prune~\(\branch[j]\), meaning that we remove it from~\(\mathrm{branches}\).
We collect the branches with~\(\underline{\vec{y}}_j \geq 0\) in the set~\(\Xsat[t]\) and the branches with~\(\overline{\vec{y}}_j < 0\) in the set~\(\smash{\Xviol[t]}\), where~\(t \in \Nats\) is the current iteration.

\textbf{Updating Bounds.}
Let~\(\Xsat*[t] = \bigcup_{t'=1}^{t} \Xsat[t']\) and~\(\Xviol*[t] = \bigcup_{t'=1}^{t} \Xviol[t']\), where~\(t \in \Nats\) is the current iteration.
Then,~\(\ell^{(t)} = \prob[\inp](\Xsat*[t]) \leq p\).
Similarly,~\(\ka^{(t)} = \prob[\inp](\Xviol*[t]) \leq \prob[\inp](\fSAT*(\inp, \NN(\inp)) < 0) = 1 - p\).
Therefore,~\(1 - \ka^{(t)} = u^{(t)} \geq p\).
Practically, we only have to maintain the current bounds~\(\ell^{(t)}\) and~\(u^{(t)}\) instead of the sets~\(\Xsat*[t]\) and~\(\Xviol*[t]\).
%

Because~\(\Xsat*[t]\) and~\(\Xviol*[t]\) are a union of disjoint hyperrectangles, exactly computing~\(\prob[\inp](\Xsat*[t])\) and~\(\prob[\inp](\Xviol*[t])\) is feasible for a large class of probability distributions, including most univariate distributions, Mixture Models, and Bayesian Networks.
The precise class of supported probability distributions is discussed in \cref{sec:algo-limitations}.
While we do not account for floating point errors in this paper, our approach can readily be extended to this end.
%

\textbf{Splitting.}\label{sec:splitting}
Splitting refines a branch~\(\branch = \HR{\inp}\) by selecting a dimension~\(d \in [n]\) to split.
We split based on the type of variable that is encoded in~\(d\).
For bounded continuous variables, we split by bisecting~\(\HR{\inp}\) along~\(d\).
For unbounded variables, we split at zero if~\(-\ulx_d = \olx_d = \infty\), at~\(\max(2\ulx_d, 1)\) if~\(-\ulx_d < \olx_d = \infty\), and at~\(\min(2\olx_d, -1)\) if~\(-\infty = \ulx_d < -\olx_d\).
For integer variables, we additionally round the split points to the next smaller, respectively, larger integer.
For dimensions containing a binary indicator of a one-hot encoded categorical variable~\(A\), we jointly split all indicators of~\(A\) such that~\(A\) takes on the value encoded in~\(d\) in one branch and does not take on this value in the other branch.
\Cref{sec:split-details} defines these splitting rules formally.

\textbf{Split Selection.}\label{sec:split-select}
We present three heuristics for selecting the dimension~\(d\) for splitting.
We generally select dimensions encoding unbounded variables first in order to obtain bounded branches, since \ComputeBounds{} usually computes vacuous bounds for unbounded branches.
For bounded variables, the well-known \LongestEdge{} heuristic~\cite{BunelLuTurkaslanEtAl2020} selects the dimension with the largest \emph{edge length}~\(\overline{\inp}_d - \underline{\inp}_d\).
Alternatively, we use a variant of the \BaBSB{} heuristic~\cite{BunelLuTurkaslanEtAl2020}.
\BaBSB{} estimates the improvement in bounds that splitting dimension~\(d\) yields by using a yet less expensive technique than \ComputeBounds{}.
Our variant of \BaBSB{} uses \IntervalArithmetic{}, assuming that we use \CROWN{} for \ComputeBounds{}.
\Cref{sec:babsb} describes our \BaBSB{} variant in detail.
While \LongestEdge{} is more theoretically accessible, \BaBSB{} is practically advantageous, as discussed in \cref{sec:heuristics-additional}.
Combining the advantages of both approaches, we introduce \BaBSBLongestEdge{k}. This heuristic alternates using \BaBSB{} and \LongestEdge{}, using \LongestEdge{} for every k-th split.
If we visualise branches and their descendants from splitting in a branching tree, the splits at level~\(k, 2k, 3k, \ldots\) use \LongestEdge{} while the splits at all other levels use \BaBSB{}.

\subsection{Input Spaces and Input Distributions}\label{sec:algo-limitations}
This section discusses the concrete requirements that the input space~\(\mathcal{X}\) and the input distributions~\(\prob[\inp[i]]\) in \cref{eqn:pvp} need to satisfy for applying \ToolName{}. 
\Cref{sec:mitigating-algo-limitations} discusses how some of these requirements can be mitigated, for example, for using polytopes as input spaces.

\ToolName{} requires~\(\mathcal{X} \subseteq \Reals^n\) to be a hyperrectangle, which can be unbounded.
The dimensions of~\(\mathcal{X}\) may encode discrete random variables.
For each probability distribution~\(\prob[\inp[i]]\), we require a terminating algorithm that computes the exact probability of a hyperrectangle.
This requirement is satisfied by a large class of probability distributions, including discrete distributions with a closed-form probability mass function and univariate continuous distributions with a closed-form cumulative density function, as well as Mixture Models and probabilistic graphical models~\cite{Bishop2007}, such as Bayesian Networks, of such distributions.

\section{Theoretical Analysis}\label{sec:algo-theory}
In this section, we prove that \ToolName{} is a sound probabilistic verification algorithm when instantiated with a suitable \ComputeBounds{} procedure.
We also prove that \ToolName{} is complete under mild assumptions on the probabilistic verification problem when instantiated with suitable \Split{}, \ComputeBounds{}, and \Select{} procedures.
Soundness and completeness are defined in \cref{defn:sound-complete}.
As in \cref{sec:prob-bounds}, we omit indices and superscripts when considering only a single probability~\(p = \pterm\).
We defer all proofs to \cref{sec:proofs}.
Our first result concerns the soundness of \ProbBounds{}.
%
\begin{theorem}[Sound Bounds]\label{lem:sound-bounds}
  Let~\(N \in \Nats\) be a batch size and assume \ComputeBounds{} produces valid bounds.
  Let~\({\left\{(\ell^{(t)}, u^{(t)})\right\}}_{t \in \Nats}\) be the iterates of~\(\ProbBounds(\pterm, N)\).
  It holds that~\(\ell^{(t)} \leq \pterm \leq u^{(t)}\) for all~$t\in\mathbb{N}$.
\end{theorem}

\begin{corollary}[Soundness]\label{thm:algo-sound}
\ToolName{} is sound when using \ComputeBounds{} procedures that compute valid bounds.
\end{corollary}

Our remaining theoretical results are concerned with the completeness of \ToolName{}. Concretely, we prove that \ToolName{} instantiated with \Prob{}, \LongestEdge{} or \BaBSBLongestEdge{k}, and \IntervalArithmetic{} or \CROWN{} is complete under a mildly restrictive condition on \cref{eqn:pvp}.
\Cref{sec:proofs-completeness} defines a more general class of pruning and splitting heuristics for which \ToolName{} is complete.

\begin{assumption}\label{assume:one}
  Let~\(v\),~\(\fSAT\),~\(\fSAT*[i]\), and~\(\inp[i]\) be as in \cref{eqn:pvp}.
  Assume~\(\fSAT(p_1, \ldots, p_v) \neq 0\) and~\(\forall i \in [v]: \prob[{\inp[i]}](\pfun[i]{} = 0) = 0\).
\end{assumption}
\Cref{assume:one} is only mildly restrictive, since for every verification problem that does not satisfy \cref{assume:one}, there are similar problems that satisfy the assumption.
Consider the case that~\(\fSAT(p_1, \ldots, p_n) = 0\).
In this case, \cref{eqn:pvp} does not satisfy \cref{assume:one}.
However, a slightly stronger verification problem concerned with~\(\fSATprime(p_1, \ldots, p_n) = \fSAT(p_1, \ldots, p_n) - \varepsilon\) for an arbitrarily small~\(\varepsilon > 0\) satisfies \cref{assume:one}.
\Cref{sec:proofs-completeness} discusses \cref{assume:one} in more detail.

To prove the completeness of \ToolName{}, we first establish that \ProbBounds{} produces a sequence of lower and upper bounds that converge towards each other.

\begin{lemma}[Converging Probability Bounds]\label{thm:converging-bounds-mainbody}
  Let~\(N \in \Nats\) be a batch size.
  Let~\({\left\{(\ell^{(t)}, u^{(t)})\right\}}_{t \in \Nats}\) be the iterates of~\(\ProbBounds{}(\pterm, N)\) instantiated with \Prob{}, \LongestEdge{} or \BaBSBLongestEdge{k}, and \IntervalArithmetic{} or \CROWN{}.
  Assume~\(\prob[\inp](\fSAT*(\inp, \NN(\inp)) = 0) = 0\) as in \cref{assume:one}.
  Then, 
  \begin{equation*}
    \lim_{t \to \infty} \ell^{(t)} = \lim_{t \to \infty} u^{(t)} = \pterm.
  \end{equation*}
\end{lemma}

\begin{theorem}[Completeness]\label{thm:complete-mainbody}
  When instantiated with \ProbBounds{} as in \cref{thm:converging-bounds-mainbody} and \IntervalArithmetic{} or \CROWN{} for \ComputeBounds{}, \ToolName{} is complete for verification problems satisfying \cref{assume:one}.
\end{theorem}
Unfortunately, our completeness result does not apply to the \BaBSB{} heuristic, which provides the best empirical performance when used in \ToolName{}.
However, our result applies to \BaBSBLongestEdge{k}, which yields comparable performance as \BaBSB{}, as we show in \cref{sec:heuristics-additional-experiments}.

\section{Experiments}\label{sec:experiments}
\begin{table*}
  \centering
  \caption{%
    Our benchmarks. Network size is the size of the neural network given as \#layers$\times$layer size.
  }\label{tab:benchmarks}
  \vspace*{0.1in}
  \begin{tabular}{l@{\hspace{1em}}cllc}
    \toprule
    \textbf{Benchmark} & \textbf{Input Dimension} & \textbf{Input Distributions} & \textbf{Network Size} & \textbf{Source} \\ \midrule
    \multirow{2}*{\textbf{FairSquare}}    & \multirow{2}*{2--3} & independent         & \multirow{2}*{1$\times$1, 1$\times$2} & \multirow{2}*{\cite{AlbarghouthiDAntoniDrewsEtAl2017}} \\
                           &      & 2 Bayesian Networks & \\[.25em]
    \textbf{ACAS Xu}       & 5    & uniform             & 6$\times$50 & \cite{KatzBarrettDillEtAl2017} \\
    \textbf{VCAS}          & 4    & uniform             & 1$\times$21 & \cite{ZhangWangKwiatkowska2024} \\[.25em]
    \multirow{2}*{\textbf{MiniACSIncome}} & \multirow{2}*{1--8}  & \multirow{2}*{Bayesian Network} & 1$\times$10--10000 & \multirow{2}*{Own} \\
    & & & 1$\times$10 -- 10$\times$10 \\
    \bottomrule
  \end{tabular}
\end{table*}
We apply our algorithms to verify the demographic parity fairness notion, count the number of safety violations of neural network controllers in safety-critical systems, and quantify the robustness of a neural network.
\Cref{tab:benchmarks} gives an overview of our benchmarks.
All verification problems are defined formally in \cref{sec:specs}.
For all benchmarks, \ProbBounds{} use the \Prob{} and \BaBSB{} heuristics and \CROWN*{} for \ComputeBounds{}, while \ToolName{} uses \IntervalArithmetic{}.

As our results show, \ToolName{} (\cref{algo:overall}) outpaces the probabilistic verification algorithms \FairSquare{}~\cite{AlbarghouthiDAntoniDrewsEtAl2017} and \SpaceScanner{}~\cite{ConverseFilieriGopinathEtAl2020}.
Additionally, we show that \ProbBounds{} (\cref{algo:bound-prob}) compares favourably to the \ProVeSLR{}~\cite{MarzariRoncolatoFarinelli2023}, \eProVe{}~\cite{MarzariCorsiMarchesiniEtAl2024}, and \PreimgApprox{}~\cite{ZhangWangKwiatkowska2024} algorithms for \#DNN verification~\cite{MarzariCorsiCicaleseEtAl2023}, which corresponds to probabilistic verification with uniformly distributed inputs.

While no code is publicly available for \SpaceScanner{}, running \ProVeSLR{} is very computationally expensive.
To enable a faithful comparison, we run our experiments on less powerful hardware (HW1) compared to the hardware used by~\textcite{ConverseFilieriGopinathEtAl2020} and~\textcite{MarzariRoncolatoFarinelli2023} and compare the runtime of our algorithms to the runtimes reported by these authors.
All other results reported in this paper were obtained on HW1, including the results for \FairSquare{}, \eProVe{}, and \PreimgApprox{}.

To test the limits of \ToolName{}, we introduce a new, challenging benchmark:
MiniACSIncome is based on the ACSIncome dataset~\cite{DingHardtMillerEtAl2021}.
It consists of datasets of varying input dimensionality, probability distributions for these datasets, and neural networks trained on these datasets.
Being based on real-world US census data, MiniACSIncome offers more complex input distributions with higher input dimensionality than existing probabilistic verification benchmarks.
\ToolName{} solves seven of eight instances in MiniACS\-Income within an hour.

\textbf{Hardware and Implementation.}
We implement \ToolName{} in Python, leveraging \texttt{PyTorch}~\cite{PaszkeGrossMassaEtAl2019} and \texttt{auto\_LiRPA}~\cite{XuShiZhangEtAl2020b}. 
We run all experiments on a Ubuntu 22.04 desktop with an Intel i7--4820K CPU, 32 GB of memory, and no GPU (HW1).
\Cref{sec:details-hardware} compares our hardware to the hardware used by \textcite{ConverseFilieriGopinathEtAl2020} and~\textcite{MarzariRoncolatoFarinelli2023}.

\subsection{FairSquare Benchmark}\label{sec:fairsquare-experiment}
\Textcite{AlbarghouthiDAntoniDrewsEtAl2017} evaluate their \FairSquare{} algorithm on an application derived from the Adult dataset~\cite{AdultDataset1996}.
In particular, they verify whether three small neural networks satisfy two fairness notions with respect to a person's sex under three different distributions of the network input: a distribution of entirely independent univariate variables and two Bayesian Networks.
\Cref{sec:fairsquare-additional} describes the FairSquare benchmark in more detail.


\Cref{fig:fairsquare} compares the runtimes of \ToolName{} and \FairSquare{} on the FairSquare benchmark.
\ToolName{} significantly outperforms \FairSquare{}.
In particular, \ToolName{} solves four more instances than \FairSquare{} within the timeout of 15 minutes. 
For the instances that both tools solve, the median runtime of \ToolName{} is 4s (mean: 5s, max: 17s) compared to 44s for \FairSquare{} (mean: 109s, max: 657s).
\Cref{sec:fairsquare-additional} contains the detailed results of this experiment.
%

\begin{figure}
  \centering
  \begin{tikzpicture}
    \begin{semilogyaxis}[
      width=197.5pt,  
      height=90pt,
      scale only axis,
      ymax=900,
      xmin=1, xmax=18,
      xlabel={\# Solved Instances}, 
      ylabel={Runtime (s)},
      ytick={1,10,100,900},
      yticklabels={$10^{0}$,$10^{1}$,$10^{2}$,TO},
      ylabel shift=-5pt,
      legend pos=north west,
      legend cell align=left,
    ]
      \addplot [Colors-B,very thick,mark=x] table[x=Nr,y=Runtime,col sep=comma] {data/FairSquareAlgorithmRuntimes.csv};
      \addlegendentry{\FairSquare{}\citetalias{AlbarghouthiDAntoniDrewsEtAl2017}}

      \addplot [Colors-A,very thick,mark=x] table[x=Nr,y=Runtime,col sep=comma] {data/FairSquareRuntimeProbBaBSBCROWN.csv};
      \addlegendentry{\ToolName{} (Ours)}
    \end{semilogyaxis}
  \end{tikzpicture}
  \tikzexternaldisable%
  \caption{%
    FairSquare benchmark results.
    The timeout (TO) is 15min.
    \quad \citetalias{AlbarghouthiDAntoniDrewsEtAl2017}\textcite{AlbarghouthiDAntoniDrewsEtAl2017}
  }\label{fig:fairsquare}
\end{figure}

\subsection{Aircraft Collision Avoidance}
The ACAS~Xu networks~\cite{KatzBarrettDillEtAl2017} are a suite of 45 networks, together forming a collision avoidance system for crewless aircraft.
Each ACAS~Xu network predicts a horizontal turning direction to avoid collision with another aircraft.
VCAS~\cite{JulianKochenderfer2019} is a similar system that predicts vertical steering directions for avoiding collisions.
We reproduce the ACAS~Xu safety experiments of~\textcite{MarzariRoncolatoFarinelli2023}, the ACAS~Xu robustness experiments of~\textcite{ConverseFilieriGopinathEtAl2020}, and the VCAS correctness experiment of~\textcite{ZhangWangKwiatkowska2024}.

\textbf{ACAS~Xu Safety.}
In this experiment, we seek to \emph{quantify} the number of violations (violation rate) of several ACAS~Xu networks~\cite{KatzBarrettDillEtAl2017}.
This corresponds to computing bounds on \cref{eqn:safety-intro} under a uniform distribution of~\(\inp\).
%
%
We compare \ProbBounds{} to the \ProVeSLR{} and \eProVe{} algorithms for \#DNN verification.
\ProVeSLR{} computes the violation rate exactly, while \eProVe{} computes an upper bound on the violation rate that is sound with a certain predefined probability.
In contrast, \ProbBounds{} provides sound bounds on the violation rate at any time during its execution.

\Cref{tab:cx-counting-compare} compares \ProbBounds{} to~\ProVeSLR{} and \eProVe{} for the ACAS~Xu networks investigated by \textcite{MarzariRoncolatoFarinelli2023}.
For all three networks, \ProbBounds{} can tighten the bounds to a margin of less than~$0.7\%$ within one hour, while \ProVeSLR{} requires at least four hours to compute the exact violation rate.
In comparison to \eProVe{}, \ProbBounds{} produces tighter sound bounds within 10 seconds in two of three cases, while \eProVe{} requires at least 57~seconds to derive a probably sound upper bound for these cases.
The extended comparison in \cref{sec:acasxu-safety-additional} reveals that in 12 from a total of 36 cases, \ProbBounds{} computes a tighter sound bound faster than \eProVe{} computes a probably sound upper bound.

\begin{table*}
  \centering
  \caption{%
    Comparison of \ProbBounds{}, \ProVeSLR{}, and \eProVe{}.
    We run \ProbBounds{} with different time budgets (10s, 1m, 1h) and report the lower and upper bounds (\(\ell, u\)) computed within this time budget.
    In contrast, \ProVeSLR{} computes the exact probabilities (VR), and \eProVe{} computes a 99.9\% confidence (confid.) upper bound.
    The probabilities and probability bounds are given as percentages.
    The runtimes (Rt) of \ProVeSLR{} are taken from \textcite{MarzariRoncolatoFarinelli2023}.
  }\label{tab:cx-counting-compare}
  \vspace*{0.1in}
  \begin{tabular}{c@{\hspace{1em}}cccrrcr}
    \toprule
    & \multicolumn{3}{c}{\bfseries \ProbBounds{} (Ours)} 
    & \multicolumn{2}{c}{\bfseries\ProVeSLR{}\citetalias{MarzariRoncolatoFarinelli2023}} 
    & \multicolumn{2}{c}{\bfseries\eProVe{}\citetalias{MarzariCorsiMarchesiniEtAl2024}} 
    \\ \cmidrule(lr){2-4}\cmidrule(lr){5-6}\cmidrule(lr){7-8}
    & {\bfseries 10s} & {\bfseries 1m} & {\bfseries 1h}
    & \multicolumn{2}{c}{\bfseries Exact} 
    & \multicolumn{2}{c}{\bfseries 99.9\% confid.} 
    \\
    \(\NN\) 
    & \(\ell, u\) & \(\ell, u\) & \(\ell, u\)
    & VR & Rt & \(u\) & Rt
    \\ \midrule
    $N_{4,3}$ & $  0.17\%,   2.92\%$ & $  0.61\%,   2.27\%$ & $  1.12\%,   1.75\%$ & $1.43\%$ & 8h 46m  & $3.61\%$ & 65s \\
    $N_{4,9}$ & $  0.00\%,   3.36\%$ & $  0.00\%,   1.55\%$ & $  0.08\%,   0.29\%$ & $0.15\%$ & 12h 21m & $0.73\%$ & 20s \\
    $N_{5,8}$ & $  0.89\%,   4.16\%$ & $  1.55\%,   3.10\%$ & $  1.97\%,   2.57\%$ & $2.20\%$  & 4h 35m  & $4.52\%$ & 57s \\
    \bottomrule
  \end{tabular} \\[0.5em]
  {\small
    \citetalias{MarzariRoncolatoFarinelli2023}\citet{MarzariRoncolatoFarinelli2023}
    \quad
    \citetalias{MarzariCorsiMarchesiniEtAl2024}\citet{MarzariCorsiMarchesiniEtAl2024}
  }
\end{table*}

\textbf{ACAS~Xu Robustness.}\label{sec:experiment-acasxu-robustness}
We replicate the experiments of~\textcite{ConverseFilieriGopinathEtAl2020} who apply \SpaceScanner{} to quantify the robustness of ACAS~Xu network~\(N_{1,1}\)~\cite{KatzBarrettDillEtAl2017} under adversarial perturbations.
Overall, the experiment consists of 125 verification problems that concern the probability of obtaining a particular class for uniformly distributed perturbed inputs close to one of 25 reference input points.

The mean runtime of \ProbBounds{} for these~125 instances is~22~seconds (median:~6s, maximum:~213s).
In contrast, \textcite{ConverseFilieriGopinathEtAl2020} report a mean runtime of 33 minutes per instance for \SpaceScanner{} while running their experiments on superior hardware.
\Cref{sec:acasxu-robustness-additional} contains more details on this experiment.

\textbf{VCAS Correctness.}
\Textcite{ZhangWangKwiatkowska2024} study whether a VCAS network correctly predicts to maintain course in a scenario where there is no risk of collision.
Concretely, they verify whether the VCAS network provides the correct output at least~\(90\%\) of the time.
\ToolName{} is able to prove this within~\(0.13\)s.
In contrast, \PreimgApprox{} requires~\(16.42\)s for computing an unsound empirical lower bound on the probability of obtaining correct outputs.

\subsection{MiniACSIncome}
To test the limits of \ToolName{}, we introduce the MiniACSIncome benchmark.
MiniACSIncome is derived from the ACSIncome dataset~\cite{DingHardtMillerEtAl2021}, a replacement of the Adult dataset~\cite{AdultDataset1996} that is better suited for fair machine learning research.
The task is to predict whether a person's yearly income exceeds \$50\,000 using features such as the person's age, sex,  and education.
Our benchmark provides probabilistic verification problems of various degrees of difficulty.
We apply \ToolName{} to MiniACSIncome and compare it to a baseline approach for solving MiniACSIncome.

\textbf{Benchmark.}
To create probabilistic verification problems of increasing difficulty, we consider an increasing number of input variables from ACSIncome.
The smallest instance, MiniACSIncome-1, only contains the binary \enquote{SEX} variable. 
In contrast, the largest instance, MiniACSIncome-8, contains \enquote{SEX} and seven more variables from ACSIncome, including age, education, and working hours per week.
Our benchmark's task is to verify the demographic parity of neural networks with varying input dimension under a Bayesian Network as input distribution.
These Bayesian Networks provide complex multi-modal input distributions, as they fit the real-world US census data in ACSIncome.
\Cref{sec:miniacsincome-extra} describes MiniACSIncome in detail.
%

\textbf{Results.}
Since all variables in MiniACSIncome are discrete, a baseline approach for verifying the demographic parity of a MiniACSIncome network is to enumerate all values in the input space. 
\Cref{fig:MiniACSIncome-mainbody} displays the runtime of \ToolName{} and the baseline enumeration approach for shallow 10-neuron neural networks with increasing input size.
While enumeration is faster than \ToolName{} when the network can be evaluated for all discrete values in one batch, enumeration falls behind \ToolName{} as soon as this becomes infeasible.
\ToolName{} can solve MiniACSIncome for up to seven input variables in less than 30 minutes, only exceeding the timeout of one hour for eight input variables.
While we only consider a small network here, the runtime of \ToolName{} is largely unaffected by network size on this benchmark.
This unexpected result can be attributed to both large and small networks learning similar decision boundaries for MiniACSIncome.
\Cref{sec:miniacsincome-netsize} discusses this result in more detail.

\begin{figure}
  \centering
  \begin{tikzpicture}
    \begin{axis}[
      name=insize,
      width=245pt,  
      height=120pt,  
      xmin=1, xmax=8,
      ymin=0, ymax=3600,
      ytick={60,600,1200,1800,2400,3000,3600},
      yticklabels={1,10,20,30,40,50,TO},
      xlabel={\# Input Variables}, 
      ylabel={Runtime (min)},
      ylabel shift=-5pt,
      legend pos=north west,
      legend cell align=left,
    ]
      \addplot [Colors-C,very thick,mark=x] table[x=InputVariables,y=EnumerateRuntime,col sep=comma] {data/MiniACSIncomeInputSize.csv};
      \addlegendentry{Enumerate}
      
      \addplot [Colors-A,very thick,mark=x] table[x=InputVariables,y=VerifyRuntime,col sep=comma] {data/MiniACSIncomeInputSize.csv};
      \addlegendentry{\ToolName{} (Ours)}
    \end{axis}
  \end{tikzpicture}
  \tikzexternaldisable%
  \caption{%
    MiniACSIncome results.
    The timeout (TO) is one hour.
   }\label{fig:MiniACSIncome-mainbody}
\end{figure}

\section{Conclusion}\label{sec:conclude}
Our \ToolName{} algorithm for the probabilistic verification of neural networks significantly outpaces existing algorithms for probabilistic verification.
We achieve this speedup by applying a massively parallel branch and bound algorithm based on bound propagation algorithms for neural networks.
Our MiniACSIncome benchmark provides a challenging testbed for future probabilistic verification algorithms.
%

\newpage
%
%

\section*{Impact Statement}
This work is concerned with providing mathematical guarantees on the output distribution of a neural network given a distribution of the inputs.
Since mathematical guarantees enhance the transparency of neural networks and facilitate their faithful auditing, we anticipate that our work will have a predominantly positive societal impact.
However, obtaining an input distribution for probabilistic verification requires significant domain expertise and careful design.
For example, a poorly designed input distribution may lead to certifying an unfair classifier as fair.
Therefore, verification results are only meaningful if the concrete probabilistic verification problem that was solved is reported and made available alongside the verification result, including the input distribution. 
Ideally, verification should be conducted by a separate certification body for critical applications.

{%
  \def\UrlBreaks{\do\a\do\b\do\c\do\d\do\e\do\f\do\g\do\h\do\i\do\j\do\k\do\l\do\m\do\n\do\o\do\p\do\q\do\r\do\s\do\t\do\u\do\v\do\w\do\x\do\y\do\z\do\A\do\B\do\C\do\D\do\E\do\F\do\G\do\H\do\I\do\J\do\K\do\L\do\M\do\N\do\O\do\P\do\Q\do\R\do\S\do\T\do\U\do\V\do\W\do\X\do\Y\do\Z}

  \bibliography{references}
  \bibliographystyle{icml2025}
}

\newpage
\appendix
\onecolumn

\section{Probabilistic Verification Problems}\label{sec:specs}\label{sec:spec-extended-example}
This section contains the formal definitions of all probabilistic verification problems in this paper.

\begin{example}\label{example:demographic-parity}
  We express the demographic parity fairness notion from \cref{eqn:demographic-parity-intro} as a probabilistic verification problem.
  %
  Let~\(\mathcal{X} \subseteq \Reals^n\) be an input space that encodes information about a person, including a categorical protected attribute, such as gender, race, or disability status that is one-hot encoded at the indices~\(A \subset [n]\).
  We assume a single historically advantaged category encoded at the index~\(a \in A\).
  Consider a neural network~\(\NN: \Reals^n \to \Reals^2\) that acts as a binary classifier making a decision affecting a person, such as hiring or credit approval. 
  The neural network produces a score for each class and assigns the class with the higher score to an input.
  We rewrite \cref{eqn:demographic-parity-intro} as
  \begin{alignat*}{4}
    && \frac{
        \prob[\vec{x}](\NN(\inp) = \texttt{yes} \mid \inp\ \text{is disadvantaged})
      }{
        \prob[\vec{x}](\NN(\inp) = \texttt{yes} \mid \inp\ \text{is advantaged})
      } &\geq \gamma \\
    \Longleftrightarrow\quad
    && \frac{
        \prob[\inp]({\NN(\inp)}_1 - {\NN(\inp)}_2 \geq 0 \mid \inp_a \leq 0)
      }{
        \prob[\inp]({\NN(\inp)}_1 - {\NN(\inp)}_2 \geq 0 \mid \inp_a \geq 1)
      } &\geq \gamma \\
    \Longleftrightarrow\quad
    && \frac{
        \prob[\inp]({\NN(\inp)}_1 - {\NN(\inp)}_2 \geq 0 \wedge \inp_a \leq 0) / \prob[\inp](\inp_a \leq 0)
      }{
        \prob[\inp]({\NN(\inp)}_1 - {\NN(\inp)}_2 \geq 0 \wedge \inp_a \geq 1) / \prob[\inp](\inp_a \geq 1)
      } &\geq \gamma \\
    \Longleftrightarrow\quad
    && \frac{
        \prob[\inp](\min({\NN(\inp)}_1 - {\NN(\inp)}_2, -\inp_a) \geq 0) / \prob[\inp](-\inp_a \geq 0)
      }{
        \prob[\inp](\min({\NN(\inp)}_1 - {\NN(\inp)}_2, \inp_a - 1) \geq 0) / \prob[\inp](\inp_a - 1\geq 0)
      } &\geq \gamma \\
    \Longleftrightarrow\quad
    && \frac{
        \prob[\inp](\fSAT*[1](\inp, \NN(\inp)) \geq 0) / \prob[\inp](\fSAT*[2](\inp, \NN(\inp)) \geq 0)
      }{
        \prob[\inp](\fSAT*[3](\inp, \NN(\inp)) \geq 0) / \prob[\inp](\fSAT*[4](\inp, \NN(\inp)) \geq 0)
      } - \gamma &\geq 0 \\
    \Longleftrightarrow\quad
    && \fSAT{\left(
        \prob[\inp](\fSAT*[1](\inp, \NN(\inp)) \geq 0), 
        \ldots,
        \prob[\inp](\fSAT*[4](\inp, \NN(\inp)) \geq 0)
      \right)} &\geq 0
  \end{alignat*}
  where, \(\fSAT(p_1,p_2,p_3,p_4) = (p_1 p_4)/(p_2 p_3) - \gamma\), \(\fSAT*[1](\inp, \NN(\inp)) = \min({\NN(\inp)}_1 - {\NN(\inp)}_2, -\inp_a)\), \(\fSAT*[2](\inp, \NN(\inp))=-\inp_a\), \(\fSAT*[3](\inp, \NN(\inp)) = \min({\NN(\inp)}_1 - {\NN(\inp)}_2, \inp_a - 1)\), and \(\fSAT*[4](\inp, \NN(\inp)) = \inp_a -1\). 
\end{example}

\subsection{Parity of Qualified Persons}\label{sec:specs-parity-of-qualified}
The following probabilistic verification problem concerns verifying the parity of qualified persons, a variant of demographic parity that only considers the subpopulation of persons qualified for, for example, hiring~\cite{AlbarghouthiDAntoniDrewsEtAl2017}.
Let~\(\mathcal{X} \subseteq \Reals^n\),~\(A \subset [n]\),~\(a \in A\), and~\(\NN: \Reals^n \to \Reals^2\) be as in \cref{example:demographic-parity}.
Additionally, let~\(q \in [n] \setminus A\) and~\(\hat{q} \in \Reals\), such that persons with~\(\inp_q \geq \hat{q}\) are considered to be qualified.
In their extended set of experiments, \textcite{AlbarghouthiDAntoniDrewsEtAl2017} consider a~\(q\) that encodes age and~\(\hat{q} = 18\) so that only persons who are at least 18 years old are considered to be qualified.
The parity of qualified persons fairness notion is
\begin{alignat*}{4}
  && \frac{
      \prob[\vec{x}](\NN(\inp) = \texttt{yes} \mid \inp\ \text{is disadvantaged} \wedge \inp\ \text{is qualified})
    }{
      \prob[\vec{x}](\NN(\inp) = \texttt{yes} \mid \inp\ \text{is advantaged} \wedge \inp\ \text{is qualified})
    } &\geq \gamma \\
  \Longleftrightarrow\quad
  && \frac{
    \prob[\inp]({\NN(\inp)}_1 - {\NN(\inp)}_2 \geq 0 \mid \inp_a \leq 0 \wedge \inp_q \geq \hat{q})
    }{
      \prob[\inp]({\NN(\inp)}_1 - {\NN(\inp)}_2 \geq 0 \mid \inp_a \geq 1 \wedge \inp_q \geq \hat{q})
    } &\geq \gamma \\
  \Longleftrightarrow\quad
  && \frac{
    \prob[\inp]({\NN(\inp)}_1 - {\NN(\inp)}_2 \geq 0 \mid \min(-\inp_a, \inp_q - \hat{q}) \geq 0)
    }{
      \prob[\inp]({\NN(\inp)}_1 - {\NN(\inp)}_2 \geq 0 \mid \min(\inp_a - 1, \inp_q - \hat{q}) \geq 0)
    } &\geq \gamma \\
  \Longleftrightarrow\quad
  && \fSAT{\left(
      \prob[\inp](\fSAT*[1](\inp, \NN(\inp)) \geq 0), 
      \ldots,
      \prob[\inp](\fSAT*[4](\inp, \NN(\inp)) \geq 0)
    \right)} &\geq 0
\end{alignat*}
where~\(\gamma \in [0, 1]\),~\(\fSAT(p_1,p_2,p_3,p_4) = (p_1 p_4)/(p_2 p_3) - \gamma\), \(\fSAT*[1](\inp, \NN(\inp)) = \min({\NN(\inp)}_1 - {\NN(\inp)}_2, -\inp_a, \inp_q - \hat{q})\), \(\fSAT*[2](\inp, \NN(\inp)) = \min(-\inp_a, \inp_q - \hat{q})\), \(\fSAT*[3](\inp, \NN(\inp)) = \min({\NN(\inp)}_1 - {\NN(\inp)}_2, \inp_a - 1, \inp_q - \hat{q})\), and~\(\fSAT*[4](\inp, \NN(\inp)) = \min(\inp_a -1, \inp_q - \hat{q})\). 

\subsection{ACAS Xu Safety}
Next, we consider \cref{eqn:safety-intro} for an ACAS Xu network, where to be safe means satisfying property~\(\phi_2\) of~\textcite{KatzBarrettDillEtAl2017}.
For quantifying the number of violations, we first define what it means for an ACAS~Xu neural network~\(\NN: \Reals^5 \to \Reals^5\) to violate~\(\phi_2\).
Using the satisfaction functions of~\textcite{BauerMarquartBoetiusLeueEtAl2021}, violating~\(\phi_2\) means
\begin{equation}
  \fSAT*(\inp, \NN(\inp)) = \max_{i = 2}^{5} {\NN(\inp)}_i - {\NN(\inp)}_1 < 0 
  \quad \forall \inp \in \mathcal{X}_{\phi_2} \cap \mathcal{X},\label{eqn:acasxu-phitwo-verif}
\end{equation}
where~\(\mathcal{X}\) is the bounded hyperrectangular input space of~\(\NN\) and
\begin{equation*}
  \mathcal{X}_{\phi_2} = [55947.961, \infty] \times \Reals^2 \times [1145, \infty] \times [-\infty, 60].
\end{equation*}
We refer to \textcite{KatzBarrettDillEtAl2017} for an interpretation of~\(\phi_2\) in the application context.
Quantifying the number of violations with respect to~\(\phi_2\) corresponds to computing
\begin{equation*}
  \ell \leq \prob[\inp](\fSAT*(\inp, \NN(\inp)) < 0) = \prob[\inp](-\fSAT*(\inp, \NN(\inp)) \geq 0) \leq u,
\end{equation*}
where~\(\fSAT*\) is as in \cref{eqn:acasxu-phitwo-verif} and~\(\inp\) is uniformly distributed on~\(\mathcal{X}_{\phi_2} \cap \mathcal{X}\) with all points outside~\(\mathcal{X}_{\phi_2} \cap \mathcal{X}\) having zero probability.

\subsection{ACAS Xu Robustness}
For the ACAS~Xu robustness experiment in \cref{sec:experiment-acasxu-robustness}, we solve five probabilistic verification problems for each reference input~\(\inp\)~---~one for each of the five classes.
Our goal is to bound the probability of~\(\NN\) classifying an input~\(\vec{x}'\) as class~\(i \in [5]\), where~\(\vec{x}'\) is close to the reference input~\(\inp\) in the first two dimensions and identical to~\(\inp\) in the remaining dimensions.

Let~\(\NN\) be the ACAS~Xu network~\(N_{1,1}\) of \textcite{KatzBarrettDillEtAl2017} with input space~\(\mathcal{X} = \HR{\inp}\).
Let~\(\inp\) be a reference input.
Note that the ACAS~Xu networks assign the class with the \emph{minimal} score to an input instead of using the maximal score.
For bounding the probability of obtaining class~\(i \in [5]\) for inputs close to~\(\inp\), we compute bounds on
\begin{gather*}
  \prob[\vec{x}'](\fSAT*(\vec{x}', \NN(\vec{x}')) \geq 0), \\
  \fSAT*(\vec{x}', \NN(\vec{x}')) = \min_{\substack{j = 1\\ j\neq i}}^5 {\NN(\vec{x}')}_j - {\NN(\vec{x}')}_i
\end{gather*}
where~\(\vec{x}'\) is uniformly distributed on the set~\(
  \mathcal{X} \cap 
  \left(
    [\inp_{1:2} - 0.05\cdot\vec{w}_{1:2}, \inp_{1:2} + 0.05\cdot\vec{w}_{1:2}] 
    \times \{\inp_{3:5}\}
  \right),
\)
where~\(\vec{w} = \overline{\inp} - \underline{\inp}\) and~\(\vec{z}_{i:j}\) is the vector containing the elements~\(i, \ldots, j\) of a vector~\(\vec{z}\).

\subsection{VCAS Correctness}
In the VCAS correctness experiment in \cref{sec:experiment-acasxu-robustness}, the goal is to prove whether the VCAS network~\(\NN: \Reals^4 \to \Reals^9\) of~\textcite{ZhangWangKwiatkowska2024} satisfies
\begin{equation*}
  \prob[\vec{x}](\fSAT*(\inp, \NN(\inp)) \geq 0) \geq 0.9,
\end{equation*}
where~\(\fSAT*(\inp, \NN(\inp)) = {\NN(\inp)}_{1} - \min_{j=2}^{9} {\NN(\inp)}_{j}\), which encodes that the network predicts \enquote{Clear Of Conflict}, and~\(\inp\) is uniformly distributed on the set~\([-8000, 0] \times [0, 100] \times \{-30\} \times [0, 40]\).
We refer to~\textcite{ZhangWangKwiatkowska2024} for an interpretation of this specification in the application context.

\subsection{Useful Modelling Techniques}\label{sec:mitigating-algo-limitations}
As discussed in \cref{sec:algo-limitations}, \ToolName{} requires the input space~\(\mathcal{X}\) to be a hyperectangle.
Further, it requires that each input distribution~\(\prob[\inp]\) allows for computing the probability of a hyperrectangle in closed form.
This section shows how some of these restrictions can be mitigated.
Concretely, we show how to use multivariate normal distributions as input distributions and polytopes as input spaces.

We first show how we can apply \ToolName{} to multivariate normal distributions by transforming the input distribution and the network to verify.
Consider a multivariate normal distribution~\(\prob[\vec{z}]\) with mean~\(\varvec{\mu}\) and covariance~\(\varvec{\Sigma} = \mat{A}\transp{\mat{A}}\).
If~\(\varvec{\Sigma}\) is diagonal, the probability of a hyperrectangle has a closed-form solution, so that we can compute it efficiently.
Here, we are interested in the case where~\(\varvec{\Sigma}\) is not diagonal, so that we can not compute the probability of a hyperrectangle directly.
In this case, let~\(\prob[\inp]\) be a standard multivariate normal distribution.
Now,~\(\vec{z} = \mat{A}\inp + \varvec{\mu}\) is distributed according to~\(\prob[\vec{z}]\).
Therefore, by prepending the linear transformation~\(\mat{A}\inp + \varvec{\mu}\) to~\(\NN\), we can apply \ToolName{} to general multivariate normal distributions, since the probability of a hyperrectangle under a standard multivariate normal distribution has a closed-form solution.

Second, we show how to apply \ToolName{} to polytopal input spaces, even though they can not be used as input space directly.
Let~\(\mathcal{P} = \{\vec{x} \in \Reals^n \mid \mat{A}\vec{x} \leq \vec{b}\}\) be a polytope and let~\(\mathcal{X} \supseteq \mathcal{P}\) be a hyperrectangle enclosing~\(\mathcal{P}\).
By using~\(\mathcal{X}\) as the input space and using
\begin{equation*}
  \prob[\vec{x}](\fSAT*(\vec{x}, \NN(\vec{x})) \geq 0 \mid \mat{A}\vec{x} \leq \vec{b}) 
  = \frac{%
      \prob[\inp](\fSAT*(\vec{x}, \NN(\vec{x})) \geq 0 \land \mat{A}\vec{x} \leq \vec{b})
    }{%
      \prob[\inp](\mat{A}\vec{x} \leq \vec{b})
    }
\end{equation*}
we can apply \ToolName{} to polytopes as input spaces.
However, before applying \ToolName{}, we should check whether~\(\prob[\inp](\mat{A}\vec{x} \leq \vec{b}) > 0\), since, otherwise, the verification problem is ill-defined.
We can apply \ProbBounds{} for this purpose by computing bounds on~\(\prob[\inp](\mat{A}\vec{x} \leq \vec{b})\).

\section{Additional Details on \ProbBounds{}}
This section contains additional details on \ProbBounds{} (\cref{algo:bound-prob}). 
It includes a detailed description of our procedure for splitting dimensions and a motivation and additional details on our \BaBSB{} \Split{} heuristic.

\subsection{Splitting}\label{sec:split-details}
\Cref{sec:splitting} describes how to split a dimension~\(d \in [n]\) to refine a branch.
This section formally defines the splitting procedure that \ProbBounds{} applies.
A dimension can encode several types of variables. 
We consider continuous variables, such as normalised pixel values, integer variables, such as age, and dimensions containing one indicator of a one-hot encoded categorical variable like gender.
The type of variable encoded in~\(d\) determines how we split~\(d\).
\begin{itemize}
  \item For continuous variables, we further differentiate whether~\(\branch\) is bounded, unbounded in one direction, or unbounded in both directions in dimension~\(d\).
  \begin{itemize}
    \item If~\(\branch\) is bounded in dimension~\(d\), we bisect~\(\branch\) along~\(d\) resulting in two new branches~\(\HR[']{\inp}\) and~\(\HR['']{\inp}\).
      Concretely,~\(\underline{\inp}'_{d'} = \underline{\inp}''_{d'} = \underline{\inp}_{d'}\) and~\(\overline{\inp}'_{d'} = \overline{\inp}''_{d'} = \overline{\inp}_{d'}\) for all~\(d' \in [n] \setminus \{d\}\) while~\(\overline{\inp}'_d = \underline{\inp}''_d = (\underline{\inp}_d + \overline{\inp}_d)/2\),~\(\underline{\inp}'_d = \underline{\inp}_d\), and~\(\overline{\inp}''_d = \overline{\inp}_d\).
    \item If~\(\branch\) is unbounded in both directions in~\(d\), we split~\(d\) at zero, so that~\(\overline{\inp}_d' = \underline{\inp}_d'' = 0\). The remaining bounds of the new branches~\(\HR[']{\inp}\) and~\(\HR['']{\inp}\) are as in the bounded case.
    \item If~\(\branch\) is bounded from below but unbounded from above in~\(d\), that is~\(-\infty < \underline{\inp}_d < \overline{\inp}_d = \infty\), we split at~\(\overline{\inp}_d' = \underline{\inp}_d'' = \max(2\underline{\inp}_d, 1)\), all else being as above.
    Effectively, this split rule performs an exponential search over unbounded dimensions until the remaining unbounded branches are no longer selected by~\Select{}, for example, because they have diminishing probability in the case of~\Prob{}.
    We handle the case where~\(d\) is bounded from above but unbounded from below analogously.
  \end{itemize}
  \item For integer variables, we split~\(d\) as for a continuous variable to obtain~\(\HR[']{\inp}\),~\(\HR['']{\inp}\) and round~\(\overline{\inp}'_d\) to the next smaller integer while rounding~\(\underline{\inp}''_d\) to the next larger integer.
  \item For a one-hot encoded categorical variable~\(V\) encoded in the dimensions~\(A \subseteq [n]\) with~\(d \in A\), we create one split where~\(V\) is equal to the category represented by~\(d\) and one where~\(V\) is different from this category.
    Formally,~\(\underline{\vec{x}}_d' = \overline{\vec{x}}_d' = 1\) and~\(\underline{\vec{x}}_{d'}' = \overline{\vec{x}}_{d'}' = 0\) for~\(d' \in A \setminus \{d\}\) defines~\(\HR[']{\vec{x}}\).
    For~\(\HR['']{\vec{x}}\), we set~\(\underline{\vec{x}}_{d}'' = \overline{\vec{x}}_{d}'' = 0\) and leave the remaining values are they are in~\(\underline{\vec{x}}\) and~\(\overline{\vec{x}}\).
    This splitting procedure eventually creates a new branch where all dimensions are set to zero. 
    This branch has zero probability and can be discarded immediately.
\end{itemize}
In any case, we need to ensure not to select~\(d\) if~\(\underline{\inp}_d = \overline{\inp}_d\).

\subsection{\BaBSB{}}\label{sec:babsb}
Our \BaBSB{} split selection heuristic is a variation of the \BaBSB{} heuristic for non-probabilistic neural network verification of~\textcite{BunelLuTurkaslanEtAl2020}.
One difference is that \textcite{BunelLuTurkaslanEtAl2020} use the method of~\textcite{WongKolter2018} for estimating the improvement in bounds, while we use \IntervalArithmetic{}.
Another difference is that while \textcite{BunelLuTurkaslanEtAl2020} are mainly interested in lower bounds, we are equally interested in lower and upper bounds. 
Let~\(\HR[^{(d, 1)}]{\inp}\) and~\(\HR[^{(d, 2)}]{\inp}\) be the two new branches originating from splitting dimension~\(d \in [n]\) and let~\(\underline{y}^{(d, 1)}, \overline{y}^{(d, 1)}\),~\(\underline{y}^{(d, 2)}, \overline{y}^{(d, 2)}\) be the bounds that \IntervalArithmetic{} computes on \(\fSAT*(\cdot, \NN(\cdot))\) for these branches.
Our \BaBSB{} heuristic selects~\(d = \argmax_{d \in [n]} \tilde{y}^{(d)}\), where~\(\tilde{y}^{(d)} = \max(\max(\underline{y}^{(d, 1)}, \underline{y}^{(d, 2)}), -{\min(\overline{y}^{(d, 1)}, \overline{y}^{(d, 2)})})\).
In other words, we select the dimension~\(d\) that yields the largest lower bound or smallest upper bound in any of the new branches, while~\textcite{BunelLuTurkaslanEtAl2020} select the dimension~\(d\) with the largest lower bound among the smaller lower bound for the two branches originating from splitting~\(d\).
We found this variant to be the most successful for our application. 
\Textcite{BunelLuTurkaslanEtAl2020} discuss further variants.

\subparagraph{Implementation.}
We round all bounds to four decimal places to mitigate floating point issues.
If several dimensions yield equal improvements in bounds, we randomly select one of these dimensions.
Without this random tie-breaking, we might split a single dimension repeatedly if the \IntervalArithmetic{} bounds are very loose.
We use a separate pseudo-random number generator with a fixed seed for this tie-breaking so that \BaBSB{} remains entirely deterministic.

\section{Extended Theoretical Analysis}\label{sec:proofs}
This section contains the proofs of the theorems in \cref{sec:algo-theory}.
We also give a more general completeness analysis of \ToolName{} than presented in \cref{sec:algo-theory}.

\subsection{Soundness}\label{sec:proofs-soundness}
This section contains the proofs for our soundness results from \cref{sec:algo-theory}.

\begin{proof}[Proof of \cref{lem:sound-bounds}]
  Let~\(t \in \Nats\) and let~\(\Xsat[t]\) and~\(\Xviol[t]\) be as in \cref{algo:bound-prob}.
  \ProbBounds{} computes~\(\ell^{(t)}\) as the total probability of all previously pruned satisfied branches~\(\Xsat*[t] = \bigcup_{t'=1}^t \Xsat[t']\).
  Similarly,~\(u^{(t)} = 1 - \ka^{(t)}\) where~\(\ka^{(t)}\) is the total probability of all previously pruned violated branches~\(\Xviol*[t] = \bigcup_{t'=1}^t \Xviol[t']\).
  Since we assumed that \ComputeBounds{} produces valid bounds, \Prune{} only prunes branches that are actually satisfied or violated.
  Therefore,~\(\Xsat*[t] \subseteq \{\inp \in \mathcal{X} \mid \parg\}\) and~\(\Xviol*[t] \subseteq \{\inp \in \mathcal{X} \mid \fSAT*(\inp, \NN(\inp)) < 0\}\).
  From this, it follows directly that
  \begin{align*}
    \ell^{(t)} &= \prob[\inp](\Xsat*[t]) 
    \leq \pterm \\
    \ka^{(t)} &= \prob[\inp](\Xviol*[t])
    \leq \prob[\inp](\fSAT*(\inp, \NN(\inp)) < 0),
  \end{align*}
  which implies~\(u^{(t)} = 1 - \ka^{(t)} \geq 1 - \prob[\inp](\fSAT*(\inp, \NN(\inp)) < 0) = \pterm\).
  This shows that \ProbBounds{} is sound.
\end{proof}

\begin{proof}[Proof of \cref{thm:algo-sound}]
  \Cref{thm:algo-sound} follows from \cref{lem:sound-bounds} and the soundness of the \ComputeBounds{} procedure applied by \ToolName{}.
\end{proof}

\subsection{Completeness}\label{sec:proofs-completeness}
This section is concerned with proving our completeness result from \cref{sec:algo-theory}.
We first discuss in more detail why \cref{assume:one} is only mildly restrictive.
Next, we define conditions on the \Split{}, \Select{}, and \ComputeBounds{} procedures that ensure the completeness of \ToolName{}.
We then prove that the \Prob{}, \LongestEdge{}, \BaBSBLongestEdge{k} heuristics and \IntervalArithmetic{}, as well as \CROWN{} satisfy these conditions.
Finally, we prove the completeness of \ToolName{}.

\textbf{Discussion of \cref{assume:one}.}
The proof of \cref{thm:complete-mainbody} is based on \cref{thm:converging-bounds-mainbody} that states that \ProbBounds{} produces a sequence of lower and upper bounds that converge towards each other.
Intuitively, we require \cref{assume:one} since converging bounds on~\(\fSAT(p_1, \ldots, p_n)\) are insufficient for proving~\(\fSAT(p_1, \ldots, p_n) \geq 0\) if~\(\fSAT(p_1, \ldots, p_n) = 0\)~\cite{AlbarghouthiDAntoniDrewsEtAl2017}.
Note that \(p_1, \ldots, p_v\) and \(f_{\mathrm{Sat}}(p_1, \ldots, p_v)\) are unknown but fixed values in \cref{eqn:pvp}. 

For illustration, assume we want to show~\(y \geq 0\), where~\(y \in \Reals\) is an unknown constant. 
We are provided with converging sequences of bounds~\({(\ell_t)}_{t \in \mathbb{N}}\) and~\({(u_t)}_{t \in \mathbb{N}}\) with~\(\ell_t \leq y \leq u_t\) for each~\(t \in \mathbb{N}\) and~\(\lim_{t \to \infty} \ell_t = \lim_{t \to \infty} u_t = y\). 
If~\(y = 0\), the sequences of bounds that only converge in the limit do not suffice to prove~\(y \geq 0\), since there may not be a~\(T \in \mathbb{N}\) with \(\ell_T = 0\). 
However, if the~\(y \neq 0\), obtaining a finite number of iterates of~\({(\ell_t)}_{t \in \mathbb{N}}\) and~\({(u_t)}_{t \in \mathbb{N}}\) always suffices for proving or disproving~\(y \geq 0\). 
Concretely, there will be a~\(T \in \mathbb{N}\), such that either~\(\ell_T > 0\) or~\(u_T < 0\), that proves, respectively, disproves~\(y \geq 0\).
The assumption that~\(f_{\mathrm{Sat}}(p_1, \ldots, p_v) \neq 0\) corresponds to assuming~\(y \neq 0\) in this example. 

In \cref{sec:algo-theory}, we describe studying~\(\fSATprime(p_1, \ldots, p_n) = \fSAT(p_1, \ldots, p_n) - \varepsilon\) for some~\(\varepsilon > 0\), if we suspect that~\(\fSAT(p_1, \ldots, p_n) = 0\).
If~\(\fSAT(p_1, \ldots, p_n) = 0\), the probabilistic verification problem with~\(\fSATprime\) in place of~\(\fSAT\) satisfies \cref{assume:one} and is only marginally stronger than the original verification problem.

The motivation for requiring~\(\prob[{\inp[i]}](\pfun[i]{} = 0) = 0\),~\(\forall i \in [v]\) is similar as for requiring~\(\fSAT(p_1, \ldots, p_n) \neq 0\).
If~\(\prob[{\inp[i]}](\pfun[i]{} = 0) \neq 0\), there can be a region of the input space with positive probability that we can never prune, since the bounds computed by interval arithmetic or CROWN may only converge in the limit for this region.
However, if this is the case, we can tighten~\(\parg[i]\) to~\(\pfun[i]{} \geq \varepsilon\) for some~\(\varepsilon > 0\) such that~\(\prob[{\inp[i]}](\pfun[i]{} = \varepsilon) = 0\).
Such an~\(\varepsilon > 0\) exists because~\(\prob[\inp](\pfun{} = 0) = 0\) means that the satisfaction boundary has positive volume but any neural network has only finitely many flat regions that can produce a satisfaction boundary of positive volume.
We now define conditions on the \ComputeBounds{}, \Split{}, and \Select{} procedures that ensure the completeness of \ToolName{}.

\begin{defn}[Convergent Bounds]\label{defn:convergent-bounds}
  Let~\(f: \Reals^n \to \Reals^m\). 
  We call a \ComputeBounds{} procedure that computes~\(\uly \leq f(\inp) \leq \oly\) for~\(\inp \in \HR{\inp}\) \emph{convergent} if~\(\|\oly - \uly\| \to 0\) as~\(\|\olx - \ulx\| \to 0\) and~\(\|\oly - \uly\| = 0\) if~\(\|\olx - \ulx\| = 0\).
\end{defn}

\begin{defn}[Dimension Alternation]\label{defn:dim-alt}
  Let~\(\HR{\inp} \subseteq \Reals^n\).
  A splitting procedure \Split{} is \emph{dimension-alternating} if for every~\(d \in [n]\) with~\(\ulx_d \neq \olx_d\)
  \begin{equation*}
    \exists t \in \Nats: \exists \HR[']{\inp} \in \mathrm{branches}^{(t)}: \olx'_d - \ulx'_d < \olx_d - \ulx_d,
  \end{equation*}
  where~\(\mathrm{branches}^{(t)} = \Split(\mathrm{branches}^{(t-1)})\) for~\(t \in \Nats\) and~\(\mathrm{branches}^{(0)} = \HR{\inp}\).
\end{defn}

\begin{defn}[Branch Alternation]\label{defn:branch-alt}
  A branch selection procedure \Select{} is \emph{branch-alternating} if
  \begin{equation*}
    \forall t \in \Nats: 
    \forall \branch \in \mathrm{branches}^{(t)}:
    \prob[\inp](\branch) > 0
    \Longrightarrow 
    \exists t' \geq t:
    \branch \in \Select{}(\mathrm{branches}^{(t')}, N),
  \end{equation*}
  where~\(N \in \Nats\) and~\(\mathrm{branches}^{(t)}\) is the value of the~\(\mathrm{branches}\) variable of \ProbBounds{} in iteration~\(t\) where \ProbBounds{} is instantiated with \Select{} and a \ComputeBounds{} procedure satisfying \cref{defn:convergent-bounds}.
\end{defn}

In the following, we prove that \Prob{}, \LongestEdge{}, and \BaBSBLongestEdge{k} as introduced in \cref{sec:prob-bounds} are branch alternating and dimensional alternating, respectively.
It is well-known that \IntervalArithmetic{} satisfies \cref{defn:convergent-bounds}~\cite{MooreKearfottCloud2009}.
We provide a proof in \cref{sec:ia-theorem-six-one-moore-et-al}.
We show that \CROWN{} satisfies \cref{defn:convergent-bounds} in \cref{sec:crown-theory}.

\begin{proposition}\label{prop:longest-edge-is-dim-alt}
  \LongestEdge{} satisfies \cref{defn:dim-alt}.
\end{proposition}

\begin{proof}
  Let~\(\HR{\inp} \subseteq \Reals^n\),~\(d \in [n]\) with~\(\ulx_d \neq \olx_d\), and let~\(\mathrm{branches}^{(t)}\) for~\(t \in \NatsZero\) be as in \cref{defn:dim-alt}.
  We call~\(\olx_d - \ulx_d\) the \emph{edge length} of~\(d\) in~\(\HR{\inp}\).
  
  If~\(\olx_d - \ulx_d > \max_{d' \neq d} \olx_{d'} - \ulx_{d'}\), the dimension~\(d\) is selected for splitting immediately.
  In the following, we not only show that~\(\olx_d - \ulx_d\) decreases when split but also that~\(\olx_d - \ulx_d \to 0\) in at least one branch when we split~\(d\) repeatedly.
  This result is required for the second part of this proof.
  We differentiate several cases based on the variable encoded in~\(d\).
  \begin{itemize}
      \item \emph{Bounded Continuous Variable.} 
        Let~\(d\) encode a continuous variable with~\(\olx_d - \ulx_d < \infty\).
        As described in \cref{sec:splitting} we split such dimensions by bisecting~\(\HR{\inp}\) along~\(d\).
        Bisecting decreases the edge length of~\(d\) in the resulting branches so that we have~\(\olx_d' - \ulx_d' < \olx_d - \ulx_d\) for all~\(\HR[']{\inp} \in \mathrm{branches}^{(1)} = \Split{}(\HR{\inp})\).
        Furthermore, the edge length of~\(d\) converges towards zero if we bisect along~\(d\) repeatedly.
      \item \emph{Continuous Variable Bounded from Below but Unbounded from Above.}
        Let~\(d\) encode a continuous variable with~\(-\infty < \ulx_d < \olx_d = \infty\).
        Since splitting such a dimension creates one branch where~\(d\) is bounded, we have~\(\olx_d' - \ulx_d' < \olx_d - \ulx_d\) for the bounded branch~\(\HR[']{\inp} \in \mathrm{branches}^{(1)}\).
        Furthermore, repeatedly splitting the bounded branch along~\(d\) lets the edge length of~\(d\) converge towards zero, as discussed above.
      \item \emph{Continuous Variable Bounded from Above but Unbounded from Below.}
        This case proceeds analogously to the previous case.
      \item \emph{Continuous Variable Unbounded from Both Sides.}
        Splitting along such variables creates two branches that are bounded from one side. 
        Therefore, after two splits, we obtain two bounded branches, such that~\(\olx_d' - \ulx_d' < \olx_d - \ulx_d\) for two~\(\HR[']{\inp} \in \mathrm{branches}^{(2)} = \Split{}(\mathrm{branches}^{(1)})\).
        Similarly, repeatedly splitting the bounded branches along~\(d\) lets the edge length of~\(d\) converge towards zero.
      \item \emph{Integer Variable.}
        Let~\(d\) encode an integer variable.
        Splitting~\(d\) proceeds as for a continuous variable, except for excluding non-integer values from the new branches.
        This decreases the edge length of~\(d\) at least as much as if we were splitting a continuous variable.
        Therefore, we have that~\(\olx_d' - \ulx_d' < \olx_d - \ulx_d\) for at least one branch~\(\HR[']{\inp} \in \mathrm{branches}^{(1)}\).
        Furthermore, the edge length of~\(d\) reaches zero after finitely many splits in all bounded branches since we exclude non-integer values.
      \item \emph{One-Hot Encoded Categorical Variable.}
        Let~\(d\) contain an indicator of a one-hot encoded categorical variable.
        Splitting~\(d\) decreases the edge length of~\(d\) to zero, so that we have~\(\olx_d' - \ulx_d' = 0 < \olx_d - \ulx_d\) for all~\(\HR[']{\inp} \in \mathrm{branches}^{(1)}\).
  \end{itemize}
  Overall, \cref{defn:dim-alt} is satisfied if dimension~\(d\) is selected for splitting immediately.

  Now consider the case that~\(d\) is not selected for splitting immediately.
  In this case, a different dimension~\(d' \in [n], d' \neq d\) with~\(\olx_{d'} - \ulx_{d'} \geq \olx_{d} - \ulx_{d}\), is selected for splitting by \LongestEdge{}.
  As we have argued above, repeatedly splitting~\(d'\) lets the edge length of~\(d'\) decrease towards zero in at least one branch.
  Therefore, we eventually obtain~\(\HR[']{\inp} \in \mathrm{branches}^{(t)}\) with~\(\olx_{d'}' - \ulx_{d'}' < \olx_{d} - \ulx_{d}\).
  Since this holds for all~\(d'' \in [n], d'' \neq d\) with~\(\olx_{d''} - \ulx_{d''} \geq \olx_{d} - \ulx_{d}\), we eventually obtain a branch where \LongestEdge{} splits~\(d\). 
  Therefore, \cref{defn:dim-alt} is also satisfied if dimension~\(d\) is not selected for splitting immediately.
  Overall, \LongestEdge{} satisfies \cref{defn:dim-alt}.
\end{proof}

\begin{corollary}\label{cor:babsb-longest-edge-k-dim-alt}
  \BaBSBLongestEdge{k} satisfies \cref{defn:dim-alt}. 
\end{corollary}

\begin{proposition}\label{prop:prob-branch-alt}
  \Prob{} satisfies \cref{defn:branch-alt}.
\end{proposition}

\begin{proof}
  Let~\(\mathrm{branches}^{(t)}\) be the value of the \(\mathrm{branches}\) variable of \ProbBounds{} (\cref{algo:bound-prob}) in iteration~\(t \in \Nats\), where \ProbBounds{} is instantiated with \Prob{} and a \ComputeBounds{} procedure satisfying \cref{defn:convergent-bounds}.
  Let~\(N, t \in \Nats\) and~\(\branch \in \mathrm{branches}^{(t)}\) with~\(\prob[\inp](\branch) > 0\).
  Our goal is to show
  \begin{equation}
    \exists t' \geq t: \branch \in \Prob{}(\mathrm{branches}^{(t')}, N).%
    \label{eqn:prob-longest-edge-branch-alt-goal}
  \end{equation}
  If~\(\branch \in \Prob{}(\mathrm{branches}^{(t)}, N)\), \cref{eqn:prob-longest-edge-branch-alt-goal} holds immediately.
  Otherwise, there are at least~\(N\) branches~\(\branch[t]'\) in iteration~\(t\) with~\(\prob[\inp](\branch[t]') \geq \prob[\inp](\branch)\).
  We show
  \begin{equation}
    \exists t' > t: 
    \forall \branch[t]', \prob[\inp](\branch[t]') \geq \prob[\inp](\branch): 
    \underbrace{%
      \forall \branch[t']', \branch[t]' \leadsto \branch[t']': 
      \prob[\inp](\branch[t']') < \prob[\inp](\branch)
    }_{(\ast)},%
    \label{eqn:evtl-prob-low}
  \end{equation}
  where~\(\branch[t]' \leadsto \branch[t']'\) if~\(\branch[t']'\) is a branch in iteration~\(t' \in \Nats\) that originates from splitting~\(\branch[t]'\), meaning that~\(\branch[t']' \subset \branch[t]'\).

  Let~\(\branch[t]'\) be a branch in iteration~\(t\) with~\(\prob[\inp](\branch[t]') \geq \prob[\inp](\branch[t])\).
  First of all, if~\(\branch[t]'\) is pruned by \ProbBounds{} in iteration~\(t\), then there are no new branches originating from~\(\branch[t]'\), so that~\((\ast)\) holds vacuously.
  Otherwise, \ProbBounds{} splits~\(\branch[t]'\).

  We first consider the special case where the input space only contains categorical and bounded integer variables.
  Dimensions encoding such variables can only be split finitely often.
  Therefore, splitting~\(\branch[t]'\) eventually produces a finite set of branches~\(\HR{\inp}\) with~\(\ulx = \olx\). 
  Since we assumed \ComputeBounds{} to satisfy \cref{defn:convergent-bounds}, \ComputeBounds{} computes the bounds~\(\uly = \oly\) for branches with~\(\ulx = \olx\).
  Branches with~\(\uly = \oly\) are certainly pruned by \ProbBounds{}.
  Therefore, if we choose~\(t' > t\) large enough, \cref{eqn:evtl-prob-low} holds with~\((\ast)\) holding vacuously, since all branches originating from~\(\branch[t]'\) have been pruned.

  Otherwise, let the input space contain at least one continuous or unbounded integer variable.
  We show that there is an iteration~\(t' > t\) such that~\((\ast)\) holds for~\(\branch[t]'\).
  Without loss of generality, assume that the dimension selected for splitting encodes a continuous variable or an unbounded integer variable.
  This does not harm generality since categorical variables and bounded integer variables can only be split finitely often and, therefore, will eventually become unavailable for splitting.

  First, consider splitting along a dimension~\(d\) encoding a bounded continuous variable.
  Since we split continuous variables by bisection, the volume of all branches~\(\branch[t']'\) originating from~\(\branch[t]'\) decreases towards zero as we split~\(d\) repeatedly.
  As stated in \cref{sec:prelim}, we assume that all continuous random variables admit a probability density function.
  This implies that the probability in all branches~\(\branch[t']'\) originating from splitting~\(\branch[t]'\) decreases towards zero as the volume decreases towards zero.

  Now, consider splitting along a dimension~\(d\) encoding an unbounded variable.
  Without loss of generality, assume that~\(\branch[t]'\) is bounded in~\(d\) in at least one direction.
  This does not harm generality since dimensions unbounded in both directions are split into two parts, each bounded in one direction.
  Given this, splitting along~\(d\) creates a bounded and an unbounded part. 
  If~\(d\) encodes an integer variable, the bounded part contains only finitely many discrete values.
  Therefore, as argued above, all branches originating from this bounded part are eventually pruned. 
  If~\(d\) encodes a continuous variable, the bounded part behaves as described above with the probability of all branches~\(\branch[t']'\) originating from~\(\branch[t]'\) decreasing towards zero.
  Therefore, we only have to show that the probability remaining in the unbounded part decreases towards zero as we continue splitting to show~\((\ast)\).
  In fact, this follows from the properties of a probability measure.

  Above, we have shown that splitting repeatedly either leads to pruning the resulting branches or the probability of all resulting branches decreases towards zero.
  Therefore, there is a~\(t'\), such that~\((\ast)\) is satisfied for~\(\branch[t]'\).
  Since there are only finitely many branches in any iteration of \ProbBounds{}, the above implies that \cref{eqn:evtl-prob-low} is satisfied.
  In turn, this directly implies that~\(\branch\) is eventually selected by \ProbBounds{}, proving \cref{prop:prob-branch-alt}.
\end{proof}

Next, we prove a generalised version of \cref{thm:converging-bounds-mainbody}.

\begin{lemma}[Converging Probability Bounds]\label{thm:converging-bounds}
  Let~\(N \in \Nats\) be a batch size.
  Let~\({\left\{(\ell^{(t)}, u^{(t)})\right\}}_{t \in \Nats}\) be the iterates of~\(\ProbBounds{}(\pterm, N)\) instantiated with a \Select{} procedure satisfying \cref{defn:dim-alt}, a sound \ComputeBounds{} procedure satisfying \cref{defn:convergent-bounds} and a \Split{} procedure satisfying \cref{defn:branch-alt}.
  Assume~\(\prob[\inp](\fSAT*(\inp, \NN(\inp)) = 0) = 0\) as in \cref{assume:one}.
  Then, 
  \begin{equation*}
    \lim_{t \to \infty} \ell^{(t)} = \lim_{t \to \infty} u^{(t)} = \pterm.
  \end{equation*}
\end{lemma}

\begin{proof}
  We first prove that~\(\lim_{t \to \infty} \ell^{(t)} = \pterm\).
  The convergence of the upper bound,~\(\lim_{t \to \infty} u^{(t)} = \pterm\) follows analogously.

  Let~\(\Xsat<\ast> = \{\inp \in \mathcal{X} \mid \pfun > 0\}\), where~\(\mathcal{X}\) is the input space of~\(\NN\).
  Note that~\(\prob[\inp](\Xsat<\ast>) = \prob[\inp](\{\inp \in \mathcal{X} \mid \parg\})\) due to \cref{assume:one}.
  Further, let~\(\Xsat*[t]\) be as in the proof of \cref{lem:sound-bounds} and recall~\(\ell^{(t)} = \prob[\inp](\Xsat*[t])\).

  First, we give an argument why the limit~\(\lim_{t \to \infty} \ell^{(t)}\) exists.
  Due to \cref{lem:sound-bounds}, the sequence~\({\{\ell^{(t)}\}}_{t\in\mathbb{N}}\) is bounded from above.
  Furthermore,~\({\{\ell^{(t)}\}}_{t\in\mathbb{N}}\) is non-decreasing in~\(t\) since \ProbBounds{} only adds elements to~\(\Xsat*[t]\).
  Therefore,~\(\lim_{t \to \infty} \ell^{(t)}\) exists.
  Now, we can equivalently rewrite
  \begin{alignat}{2}
    &                    & \ell^{(t)} &\xrightarrow[t \to \infty]{} \pterm\nonumber\\
    &\Longleftrightarrow\quad & \prob[\inp](\Xsat*[t]) &\xrightarrow[t \to \infty]{} \prob[\inp](\Xsat<\ast>)\nonumber\\
    &\Longleftrightarrow\quad & \prob[\inp](\Xsat*[t]) - \prob[\inp](\Xsat<\ast>) &\xrightarrow[t \to \infty]{} 0\nonumber\\
    &\Longleftrightarrow\quad & \prob[\inp](\Xsat<\ast> \setminus \Xsat*[t]) &\xrightarrow[t \to \infty]{} 0.\label{eqn:left-overs}
  \end{alignat}

  We now argue why~\eqref{eqn:left-overs} holds.  
  Consider the case that the input space does not contain a continuous variable.
  Let~\(\branch[t]\) be a bounded branch in iteration~\(t \in \Nats\) for an input space containing only discrete variables.
  As discussed in the proof of \cref{prop:prob-branch-alt}, there is an iteration~\(t' > t\) so that all branches~\(\branch[t']\) that originate from splitting~\(\branch[t]\) are pruned.
  Due to the properties of a probability measure, the probability of the unbounded branches~\(\branch[t]\) that remain decreases towards zero as~\(t \to \infty\).
  Therefore,~\eqref{eqn:left-overs} holds if the input space does not contain a continuous variable.

  Now, assume the input space contains at least one continuous variable.
  Let~\(\Xsat**[t] = \Xsat<\ast> \setminus \Xsat*[t]\).
  With the goal of obtaining a contradiction, assume~\(\lim_{t \to \infty} \prob[\inp](\Xsat**[t]) > 0\).
  Since we assumed in \cref{sec:prelim} that every continuous probability distribution admits a density function,~\(\vol(\Xsat**[t]) > 0\), where~\(\vol\) denotes the volume.
  Let~\(t \in \Nats\). 
  Since the branches maintained by \ProbBounds{} form a partition of the input space, there is a branch~\(\branch[t]\) in iteration~\(t\) of \ProbBounds{} such that~\(\vol(\Xsat** \cap \branch[t]) > 0\).

  Due to \Split{} satisfying \cref{defn:dim-alt} and \Select{} satisfying \cref{defn:branch-alt}, we eventually obtain~\(\branch[t']\) in iteration~\(t' > t\) with~\(\branch[t'] \subseteq \Xsat** \cap \branch[t]\).
  Additionally, since \Split{} satisfies \cref{defn:dim-alt}, we have that~\(\|\olx - \ulx\| \to 0\) for all not yet pruned branches~\(\HR{\inp}\) as~\(t \to \infty\). 
  Since \ComputeBounds{} satisfies \cref{defn:convergent-bounds}, we have that the bounds~\(\underline{y} \leq \pfun \leq \overline{y}\) produced by \ComputeBounds{} converge towards~\(\pfun\).
  Note that~\(\pfun > 0\) for all~\(\inp \in \HR{\inp}\) since~\(\HR{\inp} \subseteq \branch[t'] \subseteq \Xsat<\ast>\).
  This implies that there is an iteration~\(t'' > t'\) in which \ProbBounds{} considers a branch~\(\branch[t''] \subset \Xsat**\) for which~\(\underline{y} > 0\), which means that \ProbBounds{} prunes~\(\branch[t'']\).
  This contradicts the construction of~\(\Xsat**[t]\).
  With this contradiction, we have shown~\(\lim_{t \to \infty} \ell^{(t)} = \pterm\) when the input space contains a continuous variable.

  Overall, we have shown~\(\lim_{t \to \infty} \ell^{(t)} = \pterm\) for all possible compositions of the input space.
  The convergence of the upper bound~\(u^{(t)}\) follows from an analogous argument on~\(\ka^{(t)}\) as in the proof of \cref{lem:sound-bounds} where~\(u^{(t)} = 1 - \ka^{(t)}\).
  This establishes \cref{thm:converging-bounds}.
\end{proof}

\Cref{thm:converging-bounds-mainbody} follows from \cref{thm:converging-bounds} by inserting \Prob{} for \Select{}, \LongestEdge{} or \BaBSBLongestEdge{k} for \Split{}, and \IntervalArithmetic{} or \CROWN{} for \ComputeBounds{}.
\Prob{}, \LongestEdge{}, and \BaBSBLongestEdge{k} satisfy the requirements of \cref{thm:converging-bounds} due to \cref{prop:prob-branch-alt}, \cref{prop:longest-edge-is-dim-alt} and \cref{cor:babsb-longest-edge-k-dim-alt}, respectively.
We show that \IntervalArithmetic{} and \CROWN{} satisfy the requirements of \cref{thm:converging-bounds} in \cref{sec:interval-arithmetic-extra} and \cref{sec:crown-theory}, respectively.
We now prove a generalised version of \cref{thm:complete-mainbody}.

\begin{theorem}[Completeness]\label{thm:complete}
  When instantiated with \ProbBounds{} as in \cref{thm:converging-bounds} and a \ComputeBounds{} procedure that satisfies \cref{defn:convergent-bounds}, \ToolName{} is complete for verification problems satisfying \cref{assume:one}.
\end{theorem}

\begin{proof}
  Let~\(\NN\),~\(\fSAT\),~\(\fSAT*[1], \ldots, \fSAT*[v]\) be as in \cref{eqn:pvp}.
  First, consider~\(\fSAT(p_1, \ldots, p_v) > 0\). 
  As a consequence of \cref{thm:converging-bounds}, the bounds~\(\ell \leq \fSAT(p_1, \ldots, p_v) \leq u\) produced by \ComputeBounds{} converge towards~\(\fSAT(p_1, \ldots, p_v)\), since \ComputeBounds{} satisfies \cref{defn:convergent-bounds}. 
  This implies that eventually~\(\ell > 0\), meaning that \ToolName{} eventually proves~\(\fSAT(p_1, \ldots, p_v) \geq 0\).
  
  If~\(\fSAT(p_1, \ldots, p_v) < 0\) we eventually obtain~\(u < 0\) with the same argument as above.
  Since~\(u < 0\) disproves~\(\fSAT(p_1, \ldots, p_v) \geq 0\), \ToolName{} is complete for probabilistic verification problems satisfying \cref{assume:one}.
\end{proof}
Similarly to \cref{thm:converging-bounds-mainbody} and \cref{thm:converging-bounds}, \cref{thm:complete-mainbody} follows from \cref{thm:complete} by inserting \IntervalArithmetic{} or \CROWN{} for \ComputeBounds{}.

\section{Interval Arithmetic}\label{sec:interval-arithmetic-extra}\label{sec:theory-bound-propagation}
This section introduces additional interval arithmetic bounding rules for linear functions, multiplication, and division, complementing the interval arithmetic bounding rules for monotone functions in \cref{sec:interval-arithmetic}. 
Furthermore, we provide Theorems 5.1 and 6.1~\textcite{MooreKearfottCloud2009} for reference and provide a proof that \IntervalArithmetic{} satisfies \cref{defn:convergent-bounds}.
These results provide the foundation for the theoretical analysis of CROWN in \cref{sec:crown-theory}.

\subsection{Bounding Rules}\label{sec:interval-arithmetic-bounding-rules-extra}
Let~\(f^{(k)}\) be as in \cref{sec:interval-arithmetic}.
First, consider the multiplication of two scalars, that is,~\(f^{(k)}(z, w) = zw\) where~\(\underline{z} \leq z \leq \overline{z}\) and~\(\underline{w} \leq w \leq \overline{w}\).
We have
\begin{equation*}
  \min(\underline{z}\underline{w}, \underline{z} \overline{w}, \overline{z}\underline{w}, \overline{z}\overline{w})
  \leq z_1z_2 \leq
  \max(\underline{z}\underline{w}, \underline{z} \overline{w}, \overline{z}\underline{w}, \overline{z}\overline{w}).
\end{equation*}
For the element-wise multiplication of vectors, we apply the above rule to each element separately.
Multiplication of several arguments can be rewritten as several multiplications of two arguments.

Now, consider computing bounds of the reciprocal~\(f^{(k)}(z) = \frac{1}{z}\) with~\(\underline{z} \leq z \leq \overline{z}\).
We differentiate the following cases
\begin{equation*}
  \everymath={\displaystyle}
  \begin{array}{rlcll}
    \frac{1}{\overline{z}} \leq \frac{1}{z} & \text{if } 0 \notin (\underline{z}, \overline{z}] 
                                            & \qquad &
    \frac{1}{z} \leq \frac{1}{\underline{z}} & \text{if } 0 \notin [\underline{z}, \overline{z}) \\
    -\infty \leq \frac{1}{z} & \text{if } 0 \in (\underline{z}, \overline{z}]
    & \qquad &
    \frac{1}{z} \leq \infty & \text{if } 0 \in [\underline{z}, \overline{z}).
  \end{array}\label{eqn:ia-divide}
\end{equation*}
Using bounds on the reciprocal, we can compute bounds on a division by rewriting division as multiplication by the reciprocal.
Lastly, for an affine function~\(f^{(k)}(\vec{z}) = \mat{W}\vec{z} + \vec{b}\) where~\(\underline{\vec{z}} \leq \vec{z} \leq \overline{\vec{z}}\), we have
\begin{equation}
  \pospart{\mat{W}}\underline{\vec{z}} + \negpart{\mat{W}}\overline{\vec{z}} + \vec{b}
  \leq \mat{W} \vec{z} + \vec{b} \leq
  \pospart{\mat{W}}\overline{\vec{z}} + \negpart{\mat{W}}\underline{\vec{z}} + \vec{b},%
  \label{eqn:ia-affine-rule}
\end{equation}
where~\(\pospart{\mat{W}}_{i,j} = \max(0, \mat{W}_{i,j})\) and~\(\negpart{\mat{W}}_{i,j} = \min(0, \mat{W}_{i,j})\).
%

\subsection{Theoretical Properties}\label{sec:ia-theorem-six-one-moore-et-al}
We include Theorems 5.1 and 6.1 of~\textcite{MooreKearfottCloud2009} and relevant definitions for reference.
Let~\(\mathbb{H}^n = \{[\underline{\vec{x}}, \overline{\vec{x}}] \mid \underline{\vec{x}}, \overline{\vec{x}} \in \Reals^n, \underline{\vec{x}} \leq \overline{\vec{x}}\}\) be the set of hyperrectangles in~\(\Reals^n\).
Let~\(w: \powerset{\Reals^n} \to \RealsNonNeg\) with
\begin{equation*}
  w(\mathcal{X}) 
  = \max_{i \in [n]}{(\max_{\vec{x} \in \mathcal{X}}\vec{x}_i - \min_{\vec{x} \in \mathcal{X}}\vec{x}_i)}
  = \max_{\inp, \inp{}' \in \mathcal{X}} {\| \inp - \inp{}' \|}_\infty
\end{equation*}
be the \emph{width} of the set~\(\mathcal{X} \subseteq \Reals^n\).
We denote the \emph{image} of a hyperrectangle~\(\HR{\inp}\) under~\(f: \Reals^n \to \Reals^m\) as~\(f(\HR{\inp}) = \{f(\inp) \mid \inp \in \HR{\inp}\}\).

\subsubsection{Definitions}

Theorem 5.1 of~\textcite{MooreKearfottCloud2009} applies to inclusion isotonic interval extensions as defined below.
\begin{defn}\label{defn:incl-iso-interval-ext}
   Let~\(F: \mathbb{H}^n \to \mathbb{H}^m\) and~\(f: \Reals^n \to \Reals^m\).
   \begin{itemize}
     \item \(F\) is an \emph{interval extensions} of~\(f\) if~\(\forall \vec{x} \in \Reals^n: F([\vec{x}, \vec{x}]) = [f(\vec{x}), f(\vec{x})]\).
     \item \(F\) is \emph{inclusion isotonic} if~\(\forall [\underline{\vec{x}}, \overline{\vec{x}}], [\underline{\vec{x}}', \overline{\vec{x}}'] \in \mathbb{H}^n, [\underline{\vec{x}}', \overline{\vec{x}}'] \subseteq [\underline{\vec{x}}, \overline{\vec{x}}]: F([\underline{\vec{x}}', \overline{\vec{x}}']) \subseteq F([\underline{\vec{x}}, \overline{\vec{x}}])\).
   \end{itemize}
\end{defn}
Theorem 6.1 of~\textcite{MooreKearfottCloud2009} requires an inclusion isotonic \emph{Lipschitz} interval extension.
\begin{defn}\label{defn:lipschitz-interval-ext}
   Let~\(F: \mathbb{H}^n \to \mathbb{H}^m\) be an interval extension of~\(f: \Reals^n \to \Reals^m\).
   \(F\) is \emph{Lipschitz} if there exists an~\(L \in \Reals\) such that~\(\forall [\underline{\vec{x}}, \overline{\vec{x}}] \in \mathbb{H}^n: w(F([\underline{\vec{x}}, \overline{\vec{x}}])) \leq Lw(\HR{\vec{x}})\).
\end{defn}

\IntervalArithmetic{}, as introduced in \cref{sec:interval-arithmetic}, corresponds to \emph{natural interval extensions} in~\textcite{MooreKearfottCloud2009}.
As~\textcite{MooreKearfottCloud2009} show, natural interval extensions~---~and, therefore, \IntervalArithmetic{}~---~satisfy Definitions~\ref{defn:incl-iso-interval-ext} and~\ref{defn:lipschitz-interval-ext}.

\subsubsection{Theorems and Propositions}
Theorem 5.1 of~\textcite{MooreKearfottCloud2009} is known as the fundamental theorem of interval analysis.
It proves that \IntervalArithmetic{} is sound.
We also provide a proof for the well-known property that \IntervalArithmetic{} satisfies \cref{defn:convergent-bounds}.
This result is closely related to Theorem 6.1 of~\textcite{MooreKearfottCloud2009}, which we also include for reference.
\begin{theorem}[Theorem 5.1 of~\textcite{MooreKearfottCloud2009}]\label{thm:moore-et-al-five-one}
  If~\(F: \mathbb{H}^n \to \mathbb{H}^m\) is an inclusion isotonic interval extension of~\(f: \Reals^n \to \Reals^m\), we have~\(f(\HR{\inp}) \subseteq F(\HR{\inp})\) for every~\(\HR{\inp} \in \mathbb{H}^n\).
\end{theorem}
 
\begin{theorem}[Theorem 6.1 of~\textcite{MooreKearfottCloud2009}]\label{thm:moore-et-al-six-one}
  Let~\(F: \mathbb{H}^n \to \mathbb{H}^m\) be an inclusion isotonic Lipschitz interval extension of~\(f: \Reals^n \to \Reals^m\).
  Let~\(\mathcal{X} = \HR{\inp} \in \mathbb{H}^n\).
  We define the~\(M\)-step \emph{uniform subdivision} of~\(\HR{\inp}\) with~\(M \in \Nats\) as
  \begin{equation*}
    \mathcal{X}_{i,j} = \left[
      \underline{\inp}_i + (j - 1)\frac{w([\underline{\inp}_i, \overline{\inp}_i])}{M}, 
      \underline{\inp}_i + j\frac{w([\underline{\inp}_i, \overline{\inp}_i])}{M}
    \right], 
    \quad j \in [M].
  \end{equation*}
  Further, let
  \begin{equation*}
    F^{(M)}(\HR{\inp}) = \bigcup_{j_i = 1}^{M} F(\mathcal{X}_{1,j_1} \times \cdots \times \mathcal{X}_{n,j_n}). \\
  \end{equation*}
  It holds that
  \begin{equation*}
    w(F^{(M)}(\HR{\inp})) - w(f(\HR{\inp})) \leq 2L\frac{w(\mathcal{X})}{M},
  \end{equation*}
  where~\(L\) is the Lipschitz constant of~\(F\).
\end{theorem}

\begin{proposition}\label{prop:lipschitz-incl-iso-int-ext-is-convergent}
  Every Lipschitz interval extension satisfies \cref{defn:convergent-bounds}.
\end{proposition}

\begin{proof}
  Let~\(F: \mathbb{H}^n \to \mathbb{H}^m\) be a Lipschitz interval extension with Lipschitz constant~\(L\). 
  We write~\([\uly_{\HR{\inp}}, \oly_{\HR{\inp}}] = F(\HR{\inp})\) for~\(\HR{\inp} \in \mathbb{H}^n\).
  Using that~\(F\) is Lipschitz, we obtain
  \begin{equation*}
    \lim_{{\|\olx - \ulx\|} \to 0} {\|\oly_{\HR{\inp}} - \uly_{\HR{\inp}}\|}
    = \lim_{{\|\olx - \ulx\|} \to 0} w(F(\HR{\vec{x}}))
    \leq \lim_{{\|\olx - \ulx\|} \to 0} Lw(\HR{\inp})
    = 0.
  \end{equation*}
  Further, since~\(F\) is an interval extension, we have~\(\|\oly_{[\inp, \inp]} - \uly_{[\inp, \inp]}\| = 0\) for any~\(\inp \in \Reals^n\).
  This proves that~\(F\) satisfies \cref{defn:convergent-bounds}.
\end{proof}

\begin{corollary}
  \IntervalArithmetic{} satisfies \cref{defn:convergent-bounds}.
\end{corollary}

\section{Theoretical Analysis of CROWN}\label{sec:crown-theory}
In this section, we prove that \CROWN*{} satisfies \cref{defn:convergent-bounds}.
In fact, we provide a slightly stronger result showing that \CROWN{} possesses the properties of \IntervalArithmetic{} that follow from Theorem 5.1 and 6.1 of \textcite{MooreKearfottCloud2009}.
Concretely, we prove that the bounds computed by \CROWN{} are always contained in the bounds computed by an inclusion isotonic Lipschitz interval extension, such that \cref{prop:lipschitz-incl-iso-int-ext-is-convergent,thm:moore-et-al-five-one,thm:moore-et-al-six-one} apply.
Since the bounds computed by this interval extension converge and they contain the bounds computed by \CROWN{}, the bounds computed by \CROWN{} also need to converge, establishing that \CROWN{} satisfies \cref{defn:convergent-bounds}.
We first revisit \CROWN{}, discuss why \CROWN{} itself is not inclusion isotonic, and finally provide the inclusion isotonic Lipschitz interval extension that is guaranteed to compute \emph{looser} bounds than \CROWN{}.

In the following, we restrict our attention to ReLU-activated fully-connected neural networks since these are the most relevant for this paper. 
However, similar arguments to the arguments we make below can also be made for other architectures and activation functions.

\subsection{CROWN}\label{sec:crown-recap}
We revisit \CROWN{} for ReLU-activated fully-connected neural networks. 
Applying \CROWN{} for other activation functions and architectures is described by \textcite{ZhangWengChenEtAl2018,XuShiZhangEtAl2020b}.
Let~\(\NN: \mathcal{X} \to \Reals^m\) be a ReLU-activated fully-connected neural network with~\(K \in \Nats\) layers, that is,~\(\NN = f^{(K)} \circ \ldots \circ f^{(1)}\), where~\(f^{(k)}: \Reals^{n_k} \to \Reals^{n_{k+1}}\) is either an affine function or ReLU.\@
We use~\(f^{(:k)} = f^{(k)} \circ \cdots \circ f^{(1)}\) and~\(f^{(k:)} = f^{(K)} \circ \cdots \circ f^{(k)}\) to denote partial evaluation of~\(\NN\).
Further, let~\(\HR{\inp} \subseteq \mathcal{X}\).
\CROWN{} computes a linear lower bound and a linear upper bound on~\(\NN\)
\begin{equation*}
    \underline{\mat{A}}\inp + \underline{\vec{d}} 
    \leq \NN(\inp) \leq 
    \overline{\mat{A}}\inp + \overline{\vec{d}}, \qquad \forall \inp \in \HR{\inp}.%
\end{equation*}
These linear bounds can be \emph{concretised} into constant bounds as
\begin{equation*}
    \underline{\vec{y}}_{\CROWN{}} =
    \pospart{\underline{\mat{A}}}\underline{\inp} + \negpart{\underline{\mat{A}}}\overline{\inp} + \underline{\vec{d}}
    \leq \NN(\inp) \leq 
    \pospart{\overline{\mat{A}}}\overline{\inp} + \negpart{\overline{\mat{A}}}\underline{\inp} + \overline{\vec{d}}
    = \overline{\vec{y}}_{\CROWN{}},
\end{equation*}
where~\(\underline{\inp}, \overline{\inp}\) are bounds on the input,~\(\inp \in \HR{\inp}\),~\(\pospart{\mat{A}}_{i,j} = \max(0, \mat{A}_{i,j})\), and~\(\negpart{\mat{A}}_{i,j} = \min(0, \mat{A}_{i,j})\).

\CROWN{} computes the linear bounds by a backwards walk over the network~\(\NN\), starting from~\(f^{(K)}\) and propagating linear bounds backwards until reaching~\(f^{(1)}\).
\Cref{algo:crown-top-level} defines \CROWN{} for feed-forward neural networks, such as~\(\NN\).
The algorithm iteratively computes linear bounds
\begin{equation*}
    \underline{\mat{A}}^{(k)}\vec{z}^{(k)} + \underline{\vec{d}}^{(k)} 
    \leq f^{(k+1:)}(\vec{z}^{(k)}) \leq 
    \overline{\mat{A}}^{(k)}\vec{z}^{(k)} + \overline{\vec{d}}^{(k)}, \qquad \forall \vec{z}^{(k)} \in [\underline{\vec{z}}^{(k)}, \overline{\vec{z}}^{(k)}],
\end{equation*}
where~\([\underline{\vec{z}}^{(k)}, \overline{\vec{z}}^{(k)}]\) are bounds on~\(f^{(:k)}(\inp)\) for~\(\inp \in \HR{\inp}\) that are computed before invoking \CROWN{}, for example, using \IntervalArithmetic{}~\cite{ZhangChenXiao2020}.
The bounds~\(\underline{\vec{z}}^{(k)}, \overline{\vec{z}}^{(k)}\) are sometimes known as pre-activation bounds.
Alternatively to using \IntervalArithmetic{}, we can compute~\(\underline{\vec{z}}^{(k)}, \overline{\vec{z}}^{(k)}\) by invoking \cref{algo:crown-linear-bounds} iteratively for~\(f^{(:1)}, \ldots, f^{(:K)}\) and use the concretised bounds on~\(f^{(:1)}, \ldots, f^{(:k-1)}\) as the layer bounds for the~\(k\)-th invocation of \cref{algo:crown-linear-bounds}~\cite{ZhangWengChenEtAl2018}.
\Cref{algo:crown-bootstrap} describes this approach in more detail.
Applying \cref{algo:crown-linear-bounds} with the layer bounds computed by \IntervalArithmetic{} is known as \AlgorithmName{CROWN-IBP}~\cite{ZhangChenXiao2020}, while using \cref{algo:crown-bootstrap} is referred to as just \CROWN{}.
To avoid confusion, in this section, we use \CROWN{} to only refer to \cref{algo:crown-top-level}.

\begin{algorithm}
  \caption{\CROWN{}}\label{algo:crown-top-level}\label{algo:crown-linear-bounds}
  \begin{algorithmic}[1]
    \REQUIRE Neural Network \(\NN = f^{(K)} \circ \cdots \circ f^{(1)}\), Input Bounds \(\HR{\inp}\), Layer Bounds \(\HR[^{(1)}]{\vec{z}}, \ldots, \HR[^{(K-1)}]{\vec{z}}\)
    \STATE \(\ulmat{A}^{(K)} \gets \mat{I}\), \(\olmat{A}^{(K)} \gets \mat{I}\) \COMMENT{\(\mat{I} \in \Reals^{m \times m}\) is the identity matrix}
    \STATE \(\ulvec{d}^{(K)} \gets \vec{0}\), \(\olvec{d}^{(K)} \gets \vec{0}\) \COMMENT{\(\vec{0} \in \Reals^m\) is the all-zero vector}
    \STATE \(\ulz^{(0)} \gets \ulx\), \(\olz^{(0)} \gets \olx\)
    \FOR{\(k \in \{K, \ldots, 1\}\)}
      \STATE \((\ulmat{A}^{(k-1)}, \olmat{A}^{(k-1)}, \ulvarvec{\Delta}, \olvarvec{\Delta}) \gets \textsc{CROWNRule}(f^{(k)}, \ulmat{A}^{(k)}, \olmat{A}^{(k)}, \ulvec{z}^{(k-1)}, \olvec{z}^{(k-1)})\)
      \STATE \(\ulvec{d}^{(k-1)} \gets \ulvec{d}^{(k)} + \ulvarvec{\Delta}\)
      \STATE \(\olvec{d}^{(k-1)} \gets \olvec{d}^{(k)} + \olvarvec{\Delta}\)
    \ENDFOR
    \STATE \textbf{return} \((\ulmat{A}^{(0)}, \olmat{A}^{(0)}, \ulvec{d}^{(0)}, \olvec{d}^{(0)})\)
  \end{algorithmic}
\end{algorithm}

\begin{algorithm}
  \caption{\CROWN{} with \CROWN{} Layer Bounds}\label{algo:crown-bootstrap}
  \begin{algorithmic}[1]
    \REQUIRE Neural Network \(\NN = f^{(K)} \circ \cdots \circ f^{(1)}\), Input Bounds \(\HR{\inp}\)
    \FOR[Below, \CROWN{} refers to \cref{algo:crown-linear-bounds}]{\(k \in [K]\)}
      \STATE \([\ulz^{(k)}, \olz^{(k)}] \gets \textsc{Concretise}(\CROWN{}(f^{(:k)}, \HR{\inp}, \HR[^{(1)}]{\vec{z}}, \ldots, \HR[^{k-1}]{\vec{z}}))\)
    \ENDFOR
    \STATE \textbf{return} \(\CROWN{}(\NN, \HR{\inp}, \HR[^{(1)}]{\vec{z}}, \ldots, \HR[^{K-1}]{\vec{z}})\)
  \end{algorithmic}
\end{algorithm}

Similarly to \IntervalArithmetic{}, \cref{algo:crown-top-level} uses \AlgorithmName{CROWNRule}s to propagate linear bounds through more fundamental functions, such as affine layers or ReLU layers.
Specifically, the \AlgorithmName{CROWNRule} for an affine function~\(f^{(k)}(\vec{z}) = \mat{W}\vec{z} + \vec{b}\) computes
\begin{equation}
  \ulmat{A}^{(k-1)} = \ulmat{A}^{(k)}\mat{W}, \qquad
  \olmat{A}^{(k-1)} = \olmat{A}^{(k)}\mat{W}, \qquad
  \ulvarvec{\Delta}^{(k-1)} = \ulmat{A}^{(k)}\vec{b}, \qquad
  \olvarvec{\Delta}^{(k-1)} = \olmat{A}^{(k)}\vec{b}.%
  \label{eqn:crown-rule-affine}
\end{equation}
The \AlgorithmName{CROWNRule} for a ReLU layer~\(f^{(k)}(\vec{z}) = \ReLU{\vec{z}}\) is
\begin{equation}
  \begin{split}
    \ulmat{A}^{(k-1)} = 
    \pospart{\ulmat{A}^{(k)}}\diag(\ulvarvec{\alpha}) 
    + \negpart{\ulmat{A}^{(k)}}\diag(\olvarvec{\alpha}), 
    & \quad
    \ulvarvec{\Delta}^{(k-1)} = 
    \negpart{\ulmat{A}^{(k)}}\olvarvec{\beta}, \\
    \olmat{A}^{(k-1)} = 
    \pospart{\olmat{A}^{(k)}}\diag(\olvarvec{\alpha}) + \negpart{\olmat{A}^{(k)}}\diag(\ulvarvec{\alpha}),
    & \quad
    \olvarvec{\Delta}^{(k-1)} = 
    \pospart{\olmat{A}^{(k)}}\olvarvec{\beta},
  \end{split}%
  \label{eqn:crown-rule-relu}
\end{equation}
where~\(\diag(\varvec{\alpha})\) is a diagonal matrix with the vector~\(\varvec{\alpha}\) on its diagonal and~\(\ulvarvec{\alpha}, \olvarvec{\alpha}, \olvarvec{\beta} \in \Reals^{n_k}\) have
\begin{align*}
    \olvarvec{\alpha}_i = \begin{cases}
        0 & \olvec{z}^{(k-1)}_i \leq 0 \\
        1 & \ulvec{z}^{(k-1)}_i \geq 0 \\
        \frac{\olvec{z}^{(k-1)}_i}{\olvec{z}^{(k-1)}_i - \ulvec{z}^{(k-1)}_i}
          & \text{otherwise},
    \end{cases} \qquad
    \olvarvec{\beta}_i = \begin{cases}
        0 & 0 \notin (\ulvec{z}^{(k-1)}_i, \olvec{z}^{(k-1)}_i) \\
        -\ulvarvec{z}^{(k-1)}_i \olvarvec{\alpha}_i
          & \text{otherwise},
    \end{cases} \\
    \ulvarvec{\alpha}_i = \begin{cases}
        0 & \olvec{z}^{(k-1)}_i \leq 0 \\
        1 & \ulvec{z}^{(k-1)}_i \geq 0 \\
        0 & \ulvec{z}^{(k-1)}_i < 0 < \olvec{z}^{(k-1)}_i
        \ \text{and}\ |\ulvec{z}^{(k-1)}_i| \geq |\olvec{z}^{(k-1)}_i| \\
        1 & \text{otherwise},
    \end{cases}
\end{align*}
for every~\(i \in [n_k]\).
Instead of the last two cases for~\(\ulvarvec{\alpha}_i\), we can also optimise each~\(\ulvarvec{\alpha}_i\) using gradient descent to improve the linear bounds on the network~\cite{XuZhangWangEtAl2021}, as long as we maintain~\(\ulvarvec{\alpha}_i \in [0, 1]\). 

\textbf{\CROWN{} is not Inclusion-Isotonic.}
We demonstrate that \CROWN{} is not inclusion-isotonic according to \cref{defn:incl-iso-interval-ext} using an example. 
Consider a single ReLU neuron~\(\ReLU{x}\) with input bounds~\(\HR{x} = [-3, 2]\).
The \CROWN{} bounds are~\(0 \leq \ReLU{x} \leq \frac{2}{5}x + \frac{6}{5}\) so that the concretised lower bound~\(\underline{y}_{\CROWN{}} = 0\).
However, the \CROWN{} bounds for~\(\HR[']{x} = [-1, 2] \subset \HR{x}\) are~\(x \leq \ReLU{x} \leq \frac{2}{3}x + \frac{2}{3}\) which yields the concretised lower bound~\(\underline{y}'_{\CROWN{}} = -1\).
Since~\(\underline{y}'_{\CROWN{}} < \underline{y}_{\CROWN{}}\) although~\(\HR[']{x} \subset \HR{x}\), \CROWN{} is not inclusion-isotonic.

\subsection{CROWN Interval Extension}
In this section, we introduce the \CIE*{}. 
\CIE{} is a theoretical device that we use to prove that the width of the concretised \CROWN{} bounds linearly converges to zero as the width of the input bounds converges to zero, analogously to \cref{thm:moore-et-al-six-one}. 
Concretely, \CIE{} is an inclusion-isotonic Lipschitz interval extension according to \cref{defn:incl-iso-interval-ext} and \cref{defn:lipschitz-interval-ext} that is guaranteed to contain the concretised \CROWN{} bounds. 
Since \cref{thm:moore-et-al-six-one} applies to \CIE{} and the concretised \CROWN{} bounds are always contained in the \CIE{} bounds, the convergence properties of \cref{thm:moore-et-al-six-one} also apply to \CROWN{}.
In the following, we first define \CIE{}, show that it is an inclusion-isotonic Lipschitz interval extension, show that the bounds computed by \IntervalArithmetic{} are contained in the bounds computed by \CIE{}, and finally prove that the concretised \CROWN{} bounds are contained in the bounds computed by \CIE{}. 

\subparagraph{\CIE{} Algorithm.}
To define \CIE{}, we define a bounding rule for ReLU.\@
We use the same bounding rule for affine functions as \IntervalArithmetic{}, as introduced in \cref{sec:interval-arithmetic-bounding-rules-extra}.
Similarly to \IntervalArithmetic{}, \CIE{} then computes bounds on a ReLU-activated feed-forward neural network~\(\NN = f^{(K)} \circ \cdots \circ f^{(1)}\) by computing~\((F^{(K)}_{\CIE{}} \circ \cdots \circ F^{(1)}_{\CIE{}})(\ulx, \olx)\), where~\(\HR{\inp} \subseteq \mathcal{X}\) and~\(F^{(k)}_{\CIE{}}: \Reals^{n_k} \times \Reals^{n_k} \to \Reals^{n_k} \times \Reals^{n_k}\) is the \CIE{} bounding rule for~\(f^{(k)}\). 
The \CIE{} bounding rule for ReLU layers~\(f^{(k)}(\vec{v}) = \ReLU{\vec{v}}\) is~\(F^{(k)}([\ulv, \olv]) = [\ulw, \olw]\), where
\begin{equation}
  \ulw_i = \begin{cases}
    0 & \olv_i \leq 0 \\
    \ulv_i & \text{otherwise},
  \end{cases} \qquad
  \olw_i = \begin{cases}
    0 & \olv_i \leq 0 \\
    \olv_i & \text{otherwise},
  \end{cases}\label{eqn:crown-int-ext-relu-rule}
\end{equation}
where~\(i \in [n_k]\).
Comparing this bounding rule to the \AlgorithmName{CROWNRule} for ReLU, we can see that the bounds computed by \CIE{} always contain the concretised \CROWN{} bounds for a ReLU neuron. 
The remainder of this section aims to show that this is also the case for a complete neural network.

\begin{proposition}\label{prop:crown-int-ext-lip-incl-iso-int-ext}
  Let~\(\NN = f^{(K)} \circ \cdots \circ f^{(1)}\) be a ReLU-activated fully-connected neural network. 
  The \CIE{} interval function~\(F^{(K)}_{\CIE{}} \circ \cdots \circ F^{(1)}_{\CIE{}}\) is a inclusion-isotonic Lipschitz interval extension of~\(\NN\).
\end{proposition}

\begin{proof}
  We first show that the \CIE{} bounding rule for ReLU is an inclusion-isotonic Lipschitz interval extension of ReLU.\@
  Let~\(F^{(k)}\) be the \CIE{} bounding rule for the ReLU layer~\(f^{(k)}: \Reals^{(n_k)} \to \Reals^{(n_k)}\) according to \cref{eqn:crown-int-ext-relu-rule}.
  \begin{itemize}
    \item \emph{Interval extension.} 
      Let~\(\vec{v} \in \Reals^{n_k}\) and~\(\HR{\vec{w}} = F^{(k)}([\vec{v}, \vec{v}])\). 
      Let~\(i \in [n_k]\). 
      If~\(\vec{v}_i \leq 0\), we have~\(\ulw_i = \olw = 0 = \ReLU{\vec{v}}_i\).
      Otherwise, if~\(\vec{v}_i > 0\), we have~\(\ulw_i = \olw_i = \vec{v}_i = \ReLU{\vec{v}}_i\).
    \item \emph{Inclusion-isotonic.}
      Let~\(\HR[']{\vec{v}} \subseteq \HR{\vec{v}} \subseteq \Reals^{n_k}\),~\(\HR[']{\vec{w}} = F^{(k)}(\HR[']{\vec{v}})\) and~\(\HR{\vec{w}} = F^{(k)}(\HR{\vec{v}})\).
      Let~\(i \in [n_k]\). 
      If~\(\olv_i \leq 0\), we have~\([\ulw_i, \olw_i] = [0, 0] = [\ulw_i', \olw_i']\) since~\(\olv_i' \leq \olv_i\).
      On the other hand, if~\(\olv_i > 0\), we have~\([\ulw_i, \olw_i] = [\ulv_i, \olv_i]\).
      If~\(\olv_i' \leq 0\),~\([\ulw_i', \olw_i'] = [0,0] \subseteq [\ulw_i, \olw_i]\) since~\(\ulv_i \leq \olv_i' < \olv\).
      Otherwise,~\([\ulw_i', \olw_i'] = [\ulv_i', \olv_i'] \subseteq [\ulw_i, \olw_i]\).
      Therefore,~\(F^{(k)}\) is inclusion-isotonic.
    \item \emph{Lipschitz.}
      Let~\(\HR{\vec{v}} \subseteq \Reals^{n_k}\) and~\(\HR{\vec{w}} = F^{(k)}(\HR{\vec{v}})\).
      For all~\(i \in [n_k]\) with~\(\olw_i \leq 0\) we have~\(w([\ulw_i, \olw_i]) = 0\) and, therefore, trivially~\(w([\ulw_i, \olw_i]) \leq w(\HR{\vec{v}})\).
      For~\(i \in [n_k]\) with~\(\olw_i > 0\), we have~\([\ulw_i, \olw_i] = [\ulv_i, \olv_i]\).
      Overall,~\(w(\HR{\vec{w}}) = \max_{i \in [n_k]} \olw_i - \ulw_i \leq \max_{i \in [n_k]} \olv_i - \ulv_i = w(\HR{\vec{v}})\).
    Therefore,~\(F^{(k)}\) is Lipschitz with Lipschitz constant~\(1\).
  \end{itemize}
  For affine functions, \CIE{} uses the same bounding rule as \IntervalArithmetic{} that is an inclusion-isotonic Lipschitz interval extension~\cite{MooreKearfottCloud2009}.
  Lemma 6.3 of \textcite{MooreKearfottCloud2009} now gives us that the composition~\(F^{(K)}_{\CIE{}} \circ \cdots \circ F^{(1)}_{\CIE{}}\) of inclusion-isotonic Lipschitz interval extensions is also inclusion inclusion-isotonic and Lipschitz. 
  Clearly, a composition of interval extensions is also an interval extension.
\end{proof}

\begin{proposition}\label{prop:crown-int-ext-more-loose-than-ia}
  Let~\(\NN = f^{(K)} \circ \cdots \circ f^{(1)}\) be a ReLU-activated fully-connected neural network.
  Let~\(\HR{\vec{y}}\) and~\(\HR{\vec{w}}\) be the bounds on~\(\NN(\inp)\) for~\(\inp \in \HR{\inp}\) computed by \IntervalArithmetic{} and \CIE{}, respectively.
  It holds that~\(\HR{\vec{y}} \subseteq \HR{\vec{w}}\).
\end{proposition}

\begin{proof}
  Let~\(\NN = f^{(K)} \circ \cdots \circ f^{(1)}\) and~\(\HR{\inp}\) be as in \cref{prop:crown-int-ext-more-loose-than-ia}.
  We prove \cref{prop:crown-int-ext-more-loose-than-ia} by finite induction over~\(k \in \{0, \ldots, K\}\).
  Let~\(\HR[^{(k)}]{\vec{y}}\) and~\(\HR[^{(k)}]{\vec{v}}\) be the bounds that \IntervalArithmetic{} and \CIE{} respectively compute on~\(f^{(:k)}(\inp)\) for~\(\inp \in \HR{\inp}\).

  \textit{Induction start.} For~\(k=0\), we consider the identity function~\(\NN(\inp) = \inp\), for which both \CIE{} and \IntervalArithmetic{} compute~\(\HR[^{(0)}]{\vec{y}} = \HR[^{(0)}]{\vec{v}} = \HR{\inp}\).

  \textit{Induction step.} Now, let~\(k \in [K]\) and assume~\(\HR[^{(k-1)}]{\vec{y}} \subseteq \HR[^{(k-1)}]{\vec{v}}\).
  We show~\(\HR[^{(k)}]{\vec{y}} \subseteq \HR[^{(k)}]{\vec{v}}\).
  We differentiate two cases:
  \begin{itemize}
    \item If~\(f^{(k)}\) is an affine function, both \IntervalArithmetic{} and \CIE{} use the \IntervalArithmetic{} bounding rule for affine functions from \cref{sec:interval-arithmetic-bounding-rules-extra} to compute~\(\HR[^{(k)}]{\vec{y}}\) and~\(\HR[^{(k)}]{\vec{v}}\).
      Since this bounding rule is inclusion-isotonic~\cite{MooreKearfottCloud2009}, we have~\(\HR[^{(k)}]{\vec{y}} \subseteq \HR[^{(k)}]{\vec{v}}\).
    \item Otherwise,~\(f^{(k)}\) is ReLU.\@
      The \IntervalArithmetic{} bounding rule for ReLU computes~\(\HR[^{(k)}]{\vec{y}}  = \left[\ReLU{\uly^{(k-1)}}, \ReLU{\oly^{(k-1)}}\right]\).
      Since the \CIE{} rule for ReLU is an inclusion-isotonic interval extension (\cref{prop:crown-int-ext-lip-incl-iso-int-ext}), it follows that~\(\left[\ReLU{\uly^{(k-1)}}, \ReLU{\oly^{(k-1)}}\right] \subseteq \HR[^{(k)}]{\vec{v}}\).
  \end{itemize}
  Overall, this shows that~\(\HR{\vec{y}} = \HR[^{(K)}]{\vec{y}} \subseteq \HR[^{(K)}]{\vec{v}} = \HR{\vec{v}}\).
\end{proof}

\begin{proposition}\label{prop:crown-contained-in-crown-int-ext}
  Let~\(\NN = f^{(K)} \circ \cdots \circ f^{(1)}\) be a ReLU-activated fully-connected neural network, let~\(\HR{\inp} \subseteq \Reals^n\) and let~\(\HR[^{(k)}]{\vec{v}}\) be the \CIE{} bounds on~\(f^{(:k)}(\inp)\) for~\(\inp \in \HR{\inp}\),~\(\forall k \in [K]\).
  Further, let~\(\HR[^{(k)}]{\vec{z}}\) also be bounds on~\(f^{(:k)}(\inp)\) for~\(\inp \in \HR{\inp}\) with~\(\HR[^{(k)}]{\vec{z}} \subseteq \HR[^{(k)}]{\vec{v}}\),~\(\forall k \in [K]\).
  It holds that
  \begin{equation*}
    \ulv^{(K)} \leq \ulA\inp + \uld \quad\text{and}\quad
    \olA\inp + \old \leq \olv^{(K)}, \qquad \forall \inp \in \HR{\inp},
  \end{equation*}
  where~\(\ulA\inp + \uld\) and~\(\olA\inp + \old\) are the bounds on~\(\NN(\inp)\) for~\(\inp \in \HR{\inp}\) computed by \CROWN{} (\cref{algo:crown-linear-bounds}) using the layer bounds~\(\HR[^{(1)}]{\vec{z}}, \ldots, \HR[^{(K-1)}]{\vec{z}}\).
\end{proposition}

\begin{proof}
  Let~\(\NN = f^{(K)} \circ \cdots \circ f^{(1)}\),~\(\HR{\inp}\),~\(\HR[^{(k)}]{\vec{z}}\), and~\(\HR[^{(k)}]{\vec{v}}\) for all~\(k \in [K]\) be as in \cref{prop:crown-contained-in-crown-int-ext}.
  Let~\(\ulA^{(k)}, \olA^{(k)}, \uld^{(k)}, \old^{(k)}\) for all~\(k \in \{0, \ldots, K\}\) be as in \cref{algo:crown-top-level} for the inputs~\(\NN\),~\(\HR{\inp}\) and~\(\HR[^{(1)}]{\vec{z}}, \ldots, \HR[^{(K)}]{\vec{z}}\).
  Further, let~\(\HR[^{(0)}]{\vec{z}} = \HR[^{(0)}]{\vec{v}} = \HR{\inp}\).
  In the following, we prove
  \begin{equation}
    \ulA^{(k)}\vec{z}^{(k)} + \uld^{(k)} \geq \ulv^{(K)} \quad\text{and}\quad
    \olA^{(k)}\vec{z}^{(k)} + \old^{(k)} \leq \olv^{(K)},\qquad 
    \forall \vec{z}^{(k)} \in \HR[^{(k)}]{\vec{v}},%
    \label{eqn:crown-contained-crown-int-ext-ia}
  \end{equation}
  for every~\(k \in \{0, \ldots, K\}\).
  Note that~\(\HR[^{(k)}]{\vec{z}} \subseteq \HR[^{(k)}]{\vec{v}}\) by assumption in \cref{prop:crown-contained-in-crown-int-ext}.
  Proving \cref{eqn:crown-contained-crown-int-ext-ia} also proves \cref{prop:crown-contained-in-crown-int-ext} when~\(k=0\).
  We proceed by finite induction over~\(k\) from~\(K\) to \(0\), following the backwards walk performed by \cref{algo:crown-top-level}. 

  \emph{Induction start.} 
  We consider~\(k = K\).
  In this case,~\(\ulA^{(K)} = \olA^{(K)} = \mat{I}\) and~\(\uld^{(K)} = \old^{(K)} = \vec{0}\), where~\(\mat{I} \in \Reals^{n_K \times n_K}\) is the identity matrix and~\(\vec{0} \in \Reals^{n_K}\) is the all-zero vector.
  Therefore, we have
  \begin{align*}
    \ulA^{(K)}\vec{z}^{(K)} + \uld^{(K)} = \mat{I}\vec{z}^{(K)} + \vec{0} \geq \ulv^{(K)} \\
    \olA^{(K)}\vec{z}^{(K)} + \old^{(K)} = \mat{I}\vec{z}^{(K)} + \vec{0} \leq \olv^{(K)},
  \end{align*}
  for~\(\vec{z}^{(K)} \in \HR[^{(K)}]{\vec{v}}\).
  This proves \cref{eqn:crown-contained-crown-int-ext-ia} for~\(k = K\).

  \emph{Induction step.}
Let~\(k \in [K]\) and assume \cref{eqn:crown-contained-crown-int-ext-ia} holds for~\(k\).
  We show that it also holds for~\(k-1\).
  We differentiate two cases, based on whether~\(f^{(k)}\) is a affine function or ReLU.\@
  \begin{itemize}
    \item \emph{Affine function.} 
      If~\(f^{(k)}(\vec{z}) = \mat{W}\vec{z} + \vec{b}\), \AlgorithmName{CROWNRule} computes the linear bounds on~\(f^{(k:)}\) as in \cref{eqn:crown-rule-affine}, so that we obtain
      \begin{equation*}
        \ulA^{(k-1)}\vec{z} + \uld^{(k-1)} = \ulA^{(k)}\mat{W}\vec{z} + \ulA^{(k)}\vec{b} + \uld^{(k)} = \ulA^{(k)}(\mat{W}\vec{z} + \vec{b}) + \uld^{(k)} \geq \ulv^{(K)},
      \end{equation*}
      where~\(\vec{z} \in \HR[^{(k-1)}]{\vec{v}}\).
      The final inequality is due to the induction assumption that \cref{eqn:crown-contained-crown-int-ext-ia} holds for~\(k\), which we can apply since~\(\mat{W}\vec{z} + \vec{b} \in \HR[^{(k)}]{\vec{v}}\), since \CIE{} computes~\(\HR[^{(k)}]{\vec{v}}\) from~\(\HR[^{(k-1)}]{\vec{v}}\) using \cref{eqn:ia-affine-rule}.
      Similarly, we obtain
      \begin{equation*}
        \olA^{(k-1)}\vec{z} + \old^{(k-1)} = \olA^{(k)}\mat{W}\vec{z} + \olA^{(k)}\vec{b} + \old^{(k)} = \olA^{(k)}(\mat{W}\vec{z} + \vec{b}) + \old^{(k)} \leq \olv^{(K)},
      \end{equation*}
      again by using the induction assumption.
    \item \emph{ReLU.}
      Let~\(\vec{z} \in \HR[^{(k-1)}]{\vec{v}}\) and~\(i \in [n_k]\).
      \CROWN{} computes the linear bounds on~\(f^{(k:)}\) according to \cref{eqn:crown-rule-relu}, which yields
      \begin{align*}
        \ulA^{(k-1)}_{\,:,i}\vec{z}_i + \uld^{(k-1)}_i 
        &= \left(\pospart{\ulA^{(k)}_{\,:,i}}\ulalpha_i + \negpart{\ulA^{(k)}_{\,:,i}}\olalpha_i\right)\vec{z}_i + \negpart{\ulA^{(k)}_{\,:,i}}\olbeta_i + \uld^{(k)}_i \\
        \olA^{(k-1)}_{\,:,i}\vec{z}_i + \old^{(k-1)}_i 
        &= \left(\pospart{\olA^{(k)}_{\,:,i}}\olalpha_i + \negpart{\olA^{(k)}_{\,:,i}}\ulalpha_i\right)\vec{z}_i + \pospart{\olA^{(k)}_{\,:,i}}\olbeta_i + \old^{(k)}_i,
      \end{align*}
      where~\(\mat{A}_{\,:,i}\) denotes the~\(i\)-th column of matrix~\(\mat{A}\).
      We now differentiate three cases based on the signs of~\(\ulz^{(k-1)}_i\) and~\(\olz^{(k-1)}_i\).
      \begin{enumerate}[(i)]
        \item Consider~\(\olz^{(k-1)}_i \leq 0\).
          We have~\(\ulalpha_i = \olalpha_i = \olbeta_i = 0\).
          Therefore,
          \begin{align*}
            \ulA^{(k-1)}_{\,:,i}\vec{z}_i + \uld^{(k-1)}_i &= \uld^{(k)}_i \geq \ulv^{(K)} \\
            \olA^{(k-1)}_{\,:,i}\vec{z}_i + \old^{(k-1)}_i &= \old^{(k)}_i \leq \olv^{(K)},
          \end{align*}
          where both inequalities are due to the induction assumption with~\(\vec{z}^{(k)}_i = 0\) in \cref{eqn:crown-contained-crown-int-ext-ia}.
          We can apply the induction assumption since~\(0 = \ReLU{\olz^{(k-1)}_i} \in \HR[^{(k)}]{\vec{z}}\) and, therefore,~\(0 \in \HR[^{(k)}]{\vec{v}}\) since~\(\HR[^{(k)}]{\vec{z}} \subseteq \HR[^{(k)}]{\vec{v}}\).
        \item Consider~\(\olz^{(k-1)}_i > 0\) and~\(\ulz^{(k-1)}_i \geq 0\).
          First, we find that~\(\HR[^{(k)}_i]{\vec{v}} = \HR[^{(k-1)}_i]{\vec{v}}\) according to \cref{eqn:crown-int-ext-relu-rule} since~\(\olv^{(k-1)}_i \geq \olz^{(k-1)}_i > 0\).
          Regarding \CROWN{}, we have~\(\ulalpha_i = \olalpha_i = 1\) and~\(\olbeta_i = 0\).
          Therefore,
          \begin{align*}
            \ulA^{(k-1)}_{\,:,i}\vec{z}_i + \uld^{(k-1)}_i 
            &= \left(\pospart{\ulA^{(k)}_{\,:,i}} + \negpart{\ulA^{(k)}_{\,:,i}}\right)\vec{z}_i + \uld^{(k)}_i
            = \ulA^{(k)}_{\,:,i}\vec{z}_i + \uld^{(k)}_i \geq \ulv^{(K)} \\
            \olA^{(k-1)}_{\,:,i}\vec{z}_i + \old^{(k-1)}_i 
            &= \left(\pospart{\olA^{(k)}_{\,:,i}} + \negpart{\olA^{(k)}_{\,:,i}}\right)\vec{z}_i + \old^{(k)}_i
            = \olA^{(k)}_{\,:,i}\vec{z}_i + \old^{(k)}_i \leq \olv^{(K)},
          \end{align*}
          where the inequalities are due to the induction assumption that we can apply since~\(\vec{z}_i \in \HR[^{(k-1)}_i]{\vec{v}} = \HR[^{(k)}_i]{\vec{v}}\).
        \item Finally, consider~\(\ulz^{(k-1)}_i < 0 < \olz^{(k-1)}_i\).
          With the same argument as in the previous case, we have~\(\HR[^{(k)}_i]{\vec{v}} = \HR[^{(k-1)}_i]{\vec{v}}\).
          We have
          \begin{equation*}
            \olalpha_i = \frac{\olz^{(k-1)}_i}{\olz^{(k-1)}_i - \ulz^{(k-1)}_i}, \quad
            \olbeta_i = -\ulz^{(k-1)}_i\olalpha_i 
            = \frac{-\ulz^{(k-1)}_i\olz^{(k-1)}_i}{\olz^{(k-1)}_i - \ulz^{(k-1)}_i}
          \end{equation*}
          and~\(\ulalpha_i \in [0, 1]\).
          Now,
          \begin{subequations}
            \begin{align}
              & \ulA^{(k-1)}_{\,:,i}\vec{z}_i + \uld^{(k-1)}_i \nonumber \\
              ={}& \left(\pospart{\ulA^{(k)}_{\,:,i}}\ulalpha_i + \negpart{\ulA^{(k)}_{\,:,i}}\olalpha_i\right)\vec{z}_i + \negpart{\ulA^{(k)}_{\;:,i}}\olbeta_i + \uld^{(k)}_i \label{eqn:crown-cont-cie-iii-before-minimiser-z}\\
              \geq{}&
              \pospart{\ulA^{(k)}_{\,:,i}}\ulalpha_i \ulv^{(k-1)}_i 
              + \negpart{\ulA^{(k)}_{\,:,i}}\olalpha_i \olv^{(k-1)}_i + \negpart{\ulA^{(k)}_{\;:,i}}\olbeta_i 
              + \uld^{(k)}_i \label{eqn:crown-cont-cie-iii-minimiser-z}\\
              \geq{}&
              \pospart{\ulA^{(k)}_{\,:,i}}\ulv^{(k-1)}_i 
              + \negpart{\ulA^{(k)}_{\,:,i}}\frac{\olz^{(k-1)}_i\olv^{(k-1)}_i - \ulz^{(k-1)}_i\olz^{(k-1)}_i}{\olz^{(k-1)}_i - \ulz^{(k-1)}_i} 
              + \uld^{(k)}_i \label{eqn:crown-cont-cie-iii-minimiser-alpha}\\
              \geq{}&
              \pospart{\ulA^{(k)}_{\,:,i}}\ulv^{(k-1)}_i 
              + \negpart{\ulA^{(k)}_{\,:,i}}\frac{\olz^{(k-1)}_i\olv^{(k-1)}_i - \ulz^{(k-1)}_i\olv^{(k-1)}_i}{\olz^{(k-1)}_i - \ulz^{(k-1)}_i} 
              + \uld^{(k)}_i \label{eqn:crown-cont-cie-iii-ub-z-to-ub-v}\\
              \geq{}&
              \pospart{\ulA^{(k)}_{\,:,i}}\ulv^{(k-1)}_i 
              + \negpart{\ulA^{(k)}_{\,:,i}}\olv^{(k-1)}_i + \uld^{(k)}_i \nonumber \\
              \geq{}&
              \ulA^{(k)}_{\,:,i}\left(\mathbf{1}_{\ulA^{(k)}_{\,:,i} \geq 0}\,\ulv^{(k-1)}_i 
              + \mathbf{1}_{\ulA^{(k)}_{\,:,i} < 0}\,\olv^{(k-1)}_i\right)
              + \uld^{(k)}_i \nonumber \\
              \geq{}& \ulv^{(K)},%
              \label{eqn:crown-cont-cie-iii-ia}
            \end{align}
          \end{subequations}
          where~\(\mathbf{1}_{\ulA^{(k)}_{\,:,i} \geq 0} \in {\{0, 1\}}^{n_k}\) with
          \begin{equation*}
            {\left(\mathbf{1}_{\ulA^{(k)}_{\,:,i} \geq 0}\right)}_j = \begin{cases}
              1 & \ulA^{(k)}_{j,i} \geq 0 \\
              0 & \text{otherwise}
            \end{cases}
          \end{equation*}
          for~\(j \in [n_k]\) and~\(\mathbf{1}_{\ulA^{(k)}_{\,:,i} < 0}\) is defined equivalently.
          Above, the inequality in \cref{eqn:crown-cont-cie-iii-minimiser-z} is by choosing~\(\vec{z}_i \in \HR[^{(k-1)}]{\vec{v}}\) as the minimiser of \cref{eqn:crown-cont-cie-iii-before-minimiser-z}.
          The inequality in \cref{eqn:crown-cont-cie-iii-minimiser-alpha} is by choosing~\(\ulalpha_i = 1\) which minimises \cref{eqn:crown-cont-cie-iii-minimiser-z} since~\(\ulv^{(k-1)}_i \leq \ulz^{(k-1)}_i < 0\).
          To justify the inequality in \cref{eqn:crown-cont-cie-iii-ub-z-to-ub-v} we note that~\(-\ulz^{(k-1)}_i\negpart{\ulA^{(k)}_{\,:, i}} \leq 0\) since~\(\ulz^{(k-1)}_i < 0\) and, further~\(\olv^{(k-1)}_i \geq \olz^{(k-1)}_i\).
          Finally, \cref{eqn:crown-cont-cie-iii-ia} follows from the induction assumption that we can apply since
          \begin{equation*}
            \left(\mathbf{1}_{\ulA^{(k)}_{\,:,i} \geq 0}\,\ulv^{(k-1)}_i + \mathbf{1}_{\ulA^{(k)}_{\,:,i} < 0}\,\olv^{(k-1)}_i\right) \in \HR[^{(k-1)}]{\vec{v}}.
          \end{equation*}
          For the upper bound, we apply similar steps to obtain
          \begin{align*}
              & \olA^{(k-1)}_{\,:,i}\vec{z}_i + \old^{(k-1)}_i
              = \left(
                \pospart{\olA^{(k)}_{\,:,i}}\olalpha_i 
                + \negpart{\olA^{(k)}_{\,:,i}}\ulalpha_i
              \right)\vec{z}_i 
              + \pospart{\olA^{(k)}_{\;:,i}}\olbeta_i 
              + \old^{(k)}_i \\
              \leq{}&
              \pospart{\olA^{(k)}_{\,:,i}}\olalpha_i \olv^{(k-1)}_i 
              + \negpart{\olA^{(k)}_{\,:,i}}\ulalpha_i \ulv^{(k-1)}_i 
              + \pospart{\olA^{(k)}_{\;:,i}}\olbeta_i 
              + \old^{(k)}_i \\
              \leq{}&
              \pospart{\olA^{(k)}_{\,:,i}}\frac{\olz^{(k-1)}_i\olv^{(k-1)}_i - \ulz^{(k-1)}_i\olz^{(k-1)}_i}{\olz^{(k-1)}_i - \ulz^{(k-1)}_i} 
              + \negpart{\olA^{(k)}_{\,:,i}}\ulv^{(k-1)}_i 
              + \old^{(k)}_i \\
              \leq{}&
              \pospart{\olA^{(k)}_{\,:,i}}\frac{\olz^{(k-1)}_i\olv^{(k-1)}_i - \ulz^{(k-1)}_i\olv^{(k-1)}_i}{\olz^{(k-1)}_i - \ulz^{(k-1)}_i} 
              + \negpart{\olA^{(k)}_{\,:,i}}\ulv^{(k-1)}_i 
              + \old^{(k)}_i \\
              \leq{}&
              \pospart{\olA^{(k)}_{\,:,i}}\olv^{(k-1)}_i 
              + \negpart{\olA^{(k)}_{\,:,i}}\ulv^{(k-1)}_i 
              + \old^{(k)}_i \\
              \leq{}&
              \olA^{(k)}_{\,:,i}\left(\mathbf{1}_{\olA^{(k)}_{\,:,i} \geq 0}\,\olv^{(k-1)}_i 
              + \mathbf{1}_{\olA^{(k)}_{\,:,i} < 0}\,\ulv^{(k-1)}_i\right)
              + \old^{(k)}_i \\
              \leq{}& \olv^{(K)}.
          \end{align*}
      \end{enumerate}
  \end{itemize}
  By this induction, we have shown \cref{eqn:crown-contained-crown-int-ext-ia}.
  \Cref{prop:crown-contained-in-crown-int-ext} now follows for~\(k = 0\), where~\(\ulA^{(0)} = \ulA\),~\(\olA^{(0)} = \olA\),~\(\uld^{(0)} = \uld\),~\(\old^{(0)} = \old\) and~\(\HR[^{(0)}]{\vec{v}} = \HR{\inp}\).
\end{proof}
\Cref{prop:crown-contained-in-crown-int-ext} requires layer bounds~\(\HR[^{(1)}]{\vec{z}}, \ldots, \HR[^{(K-1)}]{\vec{z}}\) which may not be looser than the \CIE{} bounds.
As we have shown in \cref{prop:crown-int-ext-more-loose-than-ia}, this applies for \IntervalArithmetic{}.
A consequence of \cref{prop:crown-contained-in-crown-int-ext} is that this also applies to \cref{algo:crown-bootstrap}.
\begin{corollary}\label{prop:crown-bootstrap-corollary}
  Let~\(\NN\),~\(\HR{\inp}\) and~\(\HR[^{(K)}]{\vec{v}}\) be as in \cref{prop:crown-contained-in-crown-int-ext}.
  It holds that
  \begin{equation*}
    \ulv^{(K)} \leq \ulA\inp + \uld \quad\text{and}\quad
    \olA\inp + \old \leq \olv^{(K)}, \qquad \forall \inp \in \HR{\inp},
  \end{equation*}
  where~\(\ulA\inp + \uld\) and~\(\olA\inp + \old\) are the bounds on~\(\NN(\inp)\) for~\(\inp \in \HR{\inp}\) computed by \cref{algo:crown-bootstrap}.
\end{corollary}

\begin{proof}
  \Cref{prop:crown-bootstrap-corollary} follows from \cref{prop:crown-contained-in-crown-int-ext} by an inductive argument over~\(k \in [K]\).
  Let~\(\NN = f^{(K)} \circ \cdots \circ f^{(1)}\),~\(\HR{\inp}\) and~\(\HR[^{(1)}]{\vec{v}}, \ldots, \HR[^{(K)}]{\vec{v}}\) be as in \cref{prop:crown-contained-in-crown-int-ext}.

  \emph{Induction start.} Applying \cref{algo:crown-linear-bounds} to~\(f^{(:1)} = f^{(1)}\) does not require layer bounds, so that \cref{prop:crown-contained-in-crown-int-ext} holds for the (concretised) \CROWN{} bounds on~\(f^{(:1)}\).

  \emph{Induction step.} Let~\(k \in \{2, \ldots, K\}\). 
  Assume the concretised \CROWN{} bounds~\(\HR[^{(1)}]{\vec{z}}, \ldots, \HR[^{(k-1)}]{\vec{z}}\) on~\(f^{(:1)}, \ldots, f^{(:k-1)}\) satisfy \cref{prop:crown-contained-in-crown-int-ext}, that is,~\(\HR[^{(k)}]{\vec{z}} \subseteq \HR[^{(k)}]{\vec{v}}\) for every~\(k \in [k-1]\).
  Since we only require layer bounds on~\(f^{(:1)}, \ldots, f^{(:k-1)}\) to apply \CROWN{} to~\(f^{(:k)}\), \cref{prop:crown-contained-in-crown-int-ext} applies to the (concretised) \CROWN{} bounds we obtain for~\(f^{(:k)}\).
\end{proof}

We are now able to prove that \CROWN{} satisfies \cref{defn:convergent-bounds} as a corollary of \cref{prop:lipschitz-incl-iso-int-ext-is-convergent} and \cref{prop:crown-contained-in-crown-int-ext}.
\begin{corollary}\label{prop:crown-is-convergent}
  \CROWN{} satisfies \cref{defn:convergent-bounds}.
\end{corollary}
\begin{proof}
  Since \CIE{} is a Lipschitz interval extension according to \cref{prop:crown-int-ext-lip-incl-iso-int-ext}, \cref{prop:lipschitz-incl-iso-int-ext-is-convergent} applies.
  Therefore, \CIE{} satisfies \cref{defn:convergent-bounds}.
  Now, due to \cref{prop:crown-contained-in-crown-int-ext}, \CROWN{} also satisfies \cref{defn:convergent-bounds}.
\end{proof}

\section{Experiments}\label{sec:experiments-additional}
This section contains additional details on the experiments from \cref{sec:experiments} and additional experimental results.
A reproducibility package containing all datasets, networks, and raw data from our experiments is available at \url{https://doi.org/10.5281/zenodo.15521583}.
Our code is also available at \url{https://github.com/sen-uni-kn/probspecs}.

\subsection{Hardware}\label{sec:details-hardware}
We run all our experiments on a workstation running Ubuntu 22.04 with an Intel i7--4820K CPU, 32 GB of memory and no GPU (HW1). 
This CPU model is ten years old (introduced in late 2013) and has four cores and eight threads.

In comparison,~\textcite{ConverseFilieriGopinathEtAl2020} use an AMD EPYC 7401P CPU for their ACAS Xu robustness experiments, limiting their tool to use 46 threads and at most 4GB of memory. 
AMD EPYC 7401P was introduced in mid-2017, targeting servers.
\Textcite{MarzariRoncolatoFarinelli2023} use a Ubuntu 22.04 workstation with an Intel i5--13600KF CPU and an Nvidia GeForce RTX 4070 Ti GPU.\@
This CPU was introduced in end-2022 and has 14 cores with 20 threads.

\Cref{fig:hardware-compare-versus} contains the CPU comparison from \url{versus.com}.
As the figure shows, our~HW1 is less performant when considering both the theoretical performance according to the hardware specification and the actual performance on a series of computational benchmarks.
Therefore, HW1 is inferior in computation power to the hardware used by~\textcite{ConverseFilieriGopinathEtAl2020} and~\textcite{MarzariRoncolatoFarinelli2023}.

\begin{figure}
  \centering
  \hfill
  \begin{subfigure}{.4\textwidth}
    \centering
    \includegraphics[width=\textwidth]{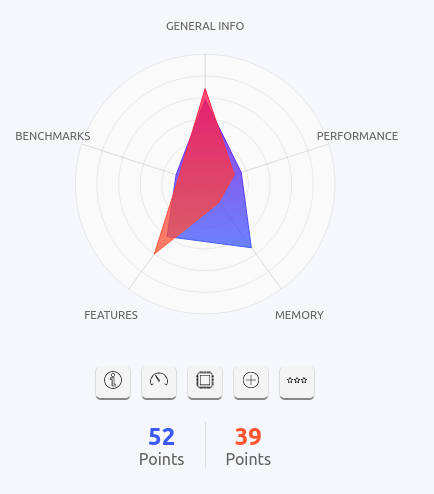}
    \caption{HW1~\legendcolorbox{VersusComOurs} vs. \Textcite{ConverseFilieriGopinathEtAl2020}~\legendcolorbox{VersusComTheirs}} 
  \end{subfigure}
  \hfill
  \begin{subfigure}{.4\textwidth}
    \centering
    \includegraphics[width=\textwidth]{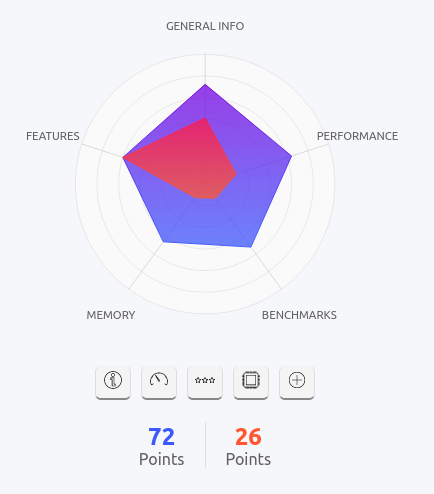}
    \caption{HW1~\legendcolorbox{VersusComOurs} vs. \Textcite{MarzariRoncolatoFarinelli2023}~\legendcolorbox{VersusComTheirs}}
  \end{subfigure}
  \hfill
  \tikzexternalenable%
  \tikzsetexternalprefix{hardware-legendcolorbox}
  \caption{%
    Hardware comparison from \url{versus.com}. 
    The figures are taken from \url{https://versus.com/en/amd-epyc-7401p-vs-intel-core-i7-6700}
    and \url{https://versus.com/en/intel-core-i5-13600kf-vs-intel-core-i7-4820k}, respectively.
    Accessed on the 17th of May 2024.
  }\label{fig:hardware-compare-versus}
  \tikzexternaldisable%
\end{figure}

\subsection{FairSquare Benchmark}\label{sec:fairsquare-additional}
We provide additional details on the FairSquare benchmark and present additional results.

\subsubsection{Extended Description of the Benchmark}
The input space in the FairSquare benchmark consists of three unbounded continuous variables for the age, the years of education (\enquote{edu}), and the yearly gain in capital of a person (\enquote{capital gain}), respectively, as well as an additional discrete protected variable indicating a person's (assumed binary) sex.
The neural networks use \enquote{age} and \enquote{edu} or \enquote{age}, \enquote{edu} and \enquote{capital gain}, depending on the network architecture, to predict whether a person has a high salary (higher than \$50\,000).
The three input distributions used by the FairSquare benchmark are
\begin{enumerate}
  \item the combination of three independent normal distributions for the continuous variables and a Bernoulli distribution for the person's sex, 
  \item a Bayesian Network with the structure in \cref{fig:bayes-net-fairsquare} introducing correlations between the variables that are similarly distributed as for the independent input distribution (\enquote{Bayes Net 1}), and 
  \item the same Bayesian Network augmented with a constraint that the years of education may not exceed a person's age (\enquote{Bayes Net 2}).
\end{enumerate}
The integrity constraint is implemented by~\textcite{AlbarghouthiDAntoniDrewsEtAl2017} as a post-processing of the samples from the Bayesian Network\@.
In particular, a person's years of education are set to the person's age if the sampled years of education are larger than the sampled age.
We introduce a pre-processing layer before the neural network we want to verify to implement this integrity constraint in our setting.
This pre-processing layer computes~\(\mathrm{edu}' = \min(\mathrm{edu}, \mathrm{age})\) and leaves the remaining inputs unchanged.

Besides experiments on the demographic parity fairness notion, the extended FairSquare benchmarks also include experiments on the parity of qualified persons fairness notion as defined in \cref{sec:specs-parity-of-qualified}.
For both fairness notions, the FairSquare benchmark uses a fairness threshold of~\(\gamma = 0.85\).

\begin{figure}
  \centering
  \begin{tikzpicture}[minimum size=1cm]
    \node[draw=black, circle] (S) at (0,0) {sex};
    \node[draw=black, circle, below=of S] (C) {cg};
    \node[draw=black, circle, below left=of C] (A) {age};
    \node[draw=black, circle, below right=of C] (E) {edu};

    \draw[black,-Stealth] (S) -- (C);
    \draw[black,-Stealth] (C) -- (A);
    \draw[black,-Stealth] (C) -- (E);
  \end{tikzpicture}
  \caption{%
    FairSquare Bayes Net 1.
    The network structure of the Bayesian Network from the FairSquare benchmark.
    In this figure, \enquote{cg} denotes the capital gain variable and \enquote{edu} denotes the number of years of education.
  }\label{fig:bayes-net-fairsquare}
\end{figure}

\subsubsection{Extended Results}
\Cref{tab:fairsquare-non-clipped} contains the comparison of \ToolName{} and \FairSquare{} on the extended set of FairSquare benchmarks from \textcite{AlbarghouthiDAntoniDrewsEtAl2017}.
The runtime from this table is visualised in \cref{fig:fairsquare}.

\begin{table}
  \centering
  \caption{%
    FairSquare benchmark results.
    The first two columns indicate the neural network~\(\NN\) that is verified and the input probability distribution~\(\prob[\inp]\), respectively.
    We verify two fairness notions: demographic parity and parity of qualified persons.
    For each fairness notion, the first two columns contain the runtime of \ToolName{} and \FairSquare{} in seconds.
    Here, \enquote{TO} indices timeout (900s).
    The last column contains the result of the verification with \ToolName{}.
   }\label{tab:fairsquare-non-clipped}
  \vspace*{0.1in}
  \begin{tabular}{cc@{\hspace{2em}}rr@{\hspace{1em}}c@{\hspace{2em}}rr@{\hspace{1em}}c}
    \toprule
    && && &
    \multicolumn{3}{c}{\textbf{Parity of}} \\
    && 
    \multicolumn{3}{c}{\textbf{Demographic Parity}} &
    \multicolumn{3}{c}{\textbf{Qualified Persons}}
    \\ \cmidrule(lr){3-5}\cmidrule(lr){6-8}
    \multicolumn{2}{c}{\textbf{Benchmark Instance}} &
    \multicolumn{2}{c}{\textbf{Runtime (s)}} & &
    \multicolumn{2}{c}{\textbf{Runtime (s)}} \\
    $\NN$ & $\prob[\inp]$ &
    \textbf{\ToolName{}} & \textbf{FairSquare} & \textbf{Fair?} &
    \textbf{\ToolName{}} & \textbf{FairSquare} & \textbf{Fair?} \\ \midrule
    NN\textsubscript{2,1} & independent &          2 & 2   & \exsuccess{} & \textbf{2}    & 3 & \exsuccess{} \\
    NN\textsubscript{2,1} & Bayes Net 1 & \textbf{4} & 76  & \exsuccess{} & \textbf{8}    & 38 & \exsuccess{} \\
    NN\textsubscript{2,1} & Bayes Net 2 & \textbf{5} & 51  & \exsuccess{} & \textbf{12}   & 26 & \exsuccess{} \\[.25em]
    NN\textsubscript{2,2} & independent & \textbf{2} & 4   & \exsuccess{} & \textbf{2}    & 7 & \exsuccess{} \\
    NN\textsubscript{2,2} & Bayes Net 1 & \textbf{4} & 33  & \exsuccess{} & \textbf{11}   & 60 & \exsuccess{} \\
    NN\textsubscript{2,2} & Bayes Net 2 & \textbf{6} & 59  & \exsuccess{} & \textbf{17}   & 61 & \exsuccess{} \\[.25em]
    NN\textsubscript{3,2} & independent & \textbf{2} & 451 & \exsuccess{} & \textbf{1}    & 657 & \exsuccess{} \\
    NN\textsubscript{3,2} & Bayes Net 1 & \textbf{4} & TO  & \exsuccess{}{} & \textbf{40} & TO  & \exsuccess{} \\
    NN\textsubscript{3,2} & Bayes Net 2 & \textbf{5} & TO  & \exsuccess{}{} & \textbf{57} & TO & \exsuccess{} \\ 
    \bottomrule
  \end{tabular}
\end{table}

\subsection{ACAS Xu Safety}\label{sec:acasxu-safety-additional}
We provide a comparison of \ProbBounds{} with \eProVe{}~\cite{MarzariCorsiMarchesiniEtAl2024} on all 36 ACAS~Xu networks to which property~\(\phi_2\) of \textcite{KatzBarrettDillEtAl2017} applies.
These results are contained in \cref{tab:acasxu-safety-eprove-full}, revealing that the networks in \cref{tab:cx-counting-compare} are outliers.
However, \ProbBounds{} still computes tighter sound bounds faster than \eProVe{} computes probably sound bounds for 12 of the 36 networks.

\Cref{tab:acasxu-safety-long} contains additional results for running \ProbBounds{} with a finer grid of time budgets on the networks from \cref{tab:cx-counting-compare} and two additional challenging instances from \textcite{KatzBarrettDillEtAl2017}.

\begin{table}
\centering
  \caption[Comparison of ProbBounds and eProVe for ACAS~Xu Safety]{%
    Comparison of \ProbBounds{} and \eProVe{} for ACAS~Xu safety.
    A \ProbBounds{} column is marked with an arrow (\(\ms\)) if \ProbBounds{} computes a tighter upper bound than \eProVe{} within a certain time budget. 
    If this is not the case for any time budget, the \eProVe{} entry is marked with an arrow (\(\ms\)).
    In total, \ProbBounds{} computes tighter upper bounds within 60s in 16 cases.
    In 12 cases, it even computes a tighter bound faster than \eProVe{}.
  }\label{tab:acasxu-safety-eprove-full}
  \vspace*{0.1in}
\begin{tabular}{llc@{}lc@{}lc@{}lc@{}lr}
 \toprule
 & & \multicolumn{6}{c}{\bfseries \ProbBounds{}} & \multicolumn{3}{c}{\bfseries \eProVe{}} \\ \cmidrule(lr){3-8}\cmidrule(lr){9-11}
 & & \multicolumn{2}{c}{\bfseries 10s} & \multicolumn{2}{c}{\bfseries 30s} & \multicolumn{2}{c}{\bfseries 60s} & \multicolumn{3}{c}{\bfseries 99.9\% confid.}
\\
 $\phi$ & $\NN$ & $\ell,u$ & & $\ell,u$ & & $\ell,u$ & & $u$ & & \multicolumn{1}{c}{Rt} \\
\midrule
\multirow[t]{36}{*}{$\varphi_{2}$} 
 & $N_{2,1}$ & $  0.03\%,   3.19\%$ &     & $  0.09\%,   2.50\%$ &$\ms$& $  0.16\%,   2.09\%$ &     & $  2.58\%$ &    &  64s  \\
 & $N_{2,2}$ & $  0.00\%,   3.86\%$ &     & $  0.08\%,   2.93\%$ &     & $  0.24\%,   2.64\%$ &     & $  2.49\%$ &$\ms$&  44s \\
 & $N_{2,3}$ & $  0.51\%,   3.46\%$ &$\ms$& $  0.85\%,   2.98\%$ &     & $  1.00\%,   2.68\%$ &     & $  3.58\%$ &    &  55s  \\
 & $N_{2,4}$ & $  0.00\%,   2.72\%$ &$\ms$& $  0.03\%,   2.13\%$ &     & $  0.04\%,   2.06\%$ &     & $  2.78\%$ &    &  66s  \\
 & $N_{2,5}$ & $  0.03\%,   4.23\%$ &     & $  0.12\%,   3.69\%$ &     & $  0.34\%,   3.14\%$ &     & $  3.06\%$ &$\ms$&  47s \\
 & $N_{2,6}$ & $  0.01\%,   3.67\%$ &     & $  0.16\%,   2.89\%$ &     & $  0.30\%,   2.50\%$ &     & $  2.32\%$ &$\ms$&  45s \\
 & $N_{2,7}$ & $  0.28\%,   5.56\%$ &     & $  0.84\%,   4.87\%$ &     & $  1.26\%,   4.32\%$ &     & $  4.04\%$ &$\ms$&  47s \\
 & $N_{2,8}$ & $  0.96\%,   3.75\%$ &     & $  1.23\%,   3.00\%$ &$\ms$& $  1.37\%,   2.79\%$ &     & $  3.28\%$ &    &  53s  \\
 & $N_{2,9}$ & $  0.01\%,   2.85\%$ &     & $  0.08\%,   1.65\%$ &     & $  0.10\%,   1.39\%$ &     & $  0.58\%$ &$\ms$&  11s \\
 & $N_{3,1}$ & $  0.23\%,   4.25\%$ &     & $  0.46\%,   3.41\%$ &     & $  0.88\%,   2.74\%$ &$\ms$& $  2.84\%$ &    &  40s  \\
 & $N_{3,2}$ & $  0.00\%,   3.57\%$ &     & $  0.00\%,   1.49\%$ &     & $  0.00\%,   1.15\%$ &     & $  0.00\%$ &$\ms$&   1s \\
 & $N_{3,3}$ & $  0.00\%,   2.94\%$ &     & $  0.00\%,   1.91\%$ &     & $  0.00\%,   0.87\%$ &     & $  0.00\%$ &$\ms$&   1s \\
 & $N_{3,4}$ & $  0.00\%,   3.27\%$ &     & $  0.09\%,   2.07\%$ &     & $  0.13\%,   1.53\%$ &     & $  1.36\%$ &$\ms$&  31s \\
 & $N_{3,5}$ & $  0.31\%,   2.90\%$ &     & $  0.53\%,   2.29\%$ &$\ms$& $  0.63\%,   2.02\%$ &     & $  2.67\%$ &    &  42s  \\
 & $N_{3,6}$ & $  0.00\%,   6.23\%$ &     & $  0.00\%,   5.42\%$ &     & $  0.01\%,   5.19\%$ &     & $  2.69\%$ &$\ms$&  32s \\
 & $N_{3,7}$ & $  0.00\%,   5.76\%$ &     & $  0.00\%,   4.59\%$ &     & $  0.00\%,   3.36\%$ &     & $  0.91\%$ &$\ms$&  30s \\
 & $N_{3,8}$ & $  0.15\%,   5.45\%$ &     & $  0.21\%,   4.43\%$ &     & $  0.32\%,   3.48\%$ &     & $  2.79\%$ &$\ms$&  61s \\
 & $N_{3,9}$ & $  0.28\%,   4.83\%$ &     & $  1.06\%,   4.00\%$ &     & $  1.39\%,   3.72\%$ &     & $  3.52\%$ &$\ms$&  32s \\
 & $N_{4,1}$ & $  0.00\%,   3.02\%$ &     & $  0.01\%,   1.96\%$ &     & $  0.04\%,   1.57\%$ &     & $  1.16\%$ &$\ms$&  26s \\
 & $N_{4,2}$ & $  0.00\%,   2.62\%$ &     & $  0.00\%,   1.69\%$ &     & $  0.00\%,   1.36\%$ &     & $  0.00\%$ &$\ms$&   1s \\
 & $N_{4,3}$ & $  0.17\%,   2.92\%$ &     & $  0.34\%,   2.65\%$ &     & $  0.61\%,   2.27\%$ &$\ms$& $  2.43\%$ &    &  41s  \\
 & $N_{4,4}$ & $  0.05\%,   2.47\%$ &     & $  0.17\%,   2.22\%$ &     & $  0.27\%,   1.98\%$ &     & $  1.94\%$ &$\ms$&  40s \\
 & $N_{4,5}$ & $  0.00\%,   4.52\%$ &     & $  0.14\%,   3.50\%$ &     & $  0.46\%,   3.04\%$ &     & $  2.85\%$ &$\ms$&  41s \\
 & $N_{4,6}$ & $  0.32\%,   4.54\%$ &     & $  0.89\%,   3.54\%$ &     & $  1.16\%,   3.27\%$ &     & $  3.22\%$ &$\ms$&  39s \\
 & $N_{4,7}$ & $  0.09\%,   4.17\%$ &     & $  0.31\%,   3.32\%$ &     & $  0.55\%,   2.92\%$ &     & $  2.58\%$ &$\ms$&  30s \\
 & $N_{4,8}$ & $  0.15\%,   3.91\%$ &     & $  0.56\%,   3.20\%$ &$\ms$& $  0.85\%,   2.84\%$ &     & $  3.51\%$ &    &  52s  \\
 & $N_{4,9}$ & $  0.00\%,   3.36\%$ &     & $  0.00\%,   1.96\%$ &     & $  0.00\%,   1.55\%$ &     & $  0.39\%$ &$\ms$& 11s  \\
 & $N_{5,1}$ & $  0.02\%,   2.60\%$ &     & $  0.30\%,   2.13\%$ &$\ms$& $  0.42\%,   1.99\%$ &     & $  2.14\%$ &    & 35s   \\
 & $N_{5,2}$ & $  0.17\%,   2.66\%$ &$\ms$& $  0.47\%,   2.03\%$ &     & $  0.52\%,   1.92\%$ &     & $  3.22\%$ &    & 67s   \\
 & $N_{5,3}$ & $  0.00\%,   1.72\%$ &     & $  0.00\%,   0.73\%$ &     & $  0.00\%,   0.56\%$ &     & $  0.00\%$ &$\ms$&  1s  \\
 & $N_{5,4}$ & $  0.01\%,   2.94\%$ &     & $  0.14\%,   2.34\%$ &     & $  0.22\%,   2.08\%$ &$\ms$& $  2.31\%$ &    & 54s   \\
 & $N_{5,5}$ & $  0.74\%,   3.29\%$ &     & $  1.32\%,   2.68\%$ &$\ms$& $  1.37\%,   2.62\%$ &     & $  2.77\%$ &    & 33s   \\
 & $N_{5,6}$ & $  0.18\%,   4.46\%$ &$\ms$& $  0.78\%,   3.19\%$ &     & $  0.98\%,   3.02\%$ &     & $  4.82\%$ &    & 74s   \\
 & $N_{5,7}$ & $  1.19\%,   4.63\%$ &     & $  1.74\%,   4.02\%$ &$\ms$& $  1.99\%,   3.75\%$ &     & $  4.27\%$ &    & 45s   \\
 & $N_{5,8}$ & $  0.89\%,   4.16\%$ &     & $  1.31\%,   3.39\%$ &$   $& $  1.55\%,   3.10\%$ &$\ms$& $  3.38\%$ &    & 42s   \\
 & $N_{5,9}$ & $  0.62\%,   3.75\%$ &     & $  1.24\%,   3.06\%$ &$\ms$& $  1.41\%,   2.89\%$ &     & $  3.06\%$ &    & 36s   \\
  \midrule
  \multicolumn{2}{r}{$\#\!\ms$} & \multicolumn{2}{r}{4} & \multicolumn{2}{r}{8} & \multicolumn{2}{r}{5} & \multicolumn{2}{r}{20} \\
  \bottomrule
  \end{tabular}
\end{table}

\begin{table}
  \centering
  \caption{%
    Extended \ProbBounds{} results for ACAS~Xu safety.
  }\label{tab:acasxu-safety-long}
  \vspace*{0.1in}
\begin{tabular}{llcrcrcr}
  \toprule
 & & \multicolumn{6}{c}{\bfseries Timeout} \\ \cmidrule(lr){3-8}
 & & \multicolumn{2}{c}{\bfseries 10s} & \multicolumn{2}{c}{\bfseries 30s} & \multicolumn{2}{c}{\bfseries 1m} \\ \cmidrule(lr){3-4} \cmidrule(lr){5-6} \cmidrule(lr){7-8}
 $\phi$ & $\NN$ & $\ell,u$ & $u-\ell$ & $\ell,u$ & $u-\ell$ & $\ell,u$ & $u-\ell$ \\
\midrule
\multirow[t]{3}{*}{$\phi_{2}$} 
 & $N_{4,3}$ & $  0.17\%,   2.92\%$ & $  2.74\%$ & $  0.34\%,   2.62\%$ & $  2.30\%$ & $  0.61\%,   2.27\%$ & $  1.66\%$  \\
 & $N_{4,9}$ & $  0.00\%,   3.36\%$ & $  3.36\%$ & $  0.00\%,   1.96\%$ & $  1.96\%$ & $  0.00\%,   1.55\%$ & $  1.55\%$  \\
 & $N_{5,8}$ & $  0.89\%,   4.16\%$ & $  3.26\%$ & $  1.31\%,   3.39\%$ & $  2.08\%$ & $  1.55\%,   3.10\%$ & $  1.55\%$  \\[.25em]
$\phi_{7}$ 
 & $N_{1,9}$ & $  0.00\%,  98.71\%$ & $ 98.71\%$ & $  0.00\%,  94.18\%$ & $ 94.18\%$ & $  0.00\%,  87.62\%$ & $ 87.62\%$  \\
$\phi_{8}$ 
 & $N_{2,9}$ & $  0.00\%,  76.39\%$ & $ 76.39\%$ & $  0.00\%,  65.83\%$ & $ 65.83\%$ & $  0.00\%,  58.11\%$ & $ 58.11\%$ \\[.75em]
 & &  \multicolumn{2}{c}{\bfseries 10m} & \multicolumn{2}{c}{\bfseries 1h} \\\cmidrule(lr){3-4}\cmidrule(lr){5-6}
 $\phi$ & $\NN$ & $\ell,u$ & $u-\ell$ & $\ell,u$ & $u-\ell$  \\
\midrule
\multirow[t]{3}{*}{$\phi_{2}$}  
 & $N_{4,3}$ & $  0.95\%,   1.93\%$ & $  0.98\%$ & $  1.12\%,   1.75\%$ & $  0.63\%$ \\
 & $N_{4,9}$ & $  0.03\%,   0.52\%$ & $  0.48\%$ & $  0.08\%,   0.29\%$ & $  0.21\%$ \\
 & $N_{5,8}$ & $  1.90\%,   2.66\%$ & $  0.76\%$ & $  1.97\%,   2.57\%$ & $  0.60\%$ \\[.25em]
$\phi_{7}$
 & $N_{1,9}$ & $  0.00\%,  51.93\%$ & $ 51.93\%$ & $  0.00\%,  30.71\%$ & $ 30.71\%$ \\
$\phi_{8}$ 
 & $N_{2,9}$ & $  0.00\%,  34.60\%$ & $ 34.60\%$ & $  0.00\%,  15.33\%$ & $ 15.33\%$ \\
\bottomrule
\end{tabular}  
\end{table}

\subsection{ACAS Xu Robustness}\label{sec:acasxu-robustness-additional}
We first provide a more detailed description of the ACAS~Xu robustness benchmark.
As \textcite{ConverseFilieriGopinathEtAl2020}, we consider 25 reference inputs~---~five for each class~---~and allow these reference inputs to be perturbed in the first two dimensions by at most~\(5\%\) of the diameter of the input space in the respective dimension.
To compute bounds on the output distribution, we bound the probability of each of the five ACAS~Xu classes for each of the 25 inputs.
Since the ACAS~Xu training data is not publicly available, we sample the reference inputs randomly.

\Cref{tab:acasxu-robust-full} contains the bounds computed by \ProbBounds{} for each reference input and each output class: COC (Clear-of-Conflict), WL (steer Weak Left), WR (steer Weak Right), SL (steer Strong Left), and SR (steer Strong Right).
The table reveals that the ACAS~Xu network~\(N_{1,1}\) could tend to classify inputs as Clear-of-Conflict (COC), regardless of the class assigned to the reference input.
This insight does not agree with the insight drawn by~\textcite{ConverseFilieriGopinathEtAl2020}. 
However, both results are based on a tiny sample of the input space of only 25 points. 
Obtaining valid results requires considering a significantly larger sample of the input space or quantifying global robustness~\cite{KatzBarrettDillEtAl2017a}.

\begin{sidewaystable}
  \centering
  \footnotesize
  \caption{%
    ACAS~Xu robustness results.
    We report the lower and upper bounds (\(\ell, u\)) \ProbBounds{} computes for the probability that a perturbed input is classified as the target label and the runtime in seconds (Rt) of \ProbBounds{} for computing these bounds.
    For each combination of unperturbed label, input, and target label, we run \ProbBounds{} until the difference between the lower and the upper bound is at most~\(0.1\%\).
   }\label{tab:acasxu-robust-full}
  \vspace*{0.1in}
  \begin{tabular}{cc@{\hspace{2em}}crcrcrcrcr}
   \toprule
    & & \multicolumn{10}{c}{\textbf{Target Label}} \\ \cmidrule(lr){3-12}
    \multicolumn{2}{c}{\textbf{Reference Input}} & \multicolumn{2}{c}{\textbf{COC}} & \multicolumn{2}{c}{\textbf{WL}} & \multicolumn{2}{c}{\textbf{WR}} & \multicolumn{2}{c}{\textbf{SL}} & \multicolumn{2}{c}{\textbf{SR}} \\\cmidrule{1-2}\cmidrule(lr){3-4}\cmidrule(lr){5-6}\cmidrule(lr){7-8}\cmidrule(lr){9-10}\cmidrule(lr){11-12}
    \textbf{Label} & \textbf{Input \#} 
    & $\ell, u$ & \multicolumn{1}{c}{Rt} & $\ell, u$ & \multicolumn{1}{c}{Rt} & $\ell, u$ & \multicolumn{1}{c}{Rt} & $\ell, u$ & \multicolumn{1}{c}{Rt} & $\ell, u$ & \multicolumn{1}{c}{Rt} \\\midrule
COC & 1 & $ 100\%,  100\%$ &    0s & $0.0\%, 0.0\%$ &    0s & $0.0\%, 0.0\%$ &    0s & $0.0\%, 0.0\%$ &    0s & $0.0\%, 0.0\%$ &    0s \\
COC & 2 & $99.9\%,  100\%$ &    1s & $0.0\%, 0.0\%$ &    1s & $0.0\%, 0.1\%$ &    1s & $0.0\%, 0.0\%$ &    1s & $0.0\%, 0.0\%$ &    1s \\
COC & 3 & $91.7\%, 91.8\%$ &    3s & $1.3\%, 1.4\%$ &    6s & $1.9\%, 2.0\%$ &    4s & $2.1\%, 2.2\%$ &   12s & $2.8\%, 2.8\%$ &    4s \\
COC & 4 & $ 100\%,  100\%$ &    0s & $0.0\%, 0.0\%$ &    0s & $0.0\%, 0.0\%$ &    0s & $0.0\%, 0.0\%$ &    0s & $0.0\%, 0.0\%$ &    0s \\
COC & 5 & $ 100\%,  100\%$ &    0s & $0.0\%, 0.0\%$ &    0s & $0.0\%, 0.0\%$ &    0s & $0.0\%, 0.0\%$ &    0s & $0.0\%, 0.0\%$ &    0s \\
WL  & 1 & $89.9\%, 90.0\%$ &    2s & $1.6\%, 1.7\%$ &   38s & $4.7\%, 4.8\%$ &    2s & $3.6\%, 3.7\%$ &   39s & $0.0\%, 0.0\%$ &    1s \\
WL  & 2 & $92.9\%, 93.0\%$ &    3s & $2.3\%, 2.4\%$ &    2s & $3.0\%, 3.1\%$ &    3s & $1.4\%, 1.5\%$ &    3s & $0.2\%, 0.3\%$ &    1s \\
WL  & 3 & $90.4\%, 90.5\%$ &    3s & $1.1\%, 1.2\%$ &    8s & $3.6\%, 3.7\%$ &    4s & $3.8\%, 3.9\%$ &   11s & $0.8\%, 0.9\%$ &    2s \\
WL  & 4 & $83.6\%, 83.7\%$ &   18s & $5.4\%, 5.5\%$ &   52s & $2.6\%, 2.7\%$ &   52s & $4.6\%, 4.7\%$ &   60s & $3.5\%, 3.6\%$ &   65s \\
WL  & 5 & $96.8\%, 96.9\%$ &    2s & $2.9\%, 3.0\%$ &    1s & $0.0\%, 0.1\%$ &    1s & $0.0\%, 0.0\%$ &    1s & $0.1\%, 0.2\%$ &    1s \\
WR  & 1 & $62.2\%, 62.3\%$ &   18s & $0.0\%, 0.0\%$ &    1s & $3.4\%, 3.5\%$ &   19s & $15.8\%, 15.9\%$ &    9s & $18.5\%, 18.6\%$ &   12s \\
WR  & 2 & $96.6\%, 96.7\%$ &    1s & $1.5\%, 1.6\%$ &    1s & $1.7\%, 1.8\%$ &    2s & $0.0\%, 0.1\%$ &    1s & $0.0\%, 0.1\%$ &    1s \\
WR  & 3 & $94.8\%, 94.9\%$ &    3s & $0.0\%, 0.1\%$ &   28s & $1.9\%, 2.0\%$ &  213s & $0.0\%, 0.1\%$ &    5s & $3.2\%, 3.3\%$ &  123s \\
WR  & 4 & $94.3\%, 94.4\%$ &    2s & $1.8\%, 1.9\%$ &    2s & $3.4\%, 3.5\%$ &    2s & $0.3\%, 0.4\%$ &    2s & $0.0\%, 0.1\%$ &    1s \\
WR  & 5 & $91.1\%, 91.1\%$ &    4s & $0.8\%, 0.9\%$ &    7s & $4.2\%, 4.3\%$ &    9s & $3.0\%, 3.1\%$ &   12s & $0.7\%, 0.8\%$ &    4s \\
SL  & 1 & $81.1\%, 81.2\%$ &   22s & $2.3\%, 2.4\%$ &   77s & $4.8\%, 4.9\%$ &  155s & $4.8\%, 4.9\%$ &   64s & $6.7\%, 6.8\%$ &  138s \\
SL  & 2 & $93.8\%, 93.9\%$ &    3s & $1.1\%, 1.2\%$ &   26s & $3.7\%, 3.8\%$ &   44s & $0.7\%, 0.8\%$ &   11s & $0.5\%, 0.6\%$ &   16s \\
SL  & 3 & $83.3\%, 83.4\%$ &   10s & $2.8\%, 2.9\%$ &   15s & $6.5\%, 6.6\%$ &   41s & $2.1\%, 2.2\%$ &   12s & $5.0\%, 5.1\%$ &   28s \\
SL  & 4 & $81.8\%, 81.9\%$ &   18s & $3.3\%, 3.4\%$ &   94s & $5.2\%, 5.3\%$ &  135s & $5.6\%, 5.7\%$ &   85s & $4.0\%, 4.1\%$ &   58s \\
SL  & 5 & $84.9\%, 85.0\%$ &   13s & $1.7\%, 1.8\%$ &   25s & $5.4\%, 5.5\%$ &   40s & $3.0\%, 3.1\%$ &   26s & $4.8\%, 4.9\%$ &   39s \\
SR  & 1 & $91.0\%, 91.0\%$ &    5s & $0.0\%, 0.1\%$ &    3s & $5.2\%, 5.3\%$ &   65s & $0.0\%, 0.1\%$ &    2s & $3.6\%, 3.7\%$ &   34s \\
SR  & 2 & $87.0\%, 87.1\%$ &    7s & $0.9\%, 1.0\%$ &    5s & $5.7\%, 5.8\%$ &   71s & $1.2\%, 1.3\%$ &    6s & $5.0\%, 5.1\%$ &   88s \\
SR  & 3 & $88.4\%, 88.5\%$ &    3s & $2.5\%, 2.6\%$ &    9s & $4.7\%, 4.8\%$ &    9s & $2.6\%, 2.7\%$ &   15s & $1.5\%, 1.6\%$ &    9s \\
SR  & 4 & $93.0\%, 93.1\%$ &    6s & $0.7\%, 0.8\%$ &   62s & $3.7\%, 3.8\%$ &  119s & $1.0\%, 1.1\%$ &   30s & $1.4\%, 1.5\%$ &   41s \\
SR  & 5 & $79.3\%, 79.4\%$ &    6s & $0.2\%, 0.3\%$ &    2s & $9.1\%, 9.2\%$ &   10s & $3.7\%, 3.8\%$ &    5s & $7.4\%, 7.5\%$ &    6s \\
\bottomrule
  \end{tabular}
\end{sidewaystable}

\subsection{MiniACSIncome}\label{sec:miniacsincome-extra}
We provide additional details on the MiniACSIncome benchmark, including how we construct the dataset and train the networks.

\subsubsection{Benchmark}
To create probabilistic verification problems of increasing difficulty, we consider an increasing number of input variables from ACSIncome.
The smallest instance, MiniACSIncome-1, only contains the binary \enquote{SEX} variable. 
In contrast, the largest instance, MiniACSIncome-8, contains \enquote{SEX} and seven more variables from ACSIncome, including age, education, and working hours per week.
We train a neural network with a single layer of ten neurons for each MiniACSIncome-$i$,~\(i \in [8]\). 
For MiniACSIncome-4, we additionally train deeper and wider networks to investigate the scalability of \ToolName{} with respect to the network size.
We fit a Bayesian Network to the MiniACSIncome-8 dataset to obtain an input distribution.
The process is described in detail in \cref{sec:miniacsincome-bayes-net}.
We use this distribution for all MiniACSIncome-$i$ instances by taking the variables not contained in MiniACSIncome-$i$ as latent variables.
The verification problem is then to verify the demographic parity of a neural network with respect to \enquote{SEX} under this input distribution.

\subsubsection{Dataset}
The MiniACSIncome benchmarks are built by sampling about 100\,000 entries from the ACS PUMPS 1-Year horizon data for all states of the USA for the year 2018 using the \texttt{folktables} Python package~\cite{DingHardtMillerEtAl2021}.
In line with ACSIncome~\cite{DingHardtMillerEtAl2021}, we only sample individuals older than 16 years, with a yearly income of at least \$100, reported working hours per week of at least 1, and a \enquote{PWGTP} (more details in~\textcite{DingHardtMillerEtAl2021}) of at least 1.
In total, our dataset contains 102\,621 samples.

\subparagraph{Variable Order.}
For obtaining the benchmark MiniACSIncome-$i$,~\(i \in [8]\), we select~\(i\) input variables in the following order: \enquote{SEX}, \enquote{COW}, \enquote{SCHL}, \enquote{WKHP}, \enquote{MAR}, \enquote{RAC1P}, \enquote{RELP}, \enquote{AGEP}.
We choose \enquote{SEX} as the first variable so that we can verify the fairness with respect to \enquote{SEX} on every benchmark instance.
The order of the remaining variables is chosen based on each variable's expected predictive value and the number of discrete values.
In particular, we select \enquote{COW} (class of work), \enquote{SCHL} (level of education), and \enquote{WKHP} (work hours per week) first, as we consider these variables to be more predictive than \enquote{MAR} (marital status), \enquote{RAC1P} (races of a person), \enquote{RELP} (relationship), and \enquote{AGEP} (age of a person).
The variables are ordered by their number of discrete values within these groups of expected predictive value.
For example, \enquote{COW} has nine categories, while \enquote{WKHP} has 99 possible integer values.

\Cref{tab:MiniACSIncome-results} contains the number of input dimensions and the total number of discrete values in each MiniACSIncome-$i$ input space.
The input space contains more dimensions than input variables due to the one-hot encoding of all categorical variables.

\begin{table}
  \centering
  \caption{MiniACSIncome details and \ToolName{} verification results.}\label{tab:MiniACSIncome-results}
  \vspace*{0.1in}
  \begin{tabular}{rcccccccc}
     \toprule
     & \multicolumn{8}{c}{\textbf{\#Input Variables}} \\
     & 1 & 2 & 3 & 4 & 5 & 6 & 7 & 8 \\ \midrule
     \textbf{\#Input Dimensions} & 2 & 10 & 34 & 35 & 40 & 49 & 67 & 68 \\
     \textbf{\#Discrete Values} &
     2 & 16 & 382 & 38K & 190K & 1.7M & 31M & 2B \\
     \textbf{\ToolName{} Runtime} & 
     16s & 44s & 127s & 231s & 374s & 699s & 1383s & TO
     \\
     \textbf{10-Neuron Network Fair?} &
     \textcolor{failure-color}{\exfailure} & \textcolor{failure-color}{\exfailure} & \textcolor{failure-color}{\exfailure} & \textcolor{failure-color}{\exfailure} & \textcolor{failure-color}{\exfailure} & \textcolor{failure-color}{\exfailure} & \textcolor{failure-color}{\exfailure} & \exunknown{} \\
     \bottomrule
  \end{tabular}
\end{table}

\subsubsection{Input Distribution}\label{sec:miniacsincome-bayes-net}
The Bayesian Network input distribution of MiniACSIncome has the network structure depicted in \cref{fig:bayes-net-miniacsincome}.
For using \enquote{AGEP} as a parent node, we summarised the age groups 17--34, 35--59, and 60--95. 
This means that the conditional probability table of \enquote{MAR} does not have 78 entries for \enquote{AGEP}, but three, corresponding to the ranges 17--34, 35--59, and 60--95. 

\begin{figure}
  \centering
  \begin{tikzpicture}[minimum size=1.5cm, node distance=3cm,on grid]
    \node[draw=black, circle] (RAC) at (0,0) {RAC1P};
    \node[draw=black, circle] (SEX) at (-3, 0) {SEX};
    \node[draw=black, circle] (AGEP) at (1.5, -2) {AGEP};

    \node[draw=black, circle] (SCHL) at (-1.5,-2)  {SCHL};
    \node[draw=black, circle] (MAR) at (0,-4) {MAR};
    \node[draw=black, circle] (RELP) at (1.5,-6) {RELP};
    \node[draw=black, circle] (COW) at (-1,-6) {COW};
    \node[draw=black, circle] (WKHP) at (-3.5,-6) {WKHP};

    \draw[black,-Stealth] (SEX) -- (SCHL);
    \draw[black,-Stealth] (RAC) -- (SCHL);
    \draw[black,-Stealth] (SCHL) -- (MAR);
    \draw[black,-Stealth] (AGEP) -- (MAR);
    \draw[black,-Stealth] (MAR) -- (RELP);
    \draw[black,-Stealth] (SCHL) -- (COW);
    \draw[black,-Stealth] (MAR) -- (COW);
    \draw[black,-Stealth] (SCHL) -- (WKHP);
    \draw[black,-Stealth] (MAR) -- (WKHP);
  \end{tikzpicture}
  \caption{%
    MiniACSIncome Bayesian network structure.
  }\label{fig:bayes-net-miniacsincome}
\end{figure}

To fit the Bayesian Network, we walk the network from sources to sinks and fit each conditional distribution to match the empirical distribution of the data subset that matches the current condition. 
For example, for fitting the conditional distribution of \enquote{SCHL} given \enquote{SEX=1}, \enquote{RAC1P=1}, we select the samples in the dataset having \enquote{SEX=1} and \enquote{RAC1P=1} and fit the conditional distribution of \enquote{SCHL} to match the empirical distribution of these samples.
We use categorical distributions for all categorical variables and \enquote{WKHP}.
For \enquote{AGEP}, we fit a mixture model of four truncated normal distributions that we discretise to integer values.

\Cref{fig:miniacsincome-pop-model-marginal} depict the marginal distributions of 10\,000 samples from the fitted Bayesian Network and the empirical marginal distributions of the MiniACSIncome-8 dataset.
\Cref{fig:miniacsincome-pop-model-corr} contains the correlation matrix of the same sample compared to the correlation matrix of MiniACSIncome-8.
The Bayesian Network approximates the empirical marginal distribution and the correlation structure of MiniACSIncome-8 reasonably well.

Note that for real-world fairness verification, the input distribution should not be fitted to the same dataset on which the neural network to verify is trained on.
Instead, the input distribution needs to be carefully constructed by domain experts.
For fairness audits, the input distribution could also be designed adversarially by a fairness auditing entity.

\begin{figure}
  \centering
  \includegraphics[width=.9\textwidth]{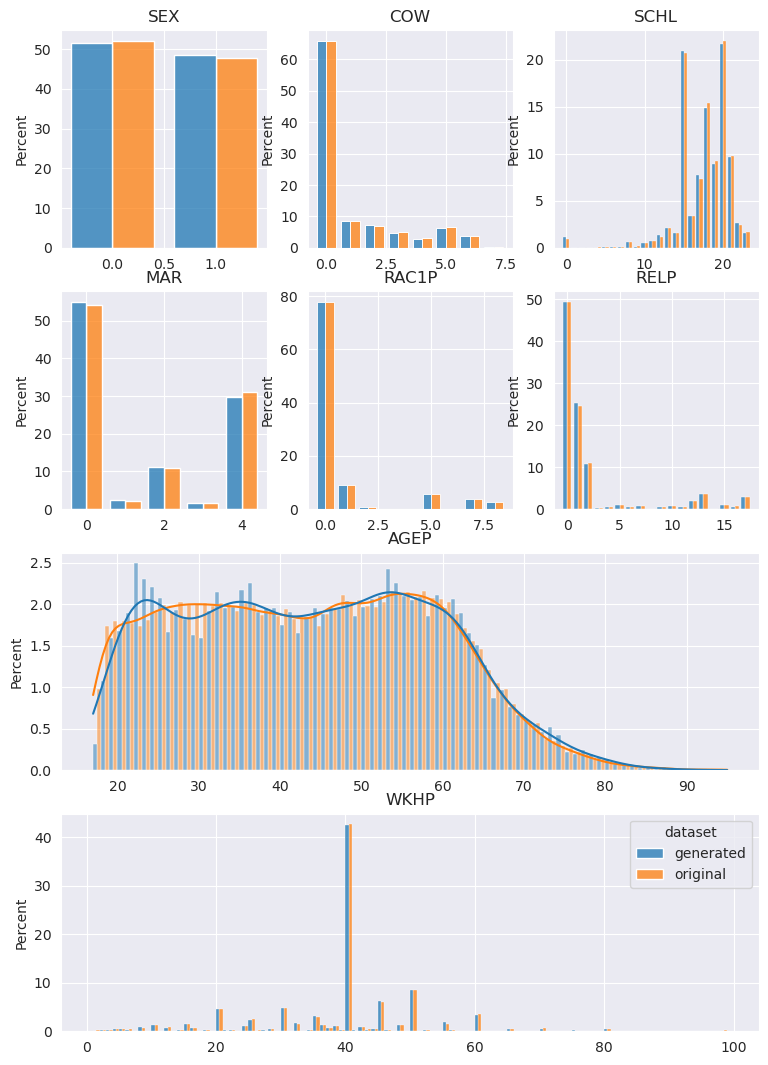}
  \caption{%
    MiniACSIncome Bayesian network~---~marginal distributions.
    The \enquote{generated} data is sampled from the Bayesian Network, while the \enquote{original} data is the MiniACSIncome-8 dataset.
  }\label{fig:miniacsincome-pop-model-marginal}
\end{figure}

\begin{figure}
  \centering
  \includegraphics[width=\textwidth]{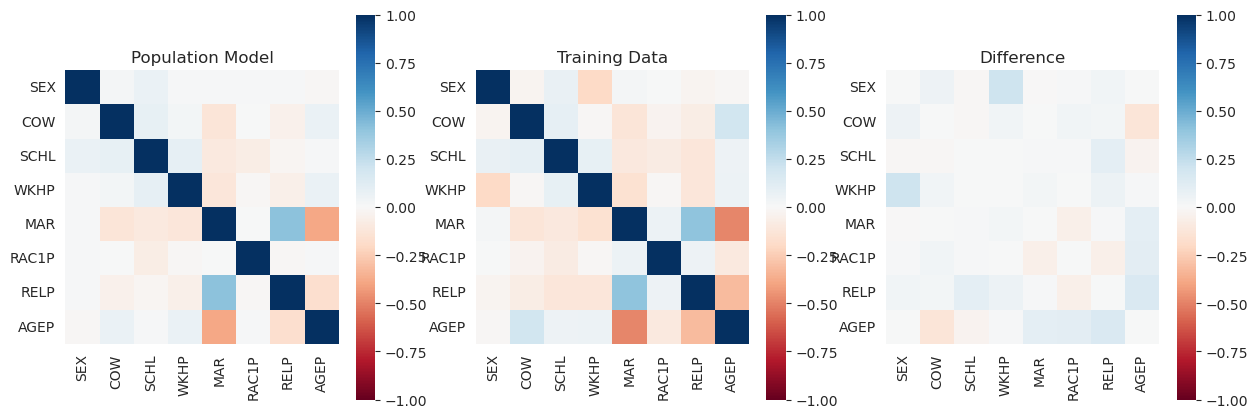}
  \caption{%
    MiniACSIncome Bayesian network~---~correlation matrix.
    \enquote{Population Model} denotes the fitted Bayesian Network, while \enquote{Training Data} stands for the full MiniACSIncome-8 dataset.
  }\label{fig:miniacsincome-pop-model-corr}
\end{figure}

\subsubsection{Training}
All MiniACSIncome neural networks are trained on a 56\%/14\%/30\% split of the MiniACSIncome-$i$ dataset into training, validation, and testing data.
This split is identical for all MiniACSIncome-$i$ datasets.
All networks are trained using: the Adam optimiser~\cite{KingmaBa2015}, cross entropy as loss function, no~\(L_2\) regularisation (weight decay),~\(\beta_1 = 0.9\),~\(\beta_2 = 0.999\),~\(\varepsilon = {10}^{-8}\) as suggested by~\textcite{KingmaBa2015}, and a learning rate decay by $0.1$ after 2000 and 4000 iterations.
The learning rate and the number of training epochs are contained in \cref{tab:miniacsincome-networks}.
We perform five random restarts for each network and select the network with the lowest cross-entropy on the validation data.

\Cref{tab:miniacsincome-perf} contains the accuracy, precision, and recall for the overall test set, the persons with female sex in the test set, and the persons with male sex in the test set.
Additionally, the table contains whether a network satisfies the demographic parity fairness notion according to \ToolName{}.

We used \AlgorithmName{Optuna}~\cite{AkibaSanoYanaseEtAl2019} for an initial exploration of the hyperparameter space but did not apply automatic hyperparameter optimisation to obtain the final training hyperparameters.

\begin{table}
  \centering
  \caption{%
    MiniACSIncome training hyperparameters.
  }\label{tab:miniacsincome-networks}
  \vspace*{0.1in}
  \begin{tabular}{lrrr}
    \toprule
    \textbf{Dataset} & \textbf{Architecture} & \textbf{Learning Rate} & \textbf{\# Epochs} \\ \midrule
    MiniACSIncome-1 & 1$\times$10 & 0.0001 & 1       \\ 
    MiniACSIncome-2 & 1$\times$10 & 0.001  & 2      \\ 
    MiniACSIncome-3 & 1$\times$10 & 0.001  & 3      \\ 
    MiniACSIncome-4 & 1$\times$10 & 0.001  & 4      \\ 
    MiniACSIncome-5 & 1$\times$10 & 0.001  & 5      \\ 
    MiniACSIncome-6 & 1$\times$10 & 0.001  & 3      \\ 
    MiniACSIncome-7 & 1$\times$10 & 0.001  & 3      \\ 
    MiniACSIncome-8 & 1$\times$10 & 0.001  & 3      \\[.25em]
    MiniACSIncome-4 & 1$\times$1000 & 0.001  & 2    \\ 
    MiniACSIncome-4 & 1$\times$2000 & 0.001  & 2    \\ 
    MiniACSIncome-4 & 1$\times$3000 & 0.001  & 2    \\ 
    MiniACSIncome-4 & 1$\times$4000 & 0.001  & 2    \\ 
    MiniACSIncome-4 & 1$\times$5000 & 0.001  & 2    \\ 
    MiniACSIncome-4 & 1$\times$6000 & 0.001  & 2    \\ 
    MiniACSIncome-4 & 1$\times$7000 & 0.001  & 2    \\ 
    MiniACSIncome-4 & 1$\times$8000 & 0.001  & 2    \\ 
    MiniACSIncome-4 & 1$\times$9000 & 0.001  & 2    \\ 
    MiniACSIncome-4 & 1$\times$10\,000 & 0.001  & 2 \\[.25em]
    MiniACSIncome-4 & 2$\times$10 & 0.001  & 2      \\
    MiniACSIncome-4 & 3$\times$10 & 0.001  & 2      \\
    MiniACSIncome-4 & 4$\times$10 & 0.001  & 2      \\
    MiniACSIncome-4 & 5$\times$10 & 0.001  & 2      \\
    MiniACSIncome-4 & 6$\times$10 & 0.001  & 2      \\
    MiniACSIncome-4 & 7$\times$10 & 0.001  & 2      \\
    MiniACSIncome-4 & 8$\times$10 & 0.001  & 2      \\
    MiniACSIncome-4 & 9$\times$10 & 0.001  & 2      \\
    MiniACSIncome-4 & 10$\times$10 & 0.001  & 2     \\
    \bottomrule
  \end{tabular}
\end{table}

\begin{table}
  \centering
  \caption{%
    MiniACSIncome neural networks. 
    The abbreviation \enquote{mACSI-$i$} stands for MiniACSIncome-$i$.
    We report the accuracy (A), precision (P), and recall (R) of each trained network (Net) for the whole test dataset (Overall), the persons with \enquote{SEX=2} in the test set (Female), and the persons with \enquote{SEX=1} in the test set (Male).
    Additionally, we report whether the network satisfies the demographic parity fairness notion according to \ToolName{} (Fair?).
  }\label{tab:miniacsincome-perf}
  \vspace*{0.1in}
  \begin{tabular}{llrrrrrrrrrc}
    \toprule
    & & \multicolumn{3}{c}{\bfseries Overall} & \multicolumn{3}{c}{\bfseries Female} & \multicolumn{3}{c}{\bfseries Male} \\ \cmidrule(lr){3-5}\cmidrule(lr){6-8}\cmidrule(lr){9-11}
    \textbf{Dataset} & \textbf{Net} & \textbf{A} & \textbf{P} & \textbf{R} & \textbf{A} & \textbf{P} & \textbf{R} & \textbf{A} & \textbf{P} & \textbf{R} & \textbf{Fair?} \\ \midrule
    mACSI-1 & 1$\times$10       & $ 57\% $ & $ 44\% $ & $ 63\% $ & $ 72\% $ &     --   & $  0\% $ & $ 44\% $ & $ 44\% $ & $100\% $ & \(\exfailure\) \\ 
    mACSI-2 & 1$\times$10       & $ 65\% $ & $ 55\% $ & $ 14\% $ & $ 71\% $ &     --   & $  0\% $ & $ 58\% $ & $ 55\% $ & $ 23\% $ & \(\exfailure\) \\ 
    mACSI-3 & 1$\times$10       & $ 73\% $ & $ 67\% $ & $ 48\% $ & $ 76\% $ & $ 65\% $ & $ 35\% $ & $ 69\% $ & $ 68\% $ & $ 55\% $ & \(\exfailure\) \\ 
    mACSI-4 & 1$\times$10       & $ 75\% $ & $ 68\% $ & $ 60\% $ & $ 78\% $ & $ 66\% $ & $ 48\% $ & $ 72\% $ & $ 69\% $ & $ 66\% $ & \(\exfailure\) \\ 
    mACSI-5 & 1$\times$10       & $ 76\% $ & $ 69\% $ & $ 63\% $ & $ 79\% $ & $ 65\% $ & $ 53\% $ & $ 74\% $ & $ 71\% $ & $ 69\% $ & \(\exfailure\) \\ 
    mACSI-6 & 1$\times$10       & $ 77\% $ & $ 69\% $ & $ 63\% $ & $ 79\% $ & $ 65\% $ & $ 55\% $ & $ 74\% $ & $ 72\% $ & $ 68\% $ & \(\exfailure\) \\ 
    mACSI-7 & 1$\times$10       & $ 77\% $ & $ 71\% $ & $ 64\% $ & $ 79\% $ & $ 67\% $ & $ 50\% $ & $ 76\% $ & $ 72\% $ & $ 72\% $ & \(\exfailure\) \\ 
    mACSI-8 & 1$\times$10       & $ 78\% $ & $ 70\% $ & $ 68\% $ & $ 80\% $ & $ 67\% $ & $ 56\% $ & $ 76\% $ & $ 71\% $ & $ 75\% $ & \(\exunknown\) \\[.25em] 
    mACSI-4 & 1$\times$1000     & $ 75\% $ & $ 71\% $ & $ 54\% $ & $ 79\% $ & $ 70\% $ & $ 44\% $ & $ 72\% $ & $ 71\% $ & $ 61\% $ & \(\exfailure\) \\ 
    mACSI-4 & 1$\times$2000     & $ 75\% $ & $ 65\% $ & $ 67\% $ & $ 78\% $ & $ 61\% $ & $ 60\% $ & $ 72\% $ & $ 67\% $ & $ 71\% $ & \(\exsuccess\) \\ 
    mACSI-4 & 1$\times$3000     & $ 74\% $ & $ 63\% $ & $ 71\% $ & $ 77\% $ & $ 58\% $ & $ 67\% $ & $ 72\% $ & $ 66\% $ & $ 73\% $ & \(\exsuccess\) \\ 
    mACSI-4 & 1$\times$4000     & $ 75\% $ & $ 71\% $ & $ 55\% $ & $ 79\% $ & $ 69\% $ & $ 44\% $ & $ 72\% $ & $ 71\% $ & $ 61\% $ & \(\exfailure\) \\ 
    mACSI-4 & 1$\times$5000     & $ 75\% $ & $ 69\% $ & $ 58\% $ & $ 78\% $ & $ 71\% $ & $ 39\% $ & $ 73\% $ & $ 68\% $ & $ 69\% $ & \(\exfailure\) \\ 
    mACSI-4 & 1$\times$6000     & $ 75\% $ & $ 69\% $ & $ 59\% $ & $ 79\% $ & $ 69\% $ & $ 44\% $ & $ 72\% $ & $ 68\% $ & $ 67\% $ & \(\exfailure\) \\ 
    mACSI-4 & 1$\times$7000     & $ 75\% $ & $ 66\% $ & $ 64\% $ & $ 78\% $ & $ 62\% $ & $ 55\% $ & $ 72\% $ & $ 68\% $ & $ 68\% $ & \(\exfailure\) \\ 
    mACSI-4 & 1$\times$8000     & $ 75\% $ & $ 66\% $ & $ 64\% $ & $ 78\% $ & $ 63\% $ & $ 55\% $ & $ 72\% $ & $ 68\% $ & $ 69\% $ & \(\exfailure\) \\ 
    mACSI-4 & 1$\times$9000     & $ 75\% $ & $ 70\% $ & $ 53\% $ & $ 78\% $ & $ 66\% $ & $ 47\% $ & $ 71\% $ & $ 72\% $ & $ 56\% $ & \(\exfailure\) \\ 
    mACSI-4 & 1$\times$10\,000  & $ 75\% $ & $ 67\% $ & $ 63\% $ & $ 79\% $ & $ 69\% $ & $ 45\% $ & $ 72\% $ & $ 66\% $ & $ 73\% $ & \(\exfailure\) \\[.25em] 
    mACSI-4 & 2$\times$10       & $ 75\% $ & $ 67\% $ & $ 59\% $ & $ 78\% $ & $ 65\% $ & $ 48\% $ & $ 72\% $ & $ 69\% $ & $ 65\% $ & \(\exfailure\) \\
    mACSI-4 & 3$\times$10       & $ 75\% $ & $ 68\% $ & $ 58\% $ & $ 79\% $ & $ 67\% $ & $ 47\% $ & $ 72\% $ & $ 69\% $ & $ 65\% $ & \(\exfailure\) \\
    mACSI-4 & 4$\times$10       & $ 75\% $ & $ 68\% $ & $ 58\% $ & $ 78\% $ & $ 66\% $ & $ 46\% $ & $ 72\% $ & $ 69\% $ & $ 66\% $ & \(\exfailure\) \\
    mACSI-4 & 5$\times$10       & $ 75\% $ & $ 68\% $ & $ 59\% $ & $ 78\% $ & $ 67\% $ & $ 47\% $ & $ 72\% $ & $ 69\% $ & $ 66\% $ & \(\exfailure\) \\
    mACSI-4 & 6$\times$10       & $ 75\% $ & $ 67\% $ & $ 61\% $ & $ 78\% $ & $ 67\% $ & $ 46\% $ & $ 72\% $ & $ 67\% $ & $ 70\% $ & \(\exfailure\) \\
    mACSI-4 & 7$\times$10       & $ 75\% $ & $ 65\% $ & $ 66\% $ & $ 78\% $ & $ 63\% $ & $ 52\% $ & $ 72\% $ & $ 65\% $ & $ 74\% $ & \(\exfailure\) \\
    mACSI-4 & 7$\times$10       & $ 75\% $ & $ 67\% $ & $ 59\% $ & $ 78\% $ & $ 65\% $ & $ 46\% $ & $ 72\% $ & $ 68\% $ & $ 67\% $ & \(\exfailure\) \\
    mACSI-4 & 9$\times$10       & $ 75\% $ & $ 67\% $ & $ 59\% $ & $ 78\% $ & $ 65\% $ & $ 47\% $ & $ 72\% $ & $ 68\% $ & $ 67\% $ & \(\exfailure\) \\
    mACSI-4 & 19$\times$10      & $ 75\% $ & $ 67\% $ & $ 61\% $ & $ 78\% $ & $ 64\% $ & $ 49\% $ & $ 72\% $ & $ 68\% $ & $ 68\% $ & \(\exfailure\)  \\
    \bottomrule
  \end{tabular}
\end{table}


\subsubsection{Extended Results}\label{sec:miniacsincome-netsize}
We provide concrete results on the effect of the network size on the runtime of \ToolName{} on the MiniACSIncome benchmark.
\Cref{fig:MiniACSIncome-net-size} displays the runtime of \ToolName{} for MiniACSIncome-4 networks of various sizes.
The studied networks include wide single-layer networks of up to~10\,000 neurons and deep networks of up to~10 layers of~10 neurons.
As the figure shows, \ToolName{} is largely unaffected by the size of the MiniACSIncome-4 networks.
This is unexpected since the network size indirectly determines the performance of \ToolName{} through the complexity of the decision boundary.
However, larger networks need not necessarily have a more complex decision boundary, and large networks do not provide a performance benefit for MiniACSIncome-4, as apparent from \cref{tab:miniacsincome-perf}.
Thoroughly exploring the impacts of network size requires more intricate datasets for which larger networks actually provide a benefit.

\begin{figure}
  \centering
  \tikzexternalenable%
  \tikzsetnextfilename{mini-acs-netsize-figure}
  \begin{tikzpicture}
    \begin{semilogxaxis}[
      width=.6\textwidth,
      height=.29\textwidth,
      ymin=0, ymax=600,
      scaled ticks=false,
      ytick={60,300,600},
      yticklabels={1,5,10},
      xlabel={\# Parameters}, 
      ylabel={Runtime (min)},
      ylabel shift=-5pt,
      legend pos=north west,
      legend cell align=left,
    ]
      \addplot [Colors-A,very thick,mark=x] table[x=NumParameters,y=Runtime,col sep=comma] {data/MiniACSIncomeNumParameters.csv};
      \addlegendentry{\ToolName{} (Ours)}
    \end{semilogxaxis}
  \end{tikzpicture}
  \tikzexternaldisable%
  \caption{%
    MiniACSIncome network size results.
    The plot depicts the runtime of \ToolName{} for MiniACSIncome-4 networks of varying sizes.
   }\label{fig:MiniACSIncome-net-size}
\end{figure}

\section{Heuristics for \ProbBounds{}}\label{sec:heuristics-additional}
In this section, we experimentally compare the \LongestEdge{}, \BaBSB{}, and \BaBSBLongestEdge{k} heuristics, as well as \IntervalArithmetic{} and \CROWN{} to justify our decision to use \BaBSB{} and \CROWN{} in \cref{sec:experiments}.

\subsection{Experiments}\label{sec:heuristics-additional-experiments}
We experimentally compare the \LongestEdge{}, \BaBSB{} and \BaBSBLongestEdge{k} heuristics to justify our selection in \cref{sec:experiments}.
Concretely, we study \BaBSBLongestEdge{10}.
Additionally, we compare \IntervalArithmetic{} to \CROWN{} when used as \ComputeBounds{} procedure of \ProbBounds{}.
To perform the comparison, we run the different variants of \ToolName{} on the FairSquare benchmark, as described in \cref{sec:fairsquare-experiment}.

\Cref{fig:split-compare} compares the \LongestEdge{}, \BaBSB{} and \BaBSBLongestEdge{10} heuristics.
It contains the runtime of \ToolName{} when using \Prob{}, \CROWN{} and either \LongestEdge{}, \BaBSB{}, or \BaBSBLongestEdge{10} as \Split{} heuristic.
While using \LongestEdge{} is faster for four easy-to-solve benchmark instances, using it only allows solving eight instances from the FairSquare benchmark, while using \BaBSB{} allows solving all 18 instances.
Using \BaBSBLongestEdge{10} also allows us to solve all benchmark instances while requiring slightly more time per benchmark instance than using \BaBSB{}.

\Cref{fig:compute-bounds-compare} contains the runtime of \ToolName{} when using \Prob{}, \BaBSB{} and either \IntervalArithmetic{} or \CROWN{} as \ComputeBounds{} procedure in \ProbBounds{}.
As the figure reveals, using \CROWN{} allows \ToolName{} to terminate faster on all benchmark instances from the FairSquare benchmark.
Furthermore, using \IntervalArithmetic{} only allows \ToolName{} to solve~14 benchmark instances, while \CROWN{} enables \ToolName{} to solve all~18 benchmark instances.

\begin{figure}
  \centering
  \tikzexternalenable%
  \tikzsetnextfilename{ablation-figure-split}
  \begin{tikzpicture}
    \begin{semilogyaxis}[
      width=.75\textwidth,
      height=.25\textwidth,
      scale only axis,
      ymax=900,
      xmin=1, xmax=18,
      xlabel={\# Solved Instances}, 
      ylabel={Runtime (s)},
      ytick={1,10,100,900},
      yticklabels={$10^{0}$,$10^{1}$,$10^{2}$,Timeout},
      ylabel shift=-22pt,
      legend pos=north west,
      legend cell align=left,
    ]


      \addplot [Colors-C,very thick,mark=x] table[x=Nr,y=Runtime,col sep=comma] {data/FairSquareRuntimeProbLongestEdgeCROWN.csv};
      \addlegendentry{\LongestEdge{}}

      \addplot [Colors-D,very thick,mark=x] table[x=Nr,y=Runtime,col sep=comma] {data/FairSquareRuntimeProbBaBSBLongestEdge10CROWN.csv};
      \addlegendentry{\BaBSBLongestEdge{10}}



      \addplot [Colors-A,very thick,mark=x] table[x=Nr,y=Runtime,col sep=comma] {data/FairSquareRuntimeProbBaBSBCROWN.csv};
      \addlegendentry{\BaBSB{}}

    \end{semilogyaxis}
  \end{tikzpicture}
  \tikzexternaldisable%
  \caption{%
    \Split{} heuristic comparison on the FairSquare benchmark.
    The timeout for the FairSquare benchmark is 15min.
  }\label{fig:split-compare}
\end{figure}

\begin{figure}
  \centering
  \tikzexternalenable%
  \tikzsetnextfilename{ablation-figure-compute-bounds}
  \begin{tikzpicture}
    \begin{semilogyaxis}[
      width=.75\textwidth,
      height=.25\textwidth,
      scale only axis,
      ymax=900,
      xmin=1, xmax=18,
      xlabel={\# Solved Instances}, 
      ylabel={Runtime (s)},
      ytick={1,10,100,900},
      yticklabels={$10^{0}$,$10^{1}$,$10^{2}$,Timeout},
      ylabel shift=-22pt,
      legend pos=north west,
      legend cell align=left,
    ]


      \addplot [Colors-E,very thick,mark=x] table[x=Nr,y=Runtime,col sep=comma] {data/FairSquareRuntimeProbBaBSBIBP.csv};
      \addlegendentry{\IntervalArithmetic{}}

      \addplot [Colors-A,very thick,mark=x] table[x=Nr,y=Runtime,col sep=comma] {data/FairSquareRuntimeProbBaBSBCROWN.csv};
      \addlegendentry{\CROWN{}}

    \end{semilogyaxis}
  \end{tikzpicture}
  \tikzexternaldisable%
  \caption{%
    \ComputeBounds{} procedure comparison on the FairSquare benchmark.
    The timeout for the FairSquare benchmark is 15min.
  }\label{fig:compute-bounds-compare}
\end{figure}


\end{document}